\documentclass[dvipsnames]{article}

\usepackage[margin=1.0in]{geometry}
\usepackage{amsmath,amsthm,amssymb, mathtools, thmtools}
\usepackage{graphicx}
\usepackage{xcolor}
\usepackage{hyperref}

\hypersetup{
    colorlinks=true,
    linkcolor=Cerulean,
    filecolor=magenta,      
    urlcolor=Cyan,
    citecolor=Cerulean
    }

\usepackage{enumitem}
\PassOptionsToPackage{sort&compress}{natbib}
\usepackage[round, sort&compress]{natbib}
\usepackage{tikz}
\usepackage{thm-restate}
\usepackage{subfig}
\usepackage{algorithm}
\usepackage{algorithmic}
\usepackage{booktabs}
\usepackage{enumitem}
\theoremstyle{definition}
\newtheorem{theorem}{Theorem}
\newtheorem{lemma}{Lemma}
\newtheorem{proposition}{Proposition}
\newtheorem{assumption}{Assumption}
\newtheorem{definition}{Definition}
\newtheorem{corollary}{Corollary}
\newtheorem{remark}{Remark}

\usepackage{amsmath,amsfonts,bm}

\def\Secref#1{Section~\ref{#1}}

\def\eqref#1{Equation~\ref{#1}} %

\def\1{\bm{1}}

\def\vr{{\bm{r}}}

\def\vx{{\bm{x}}}

\def\mA{{\bm{A}}}

\DeclareMathAlphabet{\mathsfit}{\encodingdefault}{\sfdefault}{m}{sl}
\SetMathAlphabet{\mathsfit}{bold}{\encodingdefault}{\sfdefault}{bx}{n}

\def\gB{{\mathcal{B}}}

\def\gD{{\mathcal{D}}}

\def\gX{{\mathcal{X}}}
\def\gY{{\mathcal{Y}}}

\newcommand{\E}{\mathbb{E}}
\newcommand{\Ls}{\mathcal{L}}
\newcommand{\R}{\mathbb{R}}

\newcommand{\sigmoid}{\sigma}

\newcommand{\KL}{D_{\mathrm{KL}}}
\newcommand{\Var}{\mathrm{Var}}
\newcommand{\Bias}{\mathrm{Bias}}

\newcommand{\loss}{\ell}                        %
\newcommand{\Lf}{\mathcal{L}^f}                 %

\newcommand{\glab}{g^{n}}                       %
\newcommand{\gnf}{g^{n,f}}                      %
\newcommand{\gtNf}{g^{N,f}}            %
\newcommand{\dt}{d_t}                           %
\newcommand{\glam}{\hat{g}_t}                   %

\newcommand{\Bt}{B_t}                           %

\newcommand{\norm}[1]{\left\| #1 \right\|}
\newcommand{\inner}[2]{\left\langle #1,\, #2 \right\rangle}

\renewcommand{\eqref}[1]{Eq.~(\ref{#1})}

\date{} 

\let\cite\citep

\title{ABC-Align: Prediction-Powered Alignment with\\ Adaptive Bias Control}

\author{
Eric Frankel\thanks{Corresponding author. Email: {\tt ericsf@cs.washington.edu}.} \quad
Banghua Zhu \quad
Sewoong Oh$^{\dagger}$ \quad
Lillian J.~Ratliff$^{\dagger}$
\\[6pt]
University of Washington, Seattle
}

\newcommand{\dagfootnote}{%
  \renewcommand{\thefootnote}{$\dagger$}%
  \footnotetext{Equal advising.}%
  \renewcommand{\thefootnote}{\arabic{footnote}}%
}

\begin{document}

\maketitle
\dagfootnote
\begin{abstract}
\noindent
Language model post-training is often bottlenecked by the need for human-collected preference data, which is expensive and difficult to scale.
Reinforcement learning from AI feedback (RLAIF) style approaches that leverage pseudo labels offer an abundant alternative but introduce systematic biases that degrade downstream alignment.
Recent general-purpose semi-supervised methods correct for teacher bias using a small set of human-labeled examples, but suffer from high variance especially when human annotations are scarce.
To this end, we propose ABC-Align, leveraging abundant pseudo label signal to minimize variance and applying a lightweight, adaptive correction grounded in the human-labeled subset.
The correction strength is tuned automatically during training using plug-in estimates of the relevant bias--variance quantities.
On LLM alignment with RLHF, DPO, and GRPO where human feedback is scarce, we empirically demonstrate that ABC-Align achieves superior performance over prior semi-supervised baselines in a series of experiments
on an increasing scale.
\end{abstract}

\section{Introduction}
\label{sec:intro}

Recent advances in large language model (LLM) post-training have relied heavily on substantial amounts of human labels.
In most preference-based approaches like reinforcement learnign from human feedback  (RLHF)~\citep{Christiano2017DeepRL,Ziegler2019FineTuningLM,Ouyang2022TrainingLM}, direct preference optimization (DPO)~\citep{Rafailov2023DirectPO}, 
and their many variants~\citep{Ethayarajh2024KTOMA}, these labels take the form of pairwise comparisons.
Beyond pairwise feedback, outcome-based reinforcement learning (RL) recipes such as group relative preference optimization (GRPO)~\citep{Shao2024DeepSeekMathPT} have become a standard post-training tool for reasoning models, optimizing the policy against scalar reward signals---verifiable outcomes or a learned reward model trained from human-provided ground truths---rather than explicit preference pairs.
However, this dependence on human labels 
is a shared bottleneck: human annotation is expensive, slow, and marked by substantial inter-annotator
disagreement~\citep{Casper2023OpenPA,Zhang2024DivergingPW}.

A natural response is to leverage a language model---sometimes, the training model itself---as a source of labels on unlabeled data, potentially increasing 
the volume of training signal at low cost. This idea has been explored in various forms across the LLM post-training 
literature, including using a strong LLM as a preference judge~\citep{Lee2023RLAIFVR,Bai2022ConstitutionalAH}, training 
the policy to score its own outputs and iterating~\citep{Yuan2024SelfRewardingLM}, having the model discriminate its 
current generations from human demonstrations~\citep{Chen2024SelfPlayFC}, or distilling preferences from a teacher 
model directly into the policy~\citep{Tunstall2023ZephyrDD}. However, each suffers from various failure modes that limit their
practical utility on top of human-labeled baselines. For example, LLM judges often display systematic biases towards
verbosity and sycophancy~\citep{Zheng2023JudgingLW,Dubois2024LengthControlledAA,Wei2023SimpleSD,Wang2023LargeLM}, while self-play methods plateau or degrade after entropy collapse 
and reward hacking ~\citep{Chen2024SelfPlayFC,Dong2024RLHFWF}.
Collectively, these challenges motivate the following question:

\begin{center}
\textit{When ground-truth labels  are scarce, how can we best leverage abundant but noisy pseudo labels from additional sources to improve LLM post-training?}
\end{center}

Although this problem 
has been extensively studied in  classical semi-supervised learning~\citep{Lee2013PseudoLabelT,Arazo2019PseudoLabelingAC,Sohn2020FixMatchSS,Zhang2021FlexMatchBS}, typical approaches adhere to a common design principle: the gradient estimate must remain unbiased. 
Within this family of unbiased estimators, a series of approaches have been proposed to minimize the variance.
For example, a recent line of work in  prediction-powered inference (PPI)~\citep{Angelopoulos2023PredictionpoweredI,Angelopoulos2023PPIEP} offers a general-purpose framework
for debiasing pseudo labels using a small set of trusted labels, providing provable guarantees on bias and variance. 
PPI first collects pseudo labels on the \emph{labeled} examples where ground truth is available, which are used to estimate the pseudo labeler's bias, and de-biases the pseudo labeler's predictions.  
For statistical estimation,
this yields an estimator that is provably unbiased and has lower variance than using labeled data alone.
Recently, PPI has been extended from statistical estimation to model training: PP-SSL~\citep{Shoham2025PredictionPoweredSL} constructs an
unbiased gradient estimator that tunes a variance-reduction parameter online, and PPI-SVRG~\citep{Ao2026PPISVRGUP} unifies the
PPI correction with stochastic variance reduction.
We survey these threads---PPI and its training-time extensions, classical semi-supervised learning, and pseudo-label-based LLM post-training---more comprehensively in \Secref{sec:related_work}.

We observe theoretically in Corollary~\ref{cor:variance_floor} and empirically in \Secref{sec:exp_synthetic} that, in the covariance decomposition of the prediction-powered gradient, enforcing unbiasedness creates a structural consequence:
except in a special perfect-correlation regime, the estimator inherits a variance floor of $\Theta(1/n)$ when there are $n$ human labeled examples and thus  cannot be reduced by adding more pseudo labels. 
This statement concerns the exact variance of the gradient estimator itself, rather than the looser upper bounds typically used in PP-SSL analyses.
To this end, we propose ABC-Align: a structural change that allows one to gracefully trade-off bias and variance.  
ABC-Align %
minimizes a %
bias--variance surrogate by allowing controlled bias in exchange for substantial variance reduction.
Crucially, unlike PP-SSL, our method can move much closer to pseudo-label-dominated $\Theta(1/N)$ scaling when the teacher gradient bias is small, at the cost of an explicit bias term ($N$ is the number of pseudo-labeled examples). In the ideal perfect-teacher case this removes the labeled-data variance contribution entirely; more generally it yields a controlled bias--variance trade-off. We formalize this estimator and derive its properties in \Secref{sec:abc_ssl}. Our main contributions are as follows. %
\begin{enumerate}[leftmargin=1cm, itemsep=0pt]
    \item \textbf{A biased but lower-variance gradient estimator.} For general Semi-Supervised Learning (SSL), we introduce ABC-SSL  (Algorithm~\ref{alg:abcssl}) with two parameters that can flexibly exchange bias and variance, equipped with an adaptive bias controller (ABC) that automatically tunes the two parameters in an online manner based on each mini-batch statistics  (\Secref{sec:abc_ssl}).

    \item \textbf{Oracle characterization and convergence guarantees.} 
    The controller is informed by  our characterization of the oracle per-step bias--variance trade-off of the gradient estimator. Further, we analyze convergence of ABC-SSL for smooth non-convex objectives %
    (Theorem~\ref{thm:main}). In the regime where the teacher gradient bias is small, this convergence is dominated by the abundant pseudo-labeled set rather than the scarce human-labeled set, improving upon the
unbiased methods up to a small controlled bias term %
(Corollary~\ref{cor:good_teacher}). %

    \item \textbf{Empirical validation.} We propose ABC-Align, a special instantiation of ABC-SSL under scarce human feedback tailored for each LLM alignment task, including reward model training, DPO-based policy alignment, and GRPO. To the best of our knowledge, we are the first to scale up a prediction-powered approach to LLM post-training and demonstrate significant improvement %
    over prior semi-supervised methods in the regimes we study. The largest experiment we run is on full parameter fine-tuning Llama-3.2-3B with DPO on UltraFeedback
and with GRPO on GSM8K-style.
\end{enumerate}

Our code is available at \url{https://github.com/SewoongLab/abc-align}.

\section{Background and Problem Formulation}
\label{sec:semi_supervised}

We formalize {\em prediction-powered alignment} by grounding our method in two related  areas of literature: semi-supervised learning and prediction-powered methods.
Throughout, $\norm{\cdot}$ denotes the $\ell_2$ norm and $\E_t[\cdot]$ denotes the conditional expectation given the filtration $\mathcal{F}_{t-1}$ (all randomness before iteration $t$, formally defined in \Secref{sec:estimator}), and $\prod_D$ to denote projection onto a set $D$.

\paragraph{Setup.}
Let $\loss_{\theta}(\cdot, \cdot): \gX \times \gY \to \R$ denote a loss function parameterized by $\theta\in\R^d$, where $x \in \gX$ is an input and $y \in \gY$ is a label.
We aim to find a parameter $\theta^*\in\Theta$ that minimizes the population risk
\begin{equation}
    \theta^* \;=\; \arg\min_{\theta \in \Theta} \;\Ls(\theta)\;,\; \quad \Ls(\theta) \;:=\; \E_{(x,y) \sim P}\!\bigl[\loss_{\theta}(x, y)\bigr],
    \label{eq:pop_risk}
\end{equation}
under the data-generating distribution $P:=P_X\times P_{Y|X}$.
We assume that a learner has access to a {\em labeled} dataset $\gD_{\mathrm{lab}} = \{(x_i, y_i)\}_{i=1}^n$, drawn i.i.d.\ from $P$, an {\em unlabeled} dataset $\gD_{\mathrm{unl}} = \{x_i\}_{i=n+1}^{n+N}$ of inputs drawn i.i.d.\ from the marginal $P_X$, and access to a pseudo-labeling function $f : \gX \to \gY$ that can be evaluated on both the labeled and unlabeled inputs.
We focus on the regime where labeled data is scarce ($n \ll N$), which frequently occurs in LLM post-training domains where the cost of obtaining human-labeled data is high.
In alignment settings, $y$ may represent a human preference label on a pair of examples $x=(x_1,x_2)$ and $f(x)$ may represent a preference ordering from an LLM; we defer concrete instantiations to \Secref{sec:abc_align}.

\paragraph{Semi-Supervised Learning.}
The classical approach to this problem, {\em supervised learning}, ignores the unlabeled data and optimizes the empirical risk on the labeled dataset:
\begin{equation}
    \Ls^{\mathrm{SL}}(\theta) \;:=\; L_n(\theta) \;:=\; \frac{1}{n} \sum_{i=1}^{n} \loss_{\theta}(x_i, y_i).
    \label{eq:tl_loss}
\end{equation}
While unbiased for $\Ls(\theta)$, supervised learning fails to leverage the abundant unlabeled dataset and the information contained in the pseudo-labels.
In contrast, {\em semi-supervised learning} treats pseudo-labels as ground truth and minimizes the combined loss, equally weighting all samples:
\begin{equation}
    \Ls^{\mathrm{SSL}}(\theta) \; := \; \frac{n}{n+N}L_n(\theta) + \frac{N}{n+N}L_{N}^f(\theta)\;,\quad \text{ where } \quad L_{N}^f(\theta) = \frac{1}{N} \sum_{i=n+1}^{n+N} \loss_{\theta}(x_i, f(x_i)) \;.
    \label{eq:sl_loss}
\end{equation}
This loss is biased for $\Ls(\theta)$ unless  the pseudo-labels are perfect (i.e., $f(x) = y$ a.s. for all $x\in \gX$); inaccurate pseudo-labels introduce bias that degrades the learned $\theta$, a failure mode documented across classical pseudo-labeling methods~\citep{Lee2013PseudoLabelT,Arazo2019PseudoLabelingAC} and more recent LLM alignment approaches~\citep{Bai2022ConstitutionalAH, Lee2023RLAIFVR}.
Thus, the central tension in semi-supervised learning has been how to exploit the large unlabeled dataset and its pseudo-labels without introducing bias from the pseudo-labels.

\paragraph{Prediction-Powered Inference.}
The prediction-powered inference (PPI) framework of \citet{Angelopoulos2023PredictionpoweredI,Angelopoulos2023PPIEP} resolves this tension via a zero-mean {\em control variate}. 
Because $f$ can be evaluated on the labeled data, the difference
\begin{equation*}
    L_N^f(\theta) - L_n^f(\theta)\;,\;\;\;\text{ where } L_n^f(\theta) \;=\; \frac{1}{n} \sum_{i=1}^{n} \loss_{\theta}(x_i, f(x_i)),
\end{equation*}
is the difference of two empirical averages from i.i.d.~samples, making it zero mean over $P$ for all $\theta$; therefore, $L_N^f(\theta) - L_n^f(\theta)$ can be added to the supervised loss $L_n(\theta)$ with any weight $\lambda$ without changing its expectation.
This yields a family of $\lambda$-parametrized {\em prediction-powered} losses,
\begin{equation}
    \Ls(\theta; \lambda) := L_n(\theta) + \lambda(L_N^f(\theta) - L_n^f(\theta)), \qquad \lambda \in \R,
    \label{eq:unified_lambda}
\end{equation}
that subsumes the prediction-powered estimators most relevant to our setting: $\lambda = 0$ recovers supervised learning, $\lambda = 1$ recovers vanilla PPI~\citep{Angelopoulos2023PredictionpoweredI}, and $\lambda = N/(N+n)$ recovers the doubly robust loss of \citet{Zhu2023DoublyRS}.
\citet{Angelopoulos2023PPIEP} select $\lambda\in[0,1]$ offline to minimize the asymptotic variance of their estimator $\hat\theta$ for $\theta^*$.

\paragraph{Prediction-Powered Training.} 
In the statistical settings of \citet{Angelopoulos2023PredictionpoweredI} and \citet{Zhu2023DoublyRS}, the minimizer $\theta^*$ often admits a closed-form expression, and therefore \eqref{eq:unified_lambda} can be minimized in one shot over the full datasets $\gD_{\mathrm{lab}}$ and $\gD_{\mathrm{unl}}$ via a plugin estimate.
In contrast, in the context of {\em model training}, $\theta^*$ is a predictive model that cannot be solved for in closed form and must instead be learned iteratively.
To that end, the estimator of interest is the stochastic gradient of the loss, which is computed on a mini-batch of data from two sources: labeled and unlabelled.
PP-SSL~\citep{Shoham2025PredictionPoweredSL} makes this approach explicit: at each iteration $t$, given a batch of $n$ labeled samples $\gB_t^{\mathrm{lab}}:= \{x_i, y_i\}_{i=1}^n$ and $N$ unlabeled samples $\gB_t^{\mathrm{unl}} := \{x_i\}_{i=n+1}^{n+N}$, it forms the following three batch gradients
\begin{equation*}
    \glab_t := \frac{1}{n}\sum_{i=1}^n \nabla \loss_{\theta_t}(x_i, y_i),\; \;\; \gnf_t := \frac{1}{n}\sum_{i=1}^n \nabla \loss_{\theta_t}(x_i, f(x_i)), \; \; \;\gtNf_t := \frac{1}{N}\sum_{i=n+1}^{n+N} \nabla \loss_{\theta_t}(x_i, f(x_i)),
\end{equation*}
and the \emph{prediction-powered gradient}
\begin{equation}
\label{eq:ppssl_grad}
    g^{\lambda}_t \;:=\; \glab_t \;+\; \lambda_t\bigl(\gtNf_t - \gnf_t\bigr), \qquad \lambda_t \in [0,1].
\end{equation}
This is the gradient analogue of \eqref{eq:unified_lambda}: the correction $\gtNf_t - \gnf_t$ has zero conditional mean, so $g_t^{\lambda}$ is an unbiased estimator of the true population gradient $\nabla \Ls(\theta_t)$ for all predictable $\lambda_t$.
PP-SSL tunes $\lambda_t$ online using AdaGrad-Norm to minimize the variance of $g_t^\lambda$ at each step and subtracts the (scaled) prediction-powered gradient from $\theta_t$; when $\lambda_t = 0$ for all $t$, it recovers labeled-only SGD, and $\lambda_t = 1$ recovers the vanilla PPI in its gradient form.
Note that this family is unbiased by construction for all $\lambda_t$, which we relax in \Secref{sec:abc_ssl} to design a new family of (potentially biased) gradient estimators and achieve an improved bias--variance trade-off.

\section{ABC-SSL: Adaptive Bias Control for Semi-Supervised Learning}
\label{sec:abc_ssl}
Every member of the prediction-powered family $\{g_t^\lambda\}$ in~\eqref{eq:ppssl_grad} places unit weight on the labeled gradient $\glab_t$ in order to achieve unbiasedness, which limits how much variance can be reduced: the only free parameter, $\lambda_t$, controls the contribution of the zero-mean correction $\gtNf_t - \gnf_t$ but cannot trade ground-truth signal for pseudo-label signal.
We ask whether \emph{relaxing this constraint}---admitting estimators that can downweight $\glab_t$ in exchange for a controlled bias---yields a strictly better bias--variance trade-off.
We answer affirmatively by introducing ABC-SSL, built on a two-parameter family of gradient estimators that contains PP-SSL and the doubly robust gradient as special cases (\Secref{sec:estimator}).
We characterize its oracle-optimal parameter choices and quantify when the adaptively chosen bias strictly beats the unbiased family's variance (\Secref{sec:oracle}), including a $\Theta(1/n)$ variance floor for the unbiased slice as Corollary~\ref{cor:variance_floor}. In a practical setting where we do not have access to an oracle, 
we give an online controller that learns the two parameters jointly from samples and achieves $O(\sqrt{T})$ regret against the best fixed pair. Combining it with a non-convex SGD convergence guarantee  gives an end-to-end rate that approaches the oracle (\Secref{sec:plugin}).
Notation carries over from \Secref{sec:semi_supervised}, and proofs are deferred to Appendix~\ref{app:abc_ssl_proofs}.

\begin{figure}[h!]
    \centering
    \includegraphics[width=\textwidth]{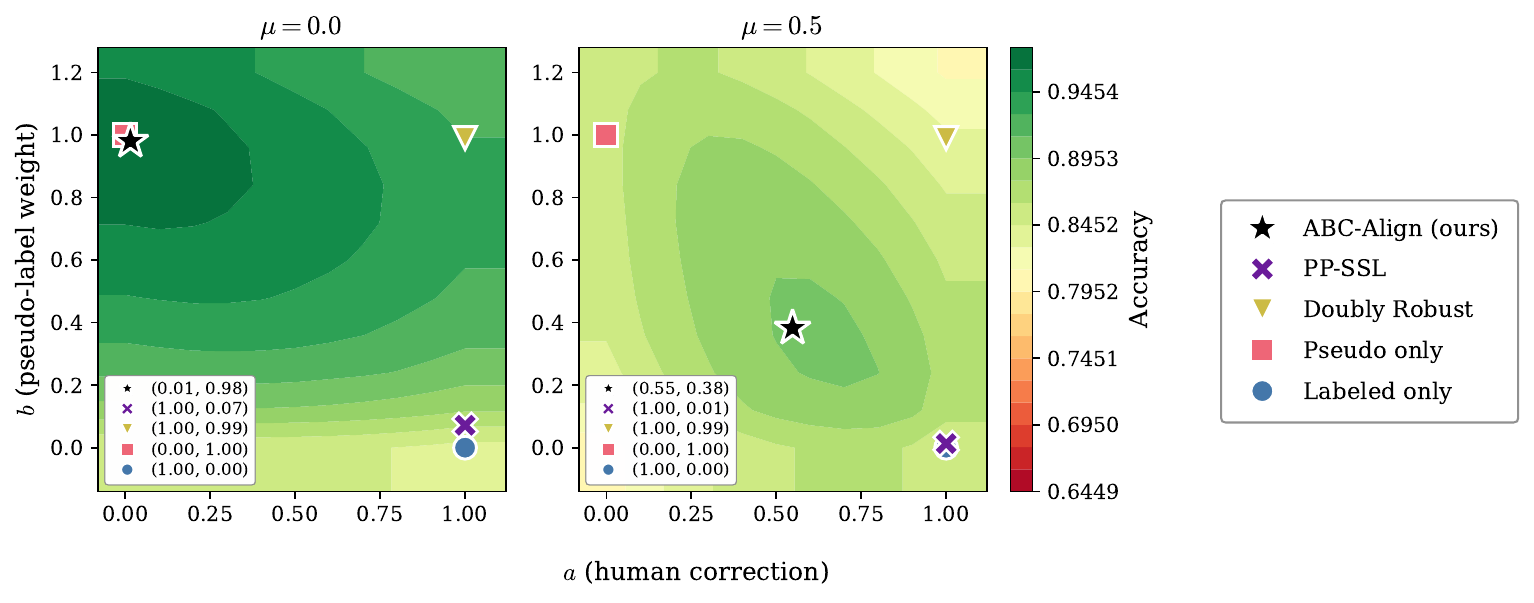}
    \caption{ Each point on the $(a, b)$ parameter plane corresponds to a model trained with a fixed pair of $(a,b)$, and the color indicated the preference accuracy achieved by the model at two teacher-bias levels: $\mu = 0$ (left) and $\mu = 0.5$ (right). 
    Each baseline corresponds to a fixed point in the plane, e.g., 
    Pseudo-only at $(0, 1)$. All the PPI-style estimators fix $a=1.0$ to eliminate the bias, resulting in suboptimal accuracy for any choices of $b$, including  
    Labeled-only at $(1, 0)$, 
    Doubly Robust at $(1, N/(N{+}n))$, and PP-SSL at $(1, \lambda_T)$---none of which coincide with the empirical maximizer in either panel. 
    Instead, ABC-Align ($\star$ shown at the final pair $(a_T,b_T)$) adaptively selects both $(a_t, b_t)$ during training and converges to the neighborhood of this maximizer in both regimes: the selected configuration approaches Pseudo-only when the teacher is accurate (left) and an interior bias--variance trade-off when the teacher is biased (right).}
    \label{fig:teaser}
\end{figure}

\subsection{An Adaptive Bias Control (ABC) Gradient Estimator}
\label{sec:estimator}

At iteration $t$, the learner draws labeled and unlabeled batches $\mathcal{B}_t^{\mathrm{lab}} \sim P^n$ and $\mathcal{B}_t^{\mathrm{unl}} \sim P_X^N$ independently of all prior randomness.
Let
\[
    \mathcal{F}_{t-1} \;:=\; \sigma\bigl(\theta_1,\,(a_1, b_1),\, \{\mathcal{B}_s^{\mathrm{lab}}, \mathcal{B}_s^{\mathrm{unl}}\}_{s < t}\bigr)
\]
denote the $\sigma$-algebra of all randomness through round $t-1$.
By construction of the update rule, $\theta_t$ and $(a_t, b_t)$ are $\mathcal{F}_{t-1}$-measurable; we call any such sequence $\{(a_t, b_t)\}_{t \ge 1}$ \emph{predictable}~\citep{Rakhlin2012OnlineLW}.
Write $\E_t[\cdot] := \E[\cdot \mid \mathcal{F}_{t-1}]$; under this convention, $\E_t$ averages over the fresh round-$t$ batches alone, holding $\theta_t$ and $(a_t, b_t)$ fixed.
Define the conditional bias $\Bias_t(\hat g) := \E_t[\hat g] - \nabla \Ls(\theta_t)$.

\begin{definition}[ABC estimator]
\label{def:abc_estimator}
Define the correction signals
\begin{equation}
    \dt := \glab_t - \gnf_t, \qquad c_t := \gtNf_t - \gnf_t,
    \label{eq:dc_def}
\end{equation}
and, for any predictable sequence $\{(a_t, b_t)\}$, the $(a_t, b_t)$-parameterized \emph{ABC estimator}, 
\begin{equation}
    \glam \;:=\; \gnf_t + a_t\,\dt + b_t\,c_t \;=\; (1 - a_t - b_t)\,\gnf_t + a_t\,\glab_t + b_t\,\gtNf_t\;.
    \label{eq:abc_grad}
\end{equation}
\end{definition}

Each baseline of \Secref{sec:semi_supervised} corresponds to a fixed slice of this family as shown in Figure~\ref{fig:teaser}: PP-SSL~\citep{Shoham2025PredictionPoweredSL} fixes $a_t = 1$ and tunes $b_t = \lambda_t$; the doubly robust gradient~\citep{Zhu2023DoublyRS} fixes $(a_t, b_t) = (1, N/(N+n))$; and naive pseudo-only training is $(a_t,b_t)=(0, 1)$.
The novelty of~\eqref{eq:abc_grad} is the adaptive $a_t$, which enables a controlled bias in exchange for additional variance reduction.
The following lemma makes the role of each parameter precise. See Appendix~\ref{app:lem_bias} for a proof.

\begin{restatable}[Bias decoupling]{lemma}{lembias}
\label{lem:bias}
Let the pseudolabel gradient bias be $B(\theta) := \nabla \Lf(\theta) - \nabla \Ls(\theta)$ with $\Lf(\theta) := \E_{x \sim P_X}[\loss(\theta; x, f(x))]$, and write $\Bt := \norm{B(\theta_t)}$.
The ABC estimator~\eqref{eq:abc_grad} satisfies $\Bias_t(\glam) = (1 - a_t)\,B(\theta_t)$.
\end{restatable}

The conditional bias of $\glam$ depends only on $a_t$, while $b_t$ scales the zero-mean correction $c_t$ and contributes only to the variance.
PP-SSL and the doubly robust estimator both enforce $a_t \equiv 1$, preserving unbiasedness but restricting variance reduction to the single parameter $b_t$; the ABC estimator instead treats $(a_t, b_t)$ jointly, admitting the controlled bias $(1 - a_t)B(\theta_t)$ in exchange for additional variance reduction.
To formalize this trade-off, we introduce the second-moment statistics, 
\begin{equation}
    \sigma^2 \;:=\; \E\!\norm{\nabla \loss(\theta; x, y) - \nabla \Ls(\theta)}^2,
    \quad
    \sigma_f^2 \;:=\; \E\!\norm{\nabla \loss(\theta; x, f(x)) - \nabla \Lf(\theta)}^2\;, \text{ and}
    \label{eq:sigma_sq}
\end{equation}
\begin{equation}
    C \;:=\; \E\!\inner{\nabla \loss(\theta; x, y) - \nabla \Ls(\theta)}{\nabla \loss(\theta; x, f(x)) - \nabla \Lf(\theta)},
    \label{eq:C_def}
\end{equation}
with $C^2 \le \sigma^2 \sigma_f^2$ by Cauchy--Schwarz; we suppress $\theta$ in these notations since it is  clear from context.
Here $\sigma^2$ and $\sigma_f^2$ are the per-sample variances of the labeled and pseudolabeled gradients, and $C$ is their cross-covariance on the \emph{same} input $x$.
Define the per-step MSE
\begin{equation}
    V_t(a, b) \;:=\; \E_t\!\norm{\glam(a, b) - \nabla \Ls(\theta_t)}^2.
    \label{eq:mse_def}
\end{equation}

\begin{restatable}[Bias--variance decomposition]{proposition}{propmse}
\label{prop:mse}
For the ABC estimator in~\eqref{eq:abc_grad}, the conditional mean-squared error decomposes as
\begin{equation}
    V_t(a_t, b_t) \;=\; (1 - a_t)^2 \Bt^2 \;+\; \frac{(1\!-\!a_t\!-\!b_t)^2 \sigma_f^2 + a_t^2 \sigma^2 + 2a_t(1\!-\!a_t\!-\!b_t)\,C}{n} \;+\; \frac{b_t^2\,\sigma_f^2}{N},
    \label{eq:mse}
\end{equation}
where the first term is $\norm{\Bias_t(\glam)}^2$ (Lemma~\ref{lem:bias}) and the remainder is $\Var_t(\glam)$.
\end{restatable}
See Appendix~\ref{app:prop_mse} for a proof. The labeled batch contributes the $1/n$ term, in which the cross covariance $C$ appears because $\glab_t$ and $\gnf_t$ are evaluated on the same $n$ inputs; the fresh unlabeled batch contributes $b_t^2 \sigma_f^2/N$ through $\gtNf_t$ alone.
If we are given an  oracle access to $(\Bt, \sigma^2, \sigma_f^2, C)$, we can find the parameters that optimally trade off bias and variance, which we do next.

\subsection{Oracle-optimal parameters}
\label{sec:oracle}

The per-step mean-square error (MSE) in~\eqref{eq:mse} is a convex quadratic in $(a, b) \in \R^2$ (Lemma~\ref{lem:convex} in Appendix~\ref{app:lem_convex}), so its unconstrained minimizer $(a^\star,b^\star)$ is the unique stationary point.
Since the unlabeled-batch term $b_t^2 \sigma_f^2/N = O(1/N)$ is lower order in the $N \gg n$ regime of interest, we state the asymptotic minimizer in the main text; the exact finite-$N$ minimizer differs by $O(1/N)$ and is presented in Appendix~\ref{app:finite_N}.

\begin{restatable}[Oracle-optimal $(a^\star, b^\star)$]{proposition}{proporacle}
\label{prop:oracle}
Let $V_\infty(a, b) := \lim_{N \to \infty} V_t(a, b)$ be the pointwise limit of~\eqref{eq:mse} and $S := \sigma^2 - C^2/\sigma_f^2$.
Whenever $\Bt^2 + S/n > 0$, the minimizer $(a^\star,b^\star)$ of $V_\infty(a,b) $ satisfy 
\begin{equation}
    a^\star = \frac{\Bt^2}{\Bt^2 + S/n} \;\in\; [0, 1],
    \qquad
    b^\star = (1 - a^\star) + a^\star \cdot \frac{C}{\sigma_f^2},
    \qquad
    V_\infty^\star = a^\star \cdot \frac{S}{n}.
    \label{eq:oracle_ab}
\end{equation}
The range $a^\star \in [0, 1]$ is an output of unconstrained minimization, not an imposed constraint: $\Bt^2 \ge 0$ trivially and $S \ge 0$ by Cauchy--Schwarz, so the classical shrinkage ratio lies in $[0, 1]$ automatically.
\end{restatable}
See Appendix~\ref{app:prop_oracle} for a proof.

\begin{corollary}[Teacher-quality regimes]
\label{cor:regimes}
\label{cor:good_teacher} The oracle MSE $V_\infty^\star = \Bt^2 S/(n\Bt^2 + S)$ smoothly interpolates between two regimes that clarify what ABC-SSL gains beyond unbiased methods.

\medskip\noindent
\textbf{Good-teacher regime ($\Bt \le \epsilon$ with $n\epsilon^2 \ll S$).}
The denominator $n\Bt^2 + S$ is dominated by $S$, so $a^\star \le n\epsilon^2/S \to 0$ and $b^\star \to 1$: the oracle places essentially full weight on the unlabeled pseudo-gradients.
Restoring the dropped $O(1/N)$ term gives $V_N^\star \;\lesssim\; \epsilon^2 + \sigma_f^2/(N+n)$, strictly below the unbiased-slice infimum $S/n$ from Corollary~\ref{cor:variance_floor} whenever $S$ is not itself $O(n/N)$, and substantially below the labeled-only MSE $\sigma^2/n$ once $N \gg n$.
The perfect-teacher case $f(x) = y$ a.s.\ is the limit $\epsilon \downarrow 0$, which yields $V_N^\star = \sigma_f^2/(N+n)$.
Figure~\ref{fig:synthetic} illustrates how ABC-SSL adaptively matches the Pseudo-only performance in this regime. %

\medskip
\noindent\textbf{Bad-teacher regime ($\Bt \to \infty$).}
The bias term dominates: $a^\star \to 1$ and $b^\star \to C/\sigma_f^2$.
The unconstrained ABC limit $\glab_t + (C/\sigma_f^2)\,c_t$ coincides with the unconstrained PP-SSL gradient~\eqref{eq:ppssl_grad} at $\lambda = C/\sigma_f^2$, and $V_\infty^\star \to S/n$.
Adaptive $a_t$ is therefore a strict generalization of PP-SSL: ABC-SSL gracefully recovers the unbiased solution when the teacher is uninformative.
\end{corollary}

The bad-teacher limit $V_\infty^\star \to S/n$ is not just a worst case for ABC-SSL---it is the \emph{best} variance the unbiased family can achieve, regardless of the teacher's quality.

\begin{restatable}[Variance floor of the unbiased family]{corollary}{corvariancefloor}
\label{cor:variance_floor}
Assume $\sigma_f^2(\theta_t) > 0$. Restricting the ABC estimator~\eqref{eq:abc_grad} to the unbiased slice $a_t = 1$ and minimizing~\eqref{eq:mse} over $b_t \in \R$ yields, for any finite $N$,
\begin{equation}
    \inf_{b_t}\, \Var_t(\glam)\big|_{a_t = 1} \;=\; \frac{S}{n} \;+\; \frac{C^2}{(N+n)\,\sigma_f^2} \;\ge\; \frac{S}{n}\;,
    \qquad S \;:=\; \sigma^2 - \frac{C^2}{\sigma_f^2} \;\ge\; 0\;,
    \label{eq:variance_floor}
\end{equation}
attained at $b_t = (C\,N/\sigma_f^2(N+n))$. The residual $S$ vanishes only when the centered per-sample labeled and pseudolabel gradients are collinear in $L^2$---a far stronger condition than samplewise label agreement. Consequently, the variance of every member of the unbiased family $\{g_t^\lambda\}$, including PP-SSL and the doubly robust gradient, is bounded below by $S/n$; whenever $S = \Theta(1)$, this is an irreducible $\Theta(1/n)$ variance floor set by the labeled batch alone.
Comparing with~\eqref{eq:oracle_ab}, the optimal ABC estimator strictly improves on this floor whenever $S > 0$ and $\Bt^2 < \infty$, and the gap is significant whenever $\Bt^2 < S/n$, which by Corollary~\ref{cor:regimes} captures the entire good-teacher regime.
\end{restatable}
See Appendix~\ref{app:cor_variance_floor} for a proof.

\subsection{Online controller and end-to-end rate}
\label{sec:plugin}

The oracle pair in~\eqref{eq:oracle_ab} depends on unobservable population statistics.
One option is to plug sample estimates of $(\Bt, \sigma^2, \sigma_f^2, C)$ into the closed form of~\eqref{eq:oracle_ab}; this finite-sample plug-in is the route we use for the on-policy prompt-group GRPO experiment of \Secref{sec:exp_grpo}, where the required per-group statistics are estimated directly from each rollout batch.
A second, more broadly applicable, option is to learn $(a_t, b_t)$ online by running a projected first-order rule directly on a surrogate of the per-step MSE $V_t(a, b)$ from~\eqref{eq:mse_def}; this online controller is the route we use for the offline reward-model and DPO experiments (\Secref{sec:exp_rm},~\Secref{sec:exp_dpo}).
The online route gives a regret bound against the best fixed $(a^\dagger, b^\dagger) \in \mathcal K$ over $T$ rounds and sidesteps the estimation of per-example gradient covariances such as $\sigma^2$ and $\sigma_f^2$ entirely; the obstacle, made explicit by the $b$-only argument of~\citet{Shoham2025PredictionPoweredSL}, is that the natural surrogate $\norm{\glam(a, b)}^2$ of $V_t(a, b)$ is faithful in $b$ but \emph{not} in $a$.
We resolve this with a single scalar correction.

\paragraph{Corrected-norm surrogate.}
Expanding $\norm{\glam(a, b)}^2$ via the bias--variance decomposition yields
\begin{equation}
    \E_t \norm{\glam(a, b)}^2
    \;=\; \norm{\nabla\Ls(\theta_t)}^2 \;+\; 2(1 - a)\,H_t \;+\; (1 - a)^2 \Bt^2 \;+\; \Var_t(\glam(a, b)),
    \qquad H_t := \inner{\nabla\Ls(\theta_t)}{B(\theta_t)}.
    \label{eq:norm_decomp}
\end{equation}
Comparing to the per-step MSE $V_t(a, b) = (1 - a)^2 \Bt^2 + \Var_t(\glam(a, b))$, the squared norm differs from $V_t$ by the constant $\norm{\nabla\Ls(\theta_t)}^2$ and the asymmetric cross term $2(1 - a) H_t$.
Over choices of $b$ alone the latter is constant, which is why the $b$-only surrogate of~\citet{Shoham2025PredictionPoweredSL} is faithful at fixed $a$; over choices of $a$, however, it is not, and any surrogate that omits a $-2(1 - a) H_t$ correction trades MSE faithfulness for a smaller-norm but larger-variance estimator.
Subtracting an unbiased estimate of the cross term defines the \emph{corrected-norm surrogate}
\begin{equation}
    \ell_t(a, b) \;:=\; \norm{\glam(a, b)}^2 \;-\; 2(1 - a)\,\widehat{H}_t,
    \qquad \E_t[\widehat H_t] = H_t.
    \label{eq:ell_def}
\end{equation}

\begin{restatable}[Faithful joint surrogate]{lemma}{lemfaithful}
\label{lem:faithful}
For any conditionally unbiased $\widehat H_t$, $\E_t[\ell_t(a, b)] = V_t(a, b) + \norm{\nabla\Ls(\theta_t)}^2$, so the additive term is independent of $(a, b)$ and minimizing $\E_t[\ell_t]$ over $(a, b)$ is equivalent to minimizing the per-step MSE.
\end{restatable}
See Appendix~\ref{app:lem_faithful} for a proof.

A conditionally unbiased $\widehat H_t$ can be built from \emph{aggregate} half-batch gradients via a split-batch construction (Lemma~\ref{lem:split_H} in Appendix~\ref{app:lem_split_H}): partition $\mathcal B_t^{\mathrm{lab}}$ into halves $A, B$ via any deterministic function of the batch indices, so exchangeability of the i.i.d.\ labeled samples makes the two halves independent draws from $P^{n/2}$, and $\widehat H_t := -\tfrac{1}{2}(\inner{g_A}{d_B} + \inner{g_B}{d_A})$ with half-batch gradients $g_A, g_A^f, g_B, g_B^f$ and $d_A := g_A - g_A^f$ satisfies $\E_t[\widehat H_t] = H_t$.
No per-example gradient covariances are required.
Crucially, $\ell_t$ is a convex quadratic in $(a, b)$ whose every coefficient is a scalar dot product of aggregate component gradients, computable in a single pass after the SGD step; in a distributed setting (DDP/FSDP/ZeRO) each rank reduces its local shard inner products with a single scalar all-reduce.

\paragraph{Online learning on $\ell_t$.}
We propose running the projected AdaGrad~\citep{Ward2018AdaGradSS,Hazan2016OCO} on the sequence $\{\ell_t\}$ from \eqref{eq:ell_def} over the feasible set $\mathcal K = [0, 1] \times [0, b_{\max}]$ (with $b_{\max} \ge 1$).
The resulting iterates $\{(a_t, b_t)\}$ are predictable by construction, and the surrogate regret propagates to a regret bound on the per-step MSE $V_t(a_t,b_t)$ via Lemma~\ref{lem:faithful}.

\begin{restatable}[Joint regret in $(a, b)$]{proposition}{propjointregret}
\label{prop:joint_regret}
Assume $\norm{\dt}, \norm{c_t}, \norm{\gnf_t}, |\widehat H_t| \le G < \infty$.
Projected AdaGrad on $\{\ell_t\}$ over $\mathcal K$ satisfies
\begin{equation}
    \sum_{t=1}^T \E[V_t(a_t, b_t)] \;-\; \min_{(a^\dagger, b^\dagger) \in \mathcal K} \sum_{t=1}^T \E[V_t(a^\dagger, b^\dagger)]
    \;\le\; D_{\mathcal K}\,G_\ell\,\sqrt{2T} \;=\; O(\sqrt T),
    \label{eq:joint_regret_V}
\end{equation}
where $D_{\mathcal K} = \sqrt{1 + b_{\max}^2}$ is the diameter of $\mathcal K$ and $G_\ell$ bounds $\norm{\nabla \ell_t}$ on $\mathcal K$.
\end{restatable}

The novelty of $\ell_t$ over the existing control-variate literature~\citep{Shoham2025PredictionPoweredSL,Zhu2023DoublyRS,Angelopoulos2023PPIEP} is the $-2(1 - a)\widehat H_t$ correction: prior surrogates live on the unbiased slice $a = 1$, where the cross term is constant and no correction is needed; the biased family $a \in [0, 1)$ requires the explicit subtraction for the controller surrogate to track the per-step MSE.

\paragraph{Convergence and end-to-end rate.}
To prove  convergence on the gradient norm $\|\nabla{\cal L}(\theta_t)\|$, we adopt the standard non-convex SGD assumptions.

\begin{assumption}[Smoothness]\label{asm:smooth}
The prediction-powered loss $\Ls$ is $\beta$-smooth---i.e.,  $\norm{\nabla \Ls(\theta_1) - \nabla \Ls(\theta_2)} \le \beta \norm{\theta_1 - \theta_2}$ for all $\theta_1, \theta_2$.
\end{assumption}
\begin{assumption}[Bounded suboptimality]\label{asm:bounded} 
The initial iterate has finite suboptimality with respect to the infimum of
the prediction-powered loss---i.e., 
$\Delta_0 := \Ls(\theta_0) - \inf_\theta \Ls(\theta) < \infty$.
\end{assumption}

The next theorem applies to \emph{any} predictable parameter sequence $\{(a_t, b_t)\}$, including the joint corrected-norm controller above, any other online-learning scheme, or a fixed $(a, b)$; the controller enters only through the realized time-averaged ABC MSE, e.g., \eqref{eq:joint_regret_V}.

\begin{restatable}[Convergence of ABC-SSL with SGD]{theorem}{thmmain}
\label{thm:main}
Under Assumptions~\ref{asm:smooth}--\ref{asm:bounded}, let $\{(a_t, b_t)\}$ be any predictable sequence in $\R^2$ and let $\theta_{t+1} = \theta_t - \eta\,\glam$ with $\eta = 1/\beta$.
Then
\begin{equation}
    \frac{1}{T} \sum_{t=1}^T \E\norm{\nabla\Ls(\theta_t)}^2 \;\le\; \frac{2\beta\Delta_0}{T} \;+\; \bar V,
    \qquad \bar V := \frac{1}{T} \sum_{t=1}^T \E[V_t(a_t, b_t)].
    \label{eq:main_bound}
\end{equation}
A variance-tuned step size yielding the same bound with $\bar V$ replaced by an a posteriori variance term is given in Appendix~\ref{app:thm_main}.
\end{restatable}

Theorem~\ref{thm:main} matches the non-convex SGD template of~\citet{Ghadimi2013StochasticFA}, with the labeled-only variance $\sigma^2/n$ replaced by the time-averaged ABC MSE $\bar V$, controlled by $(a_t, b_t)$.
Combining it with the regret bound of Proposition~\ref{prop:joint_regret} yields the headline rate.

\begin{restatable}[End-to-end rate; strict improvement over statically tuned PP-SSL]{corollary}{corendtoend}
\label{cor:end_to_end}
Under Assumptions~\ref{asm:smooth}--\ref{asm:bounded} and the hypotheses of Proposition~\ref{prop:joint_regret}, run Algorithm~\ref{alg:abcssl} with $\eta = 1/\beta$, and assume $\sigma_{f,t}^2 > 0$ for all $t$.
Let $\bar V_T^\dagger := \min_{(a, b) \in \mathcal K} \tfrac{1}{T} \sum_t \E[V_t(a, b)]$ be the in-class oracle and let
$\bar V_T^{\textrm{PP-SSL}} := \min_{b \in [0, b_{\max}]} \tfrac{1}{T} \sum_t \E[V_t(1, b)]$
be the value of the best fixed parameter on the unbiased slice $a_t = 1$, i.e., the best \emph{static} PP-SSL tuning.
Then
\begin{equation}
    \frac{1}{T} \sum_{t=1}^T \E\norm{\nabla\Ls(\theta_t)}^2
    \;\le\; \frac{2\beta\Delta_0}{T} \,+\, \bar V_T^\dagger \,+\, O(T^{-1/2})
    \;\le\; \frac{2\beta\Delta_0}{T} \,+\, \bar V_T^{\textrm{PP-SSL}} \,+\, O(T^{-1/2}),
    \label{eq:end_to_end}
\end{equation}
with strict inequality $\bar V_T^\dagger < \bar V_T^{\textrm{PP-SSL}}$ whenever
\begin{equation}
    \frac{1}{T}\sum_t \E[\Bt^2] \;<\; \infty
    \qquad\text{and}\qquad
    \frac{1}{T}\sum_t \E[S_t] \;>\; \frac{n}{4N}\cdot\frac{1}{T}\sum_t \E[\sigma_{f,t}^2],
    \qquad S_t := \sigma_t^2 - \frac{C_t^2}{\sigma_{f,t}^2}.
    \label{eq:end_to_end_strict}
\end{equation}
In the asymptotic-$N$ regime the second condition reduces to $\tfrac{1}{T}\sum_t \E[S_t]$ being positive, i.e., the centered labeled and pseudolabel per-sample gradients not being $L^2$-collinear on average.
\end{restatable}
See Appendix~\ref{app:cor_end_to_end} for a proof.
The improvement in Corollary~\ref{cor:end_to_end} is over the best \emph{static} PP-SSL tuning, matching the static comparator of Proposition~\ref{prop:joint_regret}; a comparison against the per-round retuned PP-SSL oracle $\tfrac1T\sum_t \E[\inf_b V_t(1,b)]$ would require a dynamic-regret analysis, which we do not pursue.
Pseudocode for the full ABC-SSL procedure with the joint corrected-norm controller is deferred to Appendix~\ref{app:plugin_proofs} (Algorithm~\ref{alg:abcssl}).
The bound of Corollary~\ref{cor:end_to_end} still carries the persistent bias floor $\overline{(1-a_t)^2\Bt^2}$ whenever the controller keeps $a_t$ bounded away from $1$. Appendix~\ref{app:two_timescale} develops a two-timescale variant of ABC-Align that drives this bias contribution to zero asymptotically while preserving the variance-optimality of the controller, at the cost of a single additional $\R^d$ state; its practical footprint---memory and compute overhead, initialization, and interaction with the corrected-norm controller---is discussed in Appendix~\ref{app:tt_practical}.

\section{ABC-Align for Prediction-Powered Alignment} 
\label{sec:abc_align} 
We present three instantiations of ABC-Align, i.e., ABC-SSL algorithm applied to solve LLM alignment problems: reward model training for RLHF, offline preference optimization using DPO, and on-policy RL with Group Relative Policy Optimization (GRPO).
The first two are offline preference-learning problems and use standard ABC-SSL with the corrected-norm controller of \Secref{sec:plugin} (\Secref{sec:abc_rmdpo}), under which we  empirically demonstrate the gain of the improved bias--variance trade-off.
The third is on-policy and introduces an additional challenge---the gradient distribution shifts with the policy across iterations---which we address by replacing the corrected-norm controller with a finite-sample plug-in for the oracle of Proposition~\ref{prop:oracle} (\Secref{sec:abc_grpo}).

\subsection{Background}

\paragraph{Reward Model Training for RLHF.}
We represent an LLM as a policy $\pi_\theta$ with weights $\theta$.
Given an input prompt $p$, the model generates a response $o \sim \pi_\theta(\cdot | p)$.
A reward model $r_\phi(p, o)$ is trained from human preference data: given a prompt $p$ and two responses $o_w, o_l$ where $o_w$ is preferred, the Bradley--Terry model~\citep{Bradley1952RankAO}
posits
\begin{equation}
    p(o_w \succ o_l \mid p) = \sigmoid \bigl(r_\phi(p, o_w) - r_\phi(p, o_l)\bigr),
    \label{eq:bt}
\end{equation}
and the reward model is fit by minimizing the negative log-likelihood
\begin{equation}
    \Ls_{\mathrm{RM}}(r_\phi) = -\E_{(p, o_w, o_l) \sim \mathcal{D}} \bigl[\log \sigmoid \bigl(r_\phi(p, o_w) - r_\phi(p, o_l)\bigr)\bigr].
    \label{eq:rm_loss}
\end{equation}
Once the reward model $r_\phi$ is trained, the policy $\pi_\theta$ is then optimized to maximize the KL-regularized reward objective \citep{Christiano2017DeepRL,Ziegler2019FineTuningLM,Ouyang2022TrainingLM}:
\begin{equation}
    \max_{\pi_\theta}\; \E_{p \sim \mathcal{D}}\, \E_{o \sim \pi_\theta(\cdot | p)} \bigl[r_\phi(p, o)\bigr] - \beta_{\mathrm{kl}}\, \KL \bigl[\pi_\theta(\cdot \mid p) \,\|\, \pi_{\mathrm{ref}}(\cdot \mid p)\bigr],
    \label{eq:rlhf}
\end{equation}
where $\pi_{\mathrm{ref}}$ is a reference policy (typically the supervised fine-tuned model) and $\beta_{\mathrm{kl}} > 0$ controls the strength of the KL penalty. The unique optimum of \eqref{eq:rlhf} is 
\begin{equation}
    \pi^{\star}(o \mid p) = \frac{1}{Z(p)}\, \pi_{\mathrm{ref}}(o \mid p)\, \exp \Bigl(\tfrac{1}{\beta_{\mathrm{kl}}}\, r_\phi(p, o)\Bigr),
    \label{eq:rlhf_optimal}
\end{equation}
with partition function $Z(p) = \sum_o \pi_{\mathrm{ref}}(o \mid p) \exp\bigl((1/\beta_{\mathrm{kl}})\, r_\phi(p, o)\bigr)$.

\paragraph{Direct Preference Optimization (DPO).} 

\citet{Rafailov2023DirectPO} observe that the optimal policy in \eqref{eq:rlhf_optimal} can be rearranged to express the reward as a function of the policy:
\begin{equation}
    r(p, o) = \beta_{\mathrm{kl}} \log \frac{\pi_\theta(o \mid p)}{\pi_{\mathrm{ref}}(o \mid p)} + \beta_{\mathrm{kl}} \log Z(p).
    \label{eq:implicit_reward}
\end{equation}
Substituting \eqref{eq:implicit_reward} into the Bradley--Terry model in \eqref{eq:bt} and taking the negative log-likelihood yields the DPO loss, which optimizes the policy directly without learning an explicit reward model:
\begin{equation}
    \Ls_{\mathrm{DPO}}(\pi_\theta; \pi_{\mathrm{ref}}) = -\E_{(p, o_w, o_l) \sim \mathcal{D}}  \left[\log \sigmoid \left(\beta_{\mathrm{kl}} \log \frac{\pi_\theta(o_w \mid p)}{\pi_{\mathrm{ref}}(o_w \mid p)} - \beta_{\mathrm{kl}} \log \frac{\pi_\theta(o_l \mid p)}{\pi_{\mathrm{ref}}(o_l \mid p)}\right)\right].
    \label{eq:dpo}
\end{equation}

\paragraph{Group Relative Policy Optimization (GRPO).}

GRPO~\citep{Shao2024DeepSeekMathPT} is an on-policy RL method that replaces PPO's learned value baseline with a \emph{group-relative} advantage estimated from multiple samples per prompt, and has become a standard post-training recipe for reasoning LLMs. For each prompt $p$, a behavior policy $\pi_{\mathrm{old}}$ samples a group of $G$ completions $\{o_1, \ldots, o_G\}$, each of which receives a scalar reward $R_i$; the group is then used to form the normalized advantage
\begin{equation}
    A_i \;:=\; \frac{R_i - \bar{R}}{\hat{\sigma}_R + \varepsilon_A},
    \qquad
    \bar{R} := \frac{1}{G}\sum_{j=1}^G R_j,
    \qquad
    \hat{\sigma}_R^2 := \frac{1}{G}\sum_{j=1}^G (R_j - \bar{R})^2.
    \label{eq:grpo_adv}
\end{equation}
Writing the per-token policy ratio
$\rho_{i,t}(\theta) := \pi_\theta(o_{i,t}\mid p, o_{i,<t})/\pi_{\mathrm{old}}(o_{i,t}\mid p, o_{i,<t})$,
the GRPO objective combines a PPO-style clipped surrogate with a KL anchor to a reference policy $\pi_{\mathrm{ref}}$:
\begin{equation}
    \Ls_{\mathrm{GRPO}}(\pi_\theta)
    \;=\;
    -\E_{p,\{o_i\}\sim\pi_{\mathrm{old}}} \left[
    \frac{1}{G}\sum_{i=1}^G \frac{1}{|o_i|}\sum_{t=1}^{|o_i|}
    \min \bigl(\rho_{i,t}(\theta)\,A_i,\;\mathrm{clip}(\rho_{i,t}(\theta),1{-}\epsilon,1{+}\epsilon)\,A_i\bigr)
    \right]
    + \beta_{\mathrm{kl}}\,\KL[\pi_\theta\|\pi_{\mathrm{ref}}].
    \label{eq:grpo_loss}
\end{equation}

\subsection{ABC-Align for offline preference learning: RM and DPO}
\label{sec:abc_rmdpo}

Reward-model training in~\eqref{eq:rm_loss} and DPO in~\eqref{eq:dpo} are both offline preference-learning problems: the labeled and unlabeled examples come from a fixed distribution over preference pairs, independent of the current parameters $\theta_t$.
We instantiate ABC-Align in these settings by applying Algorithm~\ref{alg:abcssl} from \Secref{sec:plugin} verbatim, with the joint corrected-norm controller updating $(a_t, b_t)$ at each step.
The only setting-specific detail is the form of the per-example loss, which we capture next.

Both objectives can be written as the Bradley--Terry log-likelihood
\begin{equation}
    \loss_i(\theta; y_i)
    \;:=\;
    -\,y_i \log \sigmoid(m_i(\theta)) \;-\; (1 - y_i) \log\bigl(1 - \sigmoid(m_i(\theta))\bigr),
    \label{eq:bt_per_example}
\end{equation}
with a (soft or hard) preference label $y_i \in [0, 1]$ and a scalar margin
\begin{equation}
    m_i(\phi) \;=\; r_\phi(p_i, o_{1,i}) - r_\phi(p_i, o_{2,i}) \quad\text{(RM),}
    \quad\;
    m_i(\theta) \;=\; \beta_{\mathrm{kl}} \log \frac{\pi_\theta(o_{1,i} \mid p_i)\,\pi_{\mathrm{ref}}(o_{2,i} \mid p_i)}{\pi_\theta(o_{2,i} \mid p_i)\,\pi_{\mathrm{ref}}(o_{1,i} \mid p_i)} \quad\text{(DPO).}
    \label{eq:scalar_margins}
\end{equation}
On a labeled example, $y_i \in \{0, 1\}$ records which of the two responses the human prefers ($y_i = 1$ if $o_{1,i}$ is preferred over $o_{2,i}$); on an unlabeled example the human preference is unavailable, and the teacher instead supplies a soft verdict $\tilde y_i := f(p_i, o_{1,i}, o_{2,i}) \in [0, 1]$ on the same pair.
The two responses are listed in a fixed but arbitrary order, so the win/lose roles are not known a priori on unlabeled data.
The chain rule then gives a single-Jacobian factorization that the controller exploits.

\begin{lemma}[Scalar-margin factorization]
\label{lem:scalar_margin}
With residuals $r_i := \sigmoid(m_i(\theta)) - y_i$, $r_i^f := \sigmoid(m_i(\theta)) - \tilde y_i$ and margin Jacobian $J_i := \nabla_\theta m_i(\theta)$,
\[
    \nabla_\theta \loss_i(\theta; y_i) \;=\; r_i\,J_i,
    \qquad
    \nabla_\theta \loss_i(\theta; \tilde y_i) \;=\; r_i^f\,J_i.
\]
\end{lemma}

The labeled component gradients $\glab_t = \tfrac{1}{n}\sum_i r_i\,J_i$ and $\gnf_t = \tfrac{1}{n}\sum_i r_i^f\,J_i$ therefore share the per-example Jacobian and differ only in the per-example residual scalar.
A single labeled forward pass produces $\{m_i(\theta_t)\}_{i=1}^n$, and two weighted backwards (with weights $\{r_i\}$ and $\{r_i^f\}$) produce both aggregates through the shared activations.
The half-batch aggregates $g_A, g_A^f, g_B, g_B^f$ used to form the split-batch $\widehat H_t$ in~\eqref{eq:H_hat} are obtained by restricting the same backwards to indices in $A$ or $B$, with no extra backward passes.
Forming $\gtNf_t$ on the unlabeled batch costs one unlabeled forward and one weighted backward.
Compared to a labeled-only baseline, the controller refresh adds exactly one extra labeled backward and a handful of scalar all-reduces; no per-example Jacobian norms are required.

\subsection{ABC-Align for on-policy GRPO}
\label{sec:abc_grpo}

GRPO~\citep{Shao2024DeepSeekMathPT} is an on-policy method: at each iteration the behavior policy $\pi_{\mathrm{old}}$ is set to the current parameters $\pi_{\theta_t}$, fresh rollouts are sampled, and the gradient distribution shifts across iterations as $\theta_t$ moves.
This shift complicates the joint corrected-norm controller of \Secref{sec:plugin}, whose regret guarantee (Proposition~\ref{prop:joint_regret}) is against the \emph{best fixed} $(a^\dagger, b^\dagger)$ over $T$ rounds.
When the per-step oracle $(a_t^\star, b_t^\star)$ of Proposition~\ref{prop:oracle} itself drifts substantially under the changing policy, the best fixed pair is a poor approximation of the running oracle, and online tracking of the surrogate $\ell_t$ becomes less informative than direct estimation of the per-step oracle at each iteration.
We therefore bypass the corrected-norm controller for GRPO and use a \emph{finite-sample plug-in} instead: at each step, estimate $(\Bt, \sigma^2, \sigma_f^2, C)$ from the freshly sampled labeled prompt-group batch and substitute into the closed form~\eqref{eq:oracle_ab} to obtain $(a_t, b_t)$.

\paragraph{Source-specific advantages and prompt-group atom.}
At iteration $t$, draw a labeled prompt batch $\{p_j\}_{j=1}^n$ and an unlabeled prompt batch $\{p_j\}_{j=n+1}^{n+N}$.
For each prompt $p_j$, sample a rollout group $\{o_{j,i}\}_{i=1}^G \sim \pi_{\mathrm{old}}(\cdot \mid p_j)$.
Labeled prompts are scored by both a trusted reward source (human feedback or a human-anchored reward model) and the teacher $f$ on the \emph{same} rollouts, yielding $\{R_{j,i}^y, R_{j,i}^f\}_{i=1}^G$; unlabeled prompts are scored by the teacher only.
We normalize each source independently into a group-relative advantage vector via~\eqref{eq:grpo_adv}:
\begin{equation}
    A_{j,i}(y) \;:=\; \frac{R_{j,i}^y - \bar R_j^y}{\hat\sigma_{R_j^y} + \varepsilon_A}\;,
    \qquad 
    A_{j,i}(\tilde y) \;:=\; \frac{R_{j,i}^f - \bar R_j^f}{\hat\sigma_{R_j^f} + \varepsilon_A}\;.
    \label{eq:grpo_adv_label}
\end{equation}
The atomic sample is the prompt-group $Z_j := (p_j, \{o_{j,i}\}_{i=1}^G)$.
Source-separated normalization is required: the map $R \mapsto A$ is nonlinear, so the naive shortcut of forming a mixed reward and normalizing once does not commute with reward mixing and yields a gradient that is no longer of the ABC form.

\paragraph{ABC mixing on the advantage component.}
Decompose the GRPO objective~\eqref{eq:grpo_loss} into a reward-dependent advantage component $\Ls^{\mathrm{adv}}$ and a label-independent KL anchor $\Ls^{\mathrm{KL}}$.
Only $\Ls^{\mathrm{adv}}$ depends on the reward source, so only its gradient carries a teacher-induced bias.
With the source-specific batch advantage gradients, 
\begin{equation}
    \glab_t := \nabla_\theta \Ls^{\mathrm{adv}}\bigl(\pi_\theta; A(y)\bigr)\bigm|_{\text{lab.}},
    \quad
    \gnf_t := \nabla_\theta \Ls^{\mathrm{adv}}\bigl(\pi_\theta; A(\tilde y)\bigr)\bigm|_{\text{lab.}},
    \quad
    \gtNf_t := \nabla_\theta \Ls^{\mathrm{adv}}\bigl(\pi_\theta; A(\tilde y)\bigr)\bigm|_{\text{unl.}},
\end{equation}
the ABC-GRPO update applies ABC mixing to the advantage component and adds the KL gradient $g_t^{\mathrm{KL}} := \nabla_\theta \Ls^{\mathrm{KL}}$ as a deterministic offset: 
\begin{equation}
    \glam^{\mathrm{ABC\text{-}GRPO}}
    \;:=\;
    g_t^{\mathrm{KL}} \,+\, \gnf_t \,+\, a_t\,\dt \,+\, b_t\,c_t,
    \qquad
    \dt := \glab_t - \gnf_t,\;\; c_t := \gtNf_t - \gnf_t.
    \label{eq:abc_grpo_update}
\end{equation}
We treat the KL gradient term as regularization and only apply our approach to the source-specific batch advantage gradients.

\paragraph{Prompt-group MSE decomposition.}
Promote the per-example atom $(x, y) \sim P$ in Proposition~\ref{prop:mse} to the prompt-group atom $Z_j$, and let $g_j^h(\theta)$, $g_j^f(\theta)$ denote the per-group advantage gradients (KL excluded) under labels $A_j(y)$ and $A_j(\tilde y)$.
Define the per-group means $\mu_t^h := \E_t[g_j^h]$, $\mu_t^f := \E_t[g_j^f]$, bias $B_{t,\mathrm{G}} := \mu_t^f - \mu_t^h$, and second-moment statistics
\begin{equation}
    \sigma_{t,\mathrm{G}}^2 := \E_t\norm{g_j^h - \mu_t^h}^2,
    \quad
    \sigma_{t,\mathrm{G},f}^2 := \E_t\norm{g_j^f - \mu_t^f}^2,
    \quad
    C_{t,\mathrm{G}} := \E_t\inner{g_j^h - \mu_t^h}{g_j^f - \mu_t^f}.
    \label{eq:grpo_pop_cov}
\end{equation}
\begin{assumption}[Prompt-group sampling]
\label{asm:grpo_exchange}
At each iteration $t$, the $n$ labeled and $N$ unlabeled prompts are drawn i.i.d.\ from a common prompt distribution $\mathcal{D}_P$.
Conditional on $\mathcal F_{t-1}$, every rollout group is sampled from $\pi_{\mathrm{old}}$, and the human and teacher rewards on each labeled prompt are evaluated on the same sampled rollouts.
\end{assumption}

\noindent Assumption~\ref{asm:grpo_exchange} makes the prompt-group atom $Z_j$ play exactly the role of the per-example atom $(x,y)$ in \Secref{sec:abc_ssl}: the groups are i.i.d.\ draws, and the shared-rollout condition is what couples $g_j^h$ and $g_j^f$ on the labeled prompts through the cross covariance $C_{t,\mathrm{G}}$.

\begin{proposition}[Group-level MSE decomposition]
\label{prop:grpo_mse}
Under Assumption~\ref{asm:grpo_exchange}, for any $\mathcal F_{t-1}$-measurable pair $(a, b)$ the advantage component of the update~\eqref{eq:abc_grpo_update} evaluated at $(a, b)$ has conditional mean-squared error
\begin{equation}
    V_{t,\mathrm{G}}(a, b) \;=\; (1 - a)^2 \norm{B_{t,\mathrm{G}}}^2 \;+\; \frac{(1\!-\!a\!-\!b)^2 \sigma_{t,\mathrm{G},f}^2 + a^2 \sigma_{t,\mathrm{G}}^2 + 2a(1\!-\!a\!-\!b)\,C_{t,\mathrm{G}}}{n} \;+\; \frac{b^2\,\sigma_{t,\mathrm{G},f}^2}{N},
    \label{eq:grpo_mse}
\end{equation}
which is the expression of~\eqref{eq:mse} with $(\Bt, \sigma^2, \sigma_f^2, C)$ replaced by $(\norm{B_{t,\mathrm{G}}}, \sigma_{t,\mathrm{G}}^2, \sigma_{t,\mathrm{G},f}^2, C_{t,\mathrm{G}})$.
Consequently $V_{t,\mathrm{G}}$ is a convex quadratic in $(a,b)$ and its minimizer is given by Proposition~\ref{prop:oracle} under the same substitution.
\end{proposition}
\noindent The proof is identical to that of Proposition~\ref{prop:mse} (Appendix~\ref{app:prop_mse}) with $Z_j$ in place of $(x,y)$; Assumption~\ref{asm:grpo_exchange} supplies the i.i.d.\ and shared-rollout properties that proof uses.
The KL gradient $g_t^{\mathrm{KL}}$ is conditionally deterministic given $\mathcal F_{t-1}$, so it shifts the gradient target but contributes neither to $B_{t,\mathrm{G}}$ nor to the variance terms.

\paragraph{Split-batch plug-in estimates from aggregate half-batch gradients.}
On the freshly sampled on-policy batch, $\rho_{j,i,t}(\theta) = 1$ and the clip in~\eqref{eq:grpo_loss} is inactive, so the per-(prompt, completion) advantage gradient factors as the GRPO analogue of Lemma~\ref{lem:scalar_margin}:
\begin{equation}
    \nabla_\theta \loss_{j,i}(y) \; = \; -A_{j,i}(y)\,J_{j,i},
    \qquad
    J_{j,i} \;:=\; \frac{1}{|o_{j,i}|}\sum_{\tau=1}^{|o_{j,i}|} \nabla_\theta \log\pi_\theta(o_{j,i,\tau} \mid p_j, o_{j,i,<\tau}).
    \label{eq:grpo_pergrad}
\end{equation}
We estimate the per-group statistics in~\eqref{eq:grpo_pop_cov} by the same split-batch device as Lemma~\ref{lem:split_H}: partition the $n$ labeled prompt-groups into halves $A, B$ by a deterministic function of the indices, form the half-batch advantage gradients $g_A^h, g_A^f, g_B^h, g_B^f$ (averages of $g_j^h$, resp.\ $g_j^f$, over $j$ in each half), and set $d_A := g_A^h - g_A^f$, $d_B := g_B^h - g_B^f$. Define
\begin{equation}
\begin{aligned}
    \widehat\sigma_{t,\mathrm{G}}^2 &:= \tfrac{n}{4}\norm{g_A^h - g_B^h}^2,
    &\qquad
    \widehat\sigma_{t,\mathrm{G},f}^2 &:= \tfrac{n}{4}\norm{g_A^f - g_B^f}^2,\\
    \widehat C_{t,\mathrm{G}} &:= \tfrac{n}{4}\inner{g_A^h - g_B^h}{g_A^f - g_B^f},
    &\qquad
    \widehat{\norm{B_{t,\mathrm{G}}}^2} &:= \bigl(\inner{d_A}{d_B}\bigr)_+.
\end{aligned}
    \label{eq:grpo_splitbatch}
\end{equation}
Under Assumption~\ref{asm:grpo_exchange} the two halves are independent draws from the prompt-group distribution, so each estimator is conditionally unbiased before the clipping $(\cdot)_+$: e.g., $\E_t\norm{g_A^h - g_B^h}^2 = \tfrac{4}{n}\,\sigma_{t,\mathrm{G}}^2$ and $\E_t\inner{d_A}{d_B} = \norm{B_{t,\mathrm{G}}}^2$ (the clip only guards against negative realizations of the squared-bias estimate). All four quantities are scalar dot products of aggregate half-batch gradients; no per-completion gradients or Gram matrices are formed.
Substituting $(\widehat{\norm{B_{t,\mathrm{G}}}^2}, \widehat\sigma_{t,\mathrm{G}}^2, \widehat\sigma_{t,\mathrm{G},f}^2, \widehat C_{t,\mathrm{G}})$ into~\eqref{eq:oracle_ab} gives the plug-in pair $(a_t, b_t)$ used by the ABC-GRPO update.

\begin{remark}[The plug-in pair is not predictable]
\label{rem:grpo_predictability}
The plug-in $(a_t, b_t)$ is computed from the same round-$t$ labeled batch that produces $\glab_t, \gnf_t$, so it is \emph{not} $\mathcal F_{t-1}$-measurable, and Proposition~\ref{prop:grpo_mse}---whose proof pulls $(a, b)$ through the conditional variance, licensed only for $\mathcal F_{t-1}$-measurable pairs (cf.\ Definition~\ref{def:abc_estimator})---describes the MSE landscape that the plug-in \emph{targets} rather than the realized MSE of the deployed update, which acquires additional cross-terms from the correlation between $(a_t, b_t)$ and the round-$t$ gradients. Predictability (and with it the exact identity~\eqref{eq:grpo_mse} at the deployed pair) can be restored by computing the plug-in from the previous round's labeled batch, at the cost of a one-round delay.
\end{remark}

\section{Experiments}
\label{sec:experiments}
We empirically validate ABC-Align across four settings of increasing scale and realism: a controlled synthetic preference task (\Secref{sec:exp_synthetic}), reward model (RM) training on real preference data (\Secref{sec:exp_rm}), DPO-based policy alignment (\Secref{sec:exp_dpo}), and on-policy GRPO (\Secref{sec:exp_grpo}).
In the synthetic setting, we compare against four baselines:
\begin{enumerate}[label={(\roman*)}, itemsep=0pt]
    \item \textbf{Labeled-only}, which trains exclusively on the $n$ human-labeled samples;
\item \textbf{Pseudo-label-only}, which trains on all $N$ pseudo-labeled samples, ignoring human labels;
\item \textbf{Doubly Robust (DR)}, corresponding to $(a,b) = (1, N/(N+n))$ in~\eqref{eq:abc_grad}, fixed over time; and
\item \textbf{PP-SSL}~\citep{Shoham2025PredictionPoweredSL}, the prediction-powered method from~\eqref{eq:ppssl_grad}.
\end{enumerate}
In the LLM experiments, Doubly Robust is replaced with \textbf{Na\"ive combined}, which pools human and teacher labels without any correction, a natural baseline used in practice.
All methods within each experiment use identical architectures, optimizers, and batch sizes; only the gradient estimator differs.

\subsection{Synthetic preference experiment on linear reward model training}
\label{sec:exp_synthetic}
We first validate the theoretical predictions of \Secref{sec:abc_ssl} on a synthetic preference task that mimics reward model training and has fully observable ground-truth quantities, allowing us to directly verify the bias--variance trade-off before moving on to real data.

\paragraph{Setup.}
A ground-truth linear reward $r^\star(x) = {w^\star}^\top x$, where $w^\star \in \R^m$ and $m = 20$, generates pairwise preferences via the Bradley--Terry model with annotator noise $\sigma_\epsilon = 0.5$.
A biased teacher produces pseudo labels using a corrupted reward $r^f(x) = (w^\star + \mu v)^\top x$, where $v$ is a fixed unit-norm perturbation and $\mu \ge 0$ controls bias intensity; teacher labels are deterministic given $r^f$.
We use $n = 50$ labeled pairs and $N = 5{,}000$ pseudo-labeled pairs ($100\times$ ratio).
We fit a linear reward model $r_\phi(x) = \phi^\top x$ by minimizing~\eqref{eq:rm_loss} on pairwise feature differences.
All methods share the same optimizer Adam~\citep{Kingma2014AdamAM} with $\text{lr} = 10^{-3}$, same batch size~$32$, and run for $T = 2{,}000$ iterations; they differ only in the gradient estimator.
For ABC-Align, the pair $(a_t, b_t)$ is updated by the joint corrected-norm controller of \Secref{sec:plugin}: at each step $\widehat H_t$ from~\eqref{eq:H_hat} is EMA-smoothed (rate $\alpha_H = 0.05$; a variance-reduction heuristic, see Remark~\ref{rem:ema_heuristic}), and projected AdaGrad on $\ell_t$ over $\mathcal K = [0, 1] \times [0, b_{\max}]$ produces $(a_{t+1}, b_{t+1})$. For PP-SSL, $\lambda_t$ is tuned online by projected gradient descent on $\|\hat g_t\|^2$.
We sweep $\mu \in \{0.0, 0.2, \ldots, 1.0\}$ and average over $25$ trials.
Additional details are in Appendix~\ref{app:synth_details}.

\begin{figure}[h]
    \centering
    \includegraphics[width=\textwidth]{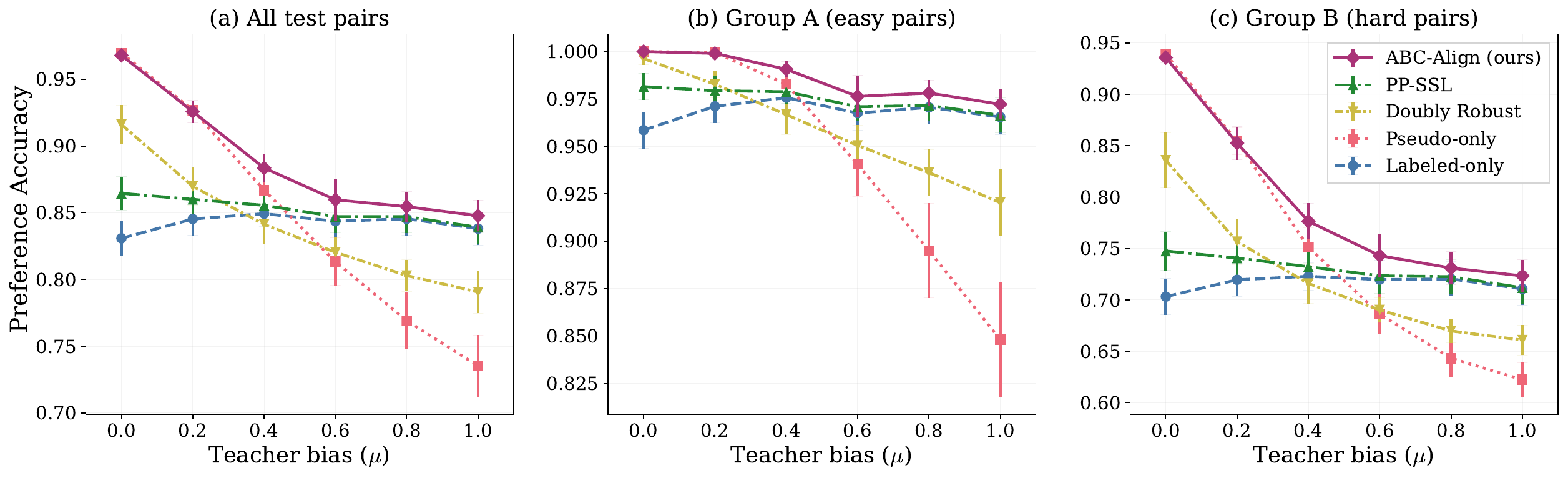}
    \caption{Preference accuracy  vs.\ teacher bias on synthetic data ($n = 50$, $N = 5{,}000$, $25$ trials) shows that ABC-Align robustly improves upon the best performing baseline at every operating point of $\mu$.  For a more fine-grained analysis, test pairs are split by reward-gap magnitude into easy (Group~A) and hard (Group~B) subgroups. ABC-Align tracks pseudo-only when the teacher is accurate (small $\mu$), remains strongest on hard pairs, and degrades smoothly as teacher bias increases; PP-SSL cannot fully exploit pseudo-labels when the teacher is accurate.}
    \label{fig:synthetic}
\end{figure}

\paragraph{Results.}
Figure~\ref{fig:synthetic} reports preference accuracy, i.e., $\Pr[{\rm sign}({w^\star}^\top(x_1-x_2)) = {\rm sign}(\phi^\top (x_1-x_2))]$, on $10{,}000$ held-out pairs, split at the median reward gap into easy (Group~A) and hard (Group~B) subsets.
When the teacher is accurate ($\mu \le 0.4$), ABC-Align nearly matches pseudo-only performance while outperforming labeled-only, DR, and PP-SSL.
The gain concentrates on hard pairs, where $n = 50$ labels leave unbiased estimators variance-limited.
As $\mu$ increases, the controller shifts toward larger $a_t$ (stronger human correction) and ABC-Align degrades smoothly rather than collapsing---consistent with Proposition~\ref{prop:mse}.
Figure~\ref{fig:teaser} additionally visualizes the $(a, b)$ plane: ABC-Align's terminal $(a_T, b_T)$ lands near the optimal region of the heatmap across teacher-bias levels, while the baselines remain pinned to their fixed slices. Figures~\ref{fig:synth_traj} and~\ref{fig:synth_mse} further show that $(a_t, b_t)$ converges toward the per-step oracle and that ABC-Align attains the lowest gradient MSE throughout training.

\begin{figure*}[t]
    \centering
    \includegraphics[width=\textwidth]{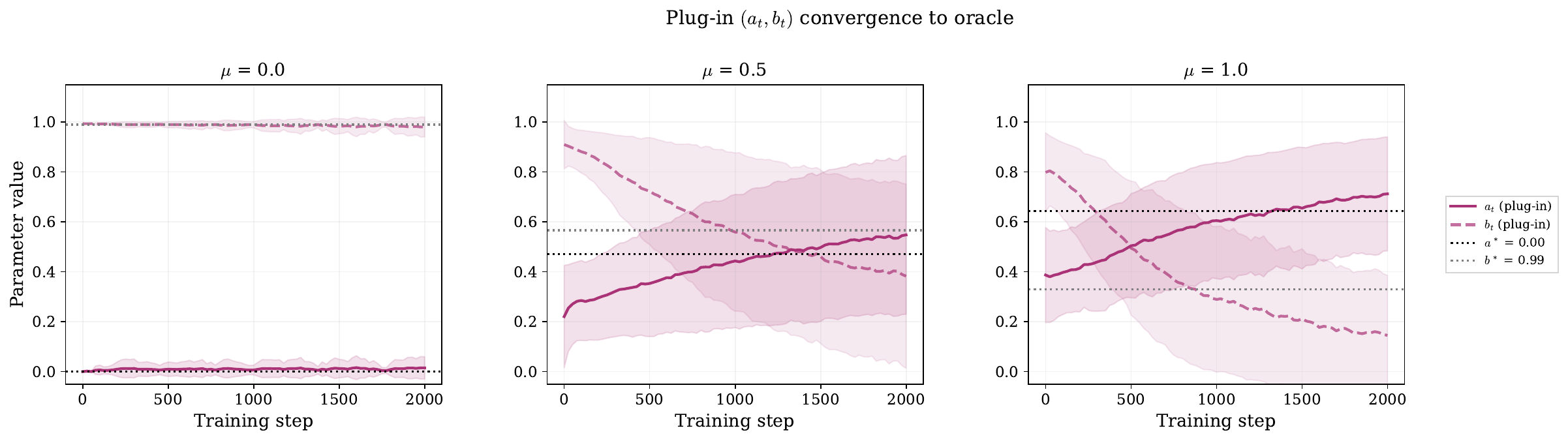}
    \caption{Evolution of the ABC-Align plug-in parameters during training, at teacher-bias levels $\mu \in \{0.0, 0.5, 1.0\}$. Solid and dashed lines show the learned $(a_t, b_t)$; dotted lines show the per-step oracle $(a^\star, b^\star)$ of Proposition~\ref{prop:oracle}, evaluated at the population statistics of the corresponding $\mu$ (so the oracle target differs across panels). The learned trajectory moves toward the oracle target and shifts from pseudo-label-dominant behavior at low teacher bias toward stronger human correction as bias increases. Shaded bands show $\pm 1$ standard error across trials.
    }
    \label{fig:synth_traj}
\end{figure*}

\begin{figure*}[t]
    \centering
    \includegraphics[width=\textwidth]{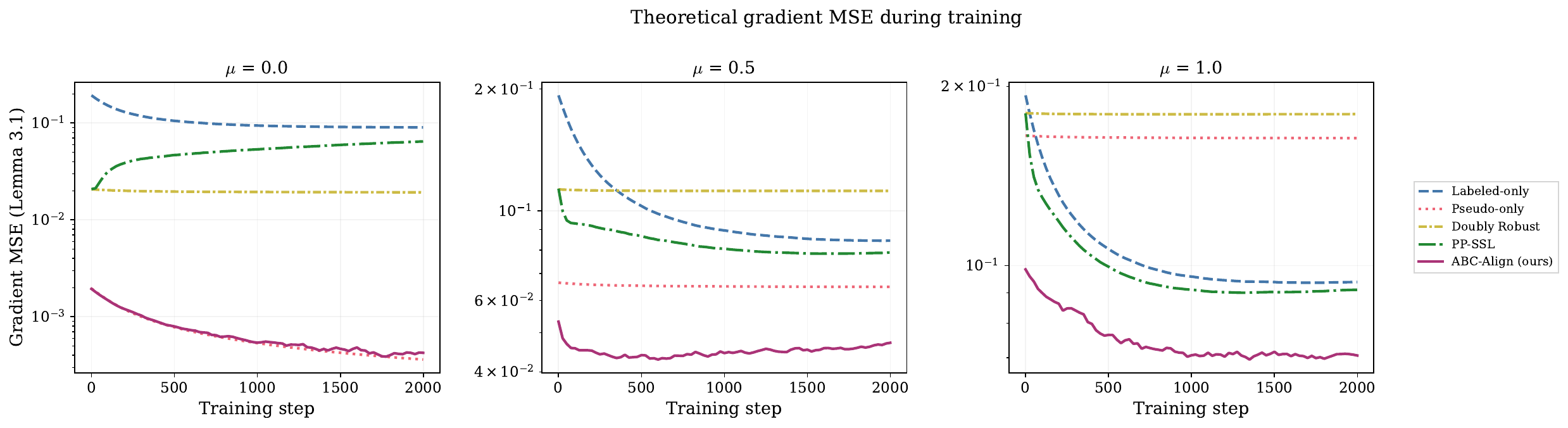}
    \caption{Theoretical gradient MSE computed along training trajectories. ABC-Align consistently attains the lowest estimator MSE among the compared methods, with the largest gap when the teacher is accurate or moderately biased.}
    \label{fig:synth_mse}
\end{figure*}
\subsection{Reward model training}
\label{sec:exp_rm}

We next evaluate ABC-Align on a real reward-model training problem using the same Bradley--Terry objective as in~\eqref{eq:rm_loss}.
Our main experiment uses Llama-3.2-3B-Instruct~\citep{Dubey2024TheL3} as the base reward model and trains only the scalar score head while freezing the backbone, which isolates the effect of the gradient estimator from full-model optimization.
We use the Skywork Reward Preference 80K~\citep{Liu2024SkyworkRewardBO,Liu2025SkyworkRewardV2SP} training split, shuffle it once, reserve $2{,}000$ examples as a held-out pairwise-accuracy set, and partition the remainder into a small human-labeled subset and a large unlabeled pool.
We sweep the labeled fraction over $\{0.01, 0.05, 0.10\}$.
On labeled examples we retain the human preference; on unlabeled examples we supply teacher preferences.
In the main sweep, teacher labels are synthesized by flipping the ground-truth preference with probability $\rho \in \{0.05, 0.10, 0.20\}$, yielding a controlled teacher-quality axis on real prompts and responses that mirrors the synthetic experiment. The flip-rate protocol gives an \emph{unbiased} teacher (label noise is symmetric across classes); we additionally validate on a structurally biased real teacher below (Table~\ref{tab:rm_real_teacher}), where teacher--human disagreements concentrate on response style rather than spreading uniformly.

All methods share the same architecture, tokenizer, optimizer, and update budget, differing only in the gradient estimator.
We compare ABC-Align against three reward-model baselines: \textbf{Labeled-only}, \textbf{Na\"ive combined}, and \textbf{PP-SSL}.
To test whether the online tuning rule is doing real work, we additionally run fixed-mixture ablations for both ABC-Align and PP-SSL at a representative label fraction and teacher-noise level.
Evaluation combines in-domain held-out pairwise accuracy with out-of-domain reward-model benchmarks: RewardBench~\citep{Lambert2024RewardBenchER,Malik2025RewardBench2A} and PPE~\citep{Frick2024HowTE}, where we emphasize human-preference accuracy and the math best-of-$k$ proxy.
Sequence lengths, batch sizes, the optimizer, and the frozen-backbone protocol are deferred to Appendix~\ref{app:rm_details}.

\paragraph{Results: in-domain pairwise accuracy.}
Figure~\ref{fig:rm_pairwise} reports held-out pairwise accuracy on the Skywork validation split across the three teacher flip rates $\rho \in \{0.05, 0.10, 0.20\}$ and label fractions $\{0.01, 0.05, 0.10\}$. ABC-Align is the strongest method at every point in the sweep.
At the hardest setting ($\rho = 0.20$, label fraction $0.01$), it reaches $0.77$ where the next-best baseline (PP-SSL) attains $0.70$ and labeled-only $0.64$. The gap persists as the teacher improves: at $\rho = 0.05$ and label fraction $0.01$, ABC-Align already exceeds what labeled-only achieves with $10\times$ as many human labels.
Naive combined and PP-SSL trail ABC-Align by $3$--$7$ points throughout, and their ordering flips with teacher noise---Naive is competitive at $\rho = 0.05$ but collapses as teacher quality drops, while PP-SSL's unbiased correction keeps it robust under $\rho = 0.20$ at the cost of leaving variance reduction on the table at $\rho = 0.05$.
This pattern matches Proposition~\ref{prop:mse}: by tuning both $(a_t, b_t)$, ABC-Align adapts along the full bias--variance frontier rather than paying either the full pseudo-pooling bias or the unbiased variance floor $S/n$.

\begin{figure}[t]
    \centering
    \includegraphics[width=\textwidth]{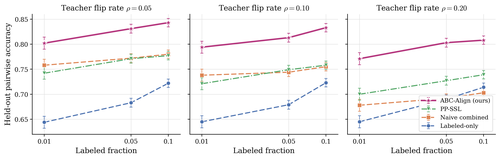}
    \caption{Held-out pairwise accuracy on Skywork-Reward-80K (Llama-3.2-3B-Instruct, frozen backbone) across label fractions $\{0.01, 0.05, 0.10\}$ and teacher flip rates $\rho \in \{0.05, 0.10, 0.20\}$. ABC-Align dominates at every setting; its advantage widens at small label fractions and persists under teacher noise.}
    \label{fig:rm_pairwise}
\end{figure}

\paragraph{Out-of-domain transfer.}
Figure~\ref{fig:rm_ood} confirms that the in-domain gains transfer to out-of-distribution reward-model benchmarks at label fraction $0.05$, $\rho = 0.10$. ABC-Align tops every evaluation: RewardBench overall ($+2.2$ pp over PP-SSL, $+8.0$ pp over labeled-only), RewardBench human-preference ($+3.9$ pp over PP-SSL), PPE human-preference ($+2.7$ pp), and PPE math best-of-$k$ ($+2.4$ pp). The math-reasoning proxy---the most out-of-distribution benchmark relative to Skywork's training mix---preserves the same ordering, indicating the gains come from better preference modeling rather than distributional overfitting to Skywork.

\begin{figure}[t]
    \centering
    \includegraphics[width=0.85\textwidth]{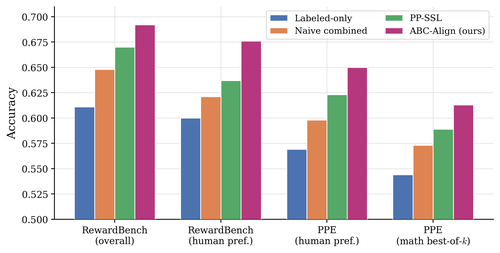}
    \caption{Out-of-domain real-world  reward-model benchmarks (label fraction $0.05$, $\rho = 0.10$). ABC-Align, trained on the Skywork-Reward-80K, achieves the best accuracy on RewardBench (overall and human-preference splits) and on PPE (human-preference and math best-of-$k$). Table~\ref{tab:rm_ood} in Appendix~\ref{app:tables} reports the same results numerically.}
    \label{fig:rm_ood}
\end{figure}

\paragraph{Online tuning vs.\ fixed mixtures.}
Figure~\ref{fig:rm_mixture} probes whether the online $(a_t, b_t)$ rule is doing work beyond a convenient default. We sweep fixed mixtures along two one-dimensional cuts: $\lambda \in [0, 1]$ for PP-SSL (i.e., $(a, b) = (1, \lambda)$) and $t \in [0, 1]$ for ABC-Align along the cut $(a, b) = (1{-}t, t)$.
The ABC-Align cut is nearly monotonically increasing in $t$, peaking around $0.69$ near $t \approx 0.8$---direct evidence that $a < 1$ is strictly helpful here and that the bias decoupling of Lemma~\ref{lem:bias} is practically exploitable. The joint corrected-norm rule (Figure~\ref{fig:rm_mixture}, dashed magenta) lands near this peak at $\sim 0.68$, showing that ABC-Align can robustly adapt to the given problem instance.
The PP-SSL cut, restricted to $a = 1$, peaks at only $\sim 0.67$ near $\lambda \approx 0.6$---below the best ABC-Align fixed mixture---and its online rule (dotted green) underperforms its own best fixed $\lambda$ at $\sim 0.64$. The adaptive ABC-Align rule thus not only accesses a better frontier than any $a = 1$ restriction can reach, but delivers it stably; PP-SSL's online variance-minimization, by contrast, drifts away from its own optimum.

\begin{figure}[t]
    \centering
    \includegraphics[width=0.65\textwidth]{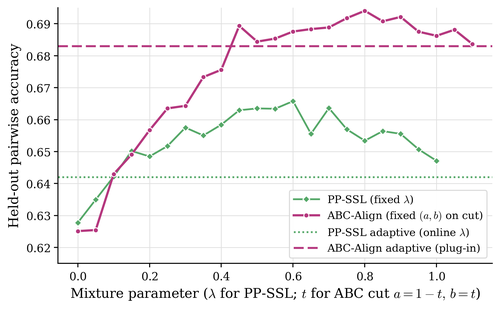}
    \caption{Fixed-mixture ablation at label fraction $0.05$, $\rho = 0.10$. Solid lines: held-out pairwise accuracy for fixed PP-SSL $\lambda$ (green) and fixed ABC-Align $(a, b) = (1{-}t, t)$ (magenta). Dashed/dotted: online joint corrected-norm ABC-Align (dashed magenta) and online AdaGrad PP-SSL (dotted green). ABC-Align's best fixed cut exceeds PP-SSL's best fixed $\lambda$, and ABC-Align's adaptive rule tracks the peak better.}
    \label{fig:rm_mixture}
\end{figure}

\paragraph{Real biased teachers.}
The synthetic flip-rate protocol above gives a controlled but unbiased teacher (label noise is symmetric across classes).
To address asymmetric/structural bias, we additionally evaluate on a \emph{real biased teacher}: pseudo labels on the unlabeled pool are produced by a smaller off-the-shelf reward model (Llama-3.2-1B-Instruct score head, trained on a disjoint preference mix) rather than by flipping ground truth.
This teacher exhibits genuine systematic bias---its disagreements with human preference are concentrated on specific response styles (verbosity, formatting) rather than spread uniformly.
We measure \emph{teacher--human agreement rate} on the held-out split and report final pairwise accuracy at label fraction $0.05$.

\begin{table}[t]
    \centering
    \caption{Real biased teacher on Skywork-Reward-80K (Llama-3.2-3B-Instruct, frozen backbone, label fraction $0.05$). Teacher--human agreement is measured on the held-out split. ABC-Align improves over the next-best baseline by $\Delta$ (rightmost column) at every teacher quality, including at chance-level agreement, where Naive combined falls below labeled-only.}
    \label{tab:rm_real_teacher}
    \begin{tabular}{lcccccc}
        \toprule
        Teacher & Agreement & Labeled-only & Naive comb. & PP-SSL & ABC-Align & $\Delta$ \\
        \midrule
        Llama-3.2-1B RM   & $0.82$ & $0.66$ & $0.71$ & $0.72$ & $\mathbf{0.76}$ & $+0.04$ \\
        Qwen2.5-0.5B RM   & $0.74$ & $0.66$ & $0.68$ & $0.71$ & $\mathbf{0.78}$ & $+0.07$ \\
        GPT-2-large RM    & $0.66$ & $0.66$ & $0.64$ & $0.70$ & $\mathbf{0.73}$ & $+0.03$ \\
        Random-init head  & $0.50$ & $0.66$ & $0.55$ & $0.66$ & $\mathbf{0.71}$ & $+0.05$ \\
        \bottomrule
    \end{tabular}
\end{table}

Table~\ref{tab:rm_real_teacher} shows that ABC-Align retains its advantage under a structurally biased teacher: at moderate teacher quality (agreement $\approx 0.66$--$0.82$) it dominates every baseline, and as teacher--human agreement degrades toward chance, ABC-Align's adaptive $(a_t, b_t)$ shifts mass back toward the labeled-only estimator and avoids the collapse that Naive combined exhibits.
PP-SSL is robust to the bias direction but, as in the synthetic-noise sweep, leaves variance reduction unrealized when the teacher is reliable.

\subsection{DPO-based policy alignment}
\label{sec:exp_dpo}

To demonstrate that ABC-Align extends beyond reward modeling, we apply it to direct preference optimization (DPO)~\citep{Rafailov2023DirectPO}, where the gradient estimator operates on the full policy parameters rather than a scalar score head.

\paragraph{Setup.}
We use Llama-3.2-3B~\citep{Dubey2024TheL3} as the base policy and fine-tune all parameters with the DPO loss~\eqref{eq:dpo} ($\beta_{\mathrm{kl}} = 0.1$).
The reference policy $\pi_{\mathrm{ref}}$ is frozen from the same pretrained checkpoint.
Training data comes from the UltraFeedback Binarized release, a curated collection of ${\sim}60{,}000$ multi-turn preference pairs spanning diverse instruction-following categories.
We shuffle the dataset once and partition into a small human-labeled subset and a large unlabeled pool, sweeping the labeled fraction over $\{0.05, 0.10, 0.30\}$.
On labeled examples the ground-truth preference is retained; on unlabeled examples the teacher preference is simulated by flipping the ground-truth label with probability $\rho = 0.10$, matching the controlled teacher-noise protocol of~\Secref{sec:exp_rm}.
We reserve $2{,}000$ examples as a held-out evaluation set.

\paragraph{Training and full-model adaptation.}
All methods share the same architecture, tokenizer, and AdamW optimizer; only the gradient estimator differs.
The key practical point is that, unlike the score-head reward model, DPO fine-tunes millions of policy parameters, yet the controller adds at most one full-model gradient buffer on top of the standard DPO step: the symmetric split-batch $\widehat H_t$ of~\eqref{eq:H_hat} is replaced by the one-sided streaming variant $\widehat H_t^{\text{stream}} := -\inner{g_A}{d_B}$ (still conditionally unbiased, $\E_t[\widehat H_t^{\text{stream}}] = H_t$), and every scalar primitive feeding $\ell_t$ is an aggregate dot product computable with a single all-reduce under FSDP/ZeRO, with no per-example gradient vectors materialized.
ABC-Align therefore touches $N + 2n$ examples per step (the teacher gradient on both the unlabeled and the labeled batch, plus the human gradient on the labeled batch), against $N + n$ for Naive combined SGD, which passes over the same pooled data once. The dominant cost ratio versus Naive combined SGD is thus $(N + 2n)/(N + n) \approx 1$ when $N \gg n$.
Optimizer hyperparameters, batch sizes, sequence lengths, and the full streaming-controller construction are deferred to Appendix~\ref{app:dpo_details}.

\paragraph{Baselines and evaluation.}
We compare the same four paradigms as in \Secref{sec:exp_rm}: \textbf{Labeled-only}, \textbf{Naive combined}, \textbf{PP-SSL}, and \textbf{ABC-Align}.
Since DPO produces a generative policy rather than a reward model, evaluation differs from the RM setting.
We report held-out DPO preference accuracy---the fraction of held-out pairs on which the policy assigns higher implicit reward to the preferred response---as the primary in-training diagnostic.

\paragraph{Results.}
Figure~\ref{fig:dpo_sweep} reports held-out DPO preference accuracy across label fractions $\{0.05, 0.10, 0.30\}$ at $\rho = 0.10$. ABC-Align is strongest at every label fraction, reaching $0.72$ at label fraction $0.05$ and $0.75$ at $0.30$. PP-SSL follows with a consistent $2$--$3$ point gap ($0.69, 0.71, 0.72$), above Naive combined ($0.67, 0.68, 0.70$) and labeled-only ($0.64, 0.66, 0.68$).
At the smallest label fraction, ABC-Align with $5\%$ labels already matches the accuracy labeled-only achieves with $30\%$ labels, indicating that the variance reduction from the two-parameter family translates directly into label efficiency in the full-policy setting.
Unlike the frozen score head of \Secref{sec:exp_rm}, DPO trains every policy parameter, yet the controller never forms a per-example gradient: it drives $(a_t, b_t)$ from aggregate dot products alone.
The consistent ordering across label fractions shows that this cheaper probe loses nothing at full-policy scale, and the $O(T^{-1/2})$ regret bound of Proposition~\ref{prop:joint_regret} applies directly to the realized $(a_t, b_t)$ trajectory.

\begin{figure}[t]
    \centering
    \includegraphics[width=0.5\textwidth]{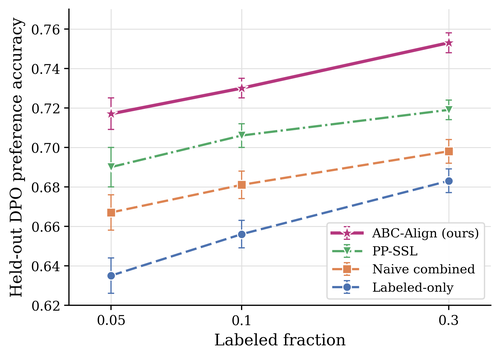} %
    \caption{DPO policy alignment on UltraFeedback Binarized (Llama-3.2-3B, $\rho = 0.10$). Held-out DPO preference accuracy across label fractions $\{0.05, 0.10, 0.30\}$. ABC-Align dominates at every label fraction; PP-SSL follows, above Naive combined and labeled-only.}
    \label{fig:dpo_sweep}
\end{figure}

\subsection{On-policy alignment with GRPO}
\label{sec:exp_grpo}

Finally, we instantiate ABC-Align on on-policy GRPO using the finite-sample plug-in controller of \Secref{sec:abc_grpo}.
This is the most challenging of the three alignment settings.
In the RM and DPO experiments the training data are drawn from a \emph{fixed} distribution, and the per-step oracle $(a_t^\star, b_t^\star)$ moves only because the parameters $\theta_t$ at which the gradient statistics are evaluated move.
Under on-policy GRPO the behavior policy is reset to $\pi_{\theta_t}$ at every iteration, so the rollouts themselves are drawn from a distribution that changes with $t$: the data distribution, not just the evaluation point, is non-stationary.
Both sources of drift act on the oracle, and the second is absent from the earlier settings.

\paragraph{Setup.}
We fine-tune Llama-3.2-3B with GRPO on a corpus of grade-school math (GSM8K-style) prompts with verifiable answers.
For each prompt we sample a rollout group of $G = 8$ completions.
Labeled prompts are scored by a trusted, verifiable correctness reward (the human-anchored source); unlabeled prompts are scored only by a teacher reward model, which is noisy relative to the verifiable signal.
At each step we estimate $(\widehat{\norm{B_{t,\mathrm{G}}}^2}, \widehat\sigma_{t,\mathrm{G}}^2, \widehat\sigma_{t,\mathrm{G},f}^2, \widehat C_{t,\mathrm{G}})$ from the labeled prompt-group batch via the split-batch plug-in of~\eqref{eq:grpo_splitbatch} and substitute into the closed form~\eqref{eq:oracle_ab} to obtain $(a_t, b_t)$.
We sweep the labeled fraction over $\{0.05, 0.10, 0.30\}$ at a fixed teacher-noise level, average over seeds, and report held-out GSM8K pass@$1$ on a disjoint evaluation split.
We compare the same four paradigms as in \Secref{sec:exp_rm}: \textbf{Labeled-only}, \textbf{Naive combined}, \textbf{PP-SSL}, and \textbf{ABC-Align}; full training details are deferred to Appendix~\ref{app:llm_details}.

\paragraph{Results.}
Figure~\ref{fig:grpo_sweep} reports held-out GSM8K accuracy across the three label fractions.
ABC-Align is strongest at every label fraction, improving from $0.721$ at fraction $0.05$ to $0.752$ at $0.30$, with PP-SSL trailing by $1.5$--$2$ points and Naive combined and labeled-only below that.
The advantage is largest at the smallest label fraction, where verifiable reward signal is scarcest and the variance reduction from the abundant teacher-scored pool matters most---the same label-efficiency pattern observed in the RM and DPO settings, now under a shifting on-policy gradient distribution.
This confirms that the finite-sample plug-in controller of \Secref{sec:abc_grpo} retains the bias--variance advantage even when the per-step oracle drifts, which is precisely the regime that motivated replacing the corrected-norm controller for GRPO.

\begin{figure}[t]
    \centering
    \includegraphics[width=0.62\textwidth]{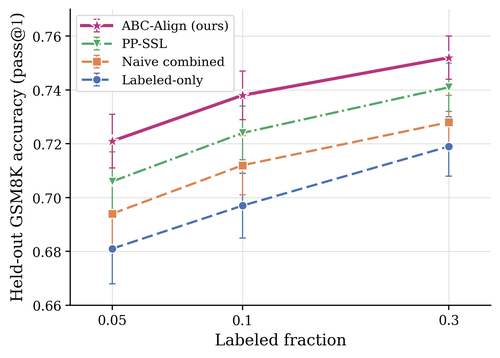}
    \caption{On-policy GRPO on GSM8K-style math (Llama-3.2-3B, $G = 8$ rollouts per prompt). Held-out GSM8K pass@$1$ across label fractions $\{0.05, 0.10, 0.30\}$. ABC-Align dominates at every label fraction; the gap is widest where verifiable labels are scarcest.}
    \label{fig:grpo_sweep}
\end{figure}

\section{Related Work}
\label{sec:related_work}
\paragraph{Preference-based post-training.}
Aligning LLMs to human preferences is standardly done by fitting a Bradley--Terry reward model~\citep{Bradley1952RankAO} to pairwise comparisons and optimizing a KL-regularized objective against it~\citep{Christiano2017DeepRL,Ziegler2019FineTuningLM,Stiennon2020LearningTS,Ouyang2022TrainingLM}, or by collapsing the two stages into a single supervised objective as in DPO~\citep{Rafailov2023DirectPO} and its variants~\citep{Azar2023AGT,Ethayarajh2024KTOMA}.
For reasoning models, outcome-based RL with group-relative advantages (GRPO)~\citep{Shao2024DeepSeekMathPT} and its descendants~\citep{Yu2025DAPOAO,Liu2025UnderstandingR1ZeroLT,Zheng2025GroupSP} have become the dominant recipe.
All three families are bottlenecked by the same resource---human-annotated ground truth---which is expensive, slow, and noisy~\citep{Casper2023OpenPA,Zhang2024DivergingPW,Gao2024ImpactOP}.
ABC-Align is agnostic to which objective is used: it modifies only the gradient \emph{estimator}, and we instantiate it for all three (\Secref{sec:abc_align}).

\paragraph{Pseudo labels for alignment, and why they are not free.}
The dominant response to label scarcity is to substitute a model for the annotator: RLAIF and constitutional-style pipelines label preferences with a strong LLM~\citep{Bai2022ConstitutionalAH,Lee2023RLAIFVR}, self-rewarding methods let the policy score its own generations~\citep{Yuan2024SelfRewardingLM,Wu2024MetaRewardingLM}, synthetic-preference pipelines mine or construct comparison pairs~\citep{Pace2024WestofNSP,Liu2025SkyworkRewardV2SP}, and weak-to-strong approaches supervise a strong student with a weak teacher~\citep{Burns2023WeaktoStrongGE,Tao2024YourWL}.
Such pseudo labels are abundant, but they are systematically---not randomly---wrong.
LLM judges exhibit reproducible verbosity, position, and self-preference biases~\citep{Zheng2023JudgingLW,Dubois2024LengthControlledAA,Ye2024JusticeOP,Shi2024JudgingTJ,Wataoka2024SelfPreferenceBI}, are sycophantic~\citep{Wei2023SimpleSD,Wang2023LargeLM}, and leak preferences when judge and generator share a lineage~\citep{Li2025PreferenceLA}.
Training on them degrades alignment in proportion to the noise rate~\citep{Gao2024ImpactOP}, saturates as the generation--verification gap closes~\citep{Song2024MindTG,Huang2024SelfImprovementIL}, and collapses without a real-data anchor~\citep{Shumailov2023TheCO,Gerstgrasser2024IsMC}.
The weak-to-strong literature makes the ceiling explicit: the achievable gain is governed by the student's misfit to the weak labels~\citep{Charikar2024QuantifyingTG}, and refining against a weak labeler alone carries irreducible error~\citep{Somerstep2024ATL}.
This is exactly the regime our estimator targets: teacher error is a persistent \emph{bias}, so more pseudo labels do not average it away, and removing it requires an external anchor.

\paragraph{Prediction-powered inference.}
PPI~\citep{Angelopoulos2023PredictionpoweredI} anchors pseudo labels to a small trusted set through a zero-mean rectifier, yielding valid inference \emph{regardless} of pseudo-label quality; PPI++~\citep{Angelopoulos2023PPIEP} tunes a scalar power parameter so that the estimator never underperforms the labeled-only baseline, and follow-up work strengthens the design via cross-fitting~\citep{Zrnic2023CrossPredictionpoweredI}, active labeling~\citep{Zrnic2024ActiveSI,chen2025surrogate,kluger2026m,sfyrakirevisiting},  stratification~\citep{Fisch2024StratifiedPI}, bootstrapping \cite{zrnic2024note,kluger2025prediction}, and recalibration \cite{ji2025predictions}. In particular, the recalibration approach attempts to unify PPI with the classical semiparametric theory, leading to a series of work at the intersection \cite{chen2025unified,testa2025semiparametric,carlson2025unifying,van2026calibeating,emmenegger2026prediction}. \citet{song2026demystifying} provide  a comprehensive survey of  strengths and weaknesses of the recent advances in PPI.

Applied to LLMs,  this line has so far lived almost entirely at \emph{evaluation} time, unlike ABC-Align: debiasing autorater-based benchmarks~\citep{Boyeau2024AutoEvalDR,Fisch2024StratifiedPI}, drawing valid conclusions from unconfident LLM annotations~\citep{Gligoric2024CanUL}, debiasing LLMs'
decisions on moral dilemmas 
\cite{broska2025mixed}, auditing LLMs on election-related information \cite{emmenegger2026prediction}, 
and unbiased win-rate estimation from mixed human/synthetic feedback~\citep{Zhou2025AcceleratingUL}.  
Two recent results sharpen the limits of this program and directly motivate ours.
\citet{Dorner2024LimitsTS} prove that when the judge is no more accurate than the model being judged, \emph{no} unbiased debiasing scheme can cut the number of required ground-truth labels by more than a factor of two, however many pseudo labels are available; and \citet{Mani2025NoFL} show non-asymptotically that PPI++ can be strictly \emph{worse} than using the labeled data alone unless pseudo-label correlation clears an $n$-dependent threshold.
Both barriers are statements about the unbiased class; ABC-Align breaks them by leaving it.

\paragraph{Prediction-powered training.}
Closest to us is work carrying PPI from estimation to iterative optimization.
PP-SSL~\citep{Shoham2025PredictionPoweredSL} forms an unbiased prediction-powered \emph{gradient} and tunes its single interpolation parameter online; PPI-SVRG~\citep{Ao2026PPISVRGUP} identifies PPI's rectifier with the SVRG control variate and inherits an error floor set by prediction uncertainty; and the doubly robust self-training loss of \citet{Zhu2023DoublyRS}---descended from doubly robust estimation in causal inference---is the fixed-$\lambda$ member of the same family.
Each of these estimators is unbiased by construction, which is precisely what pins it to the $\Theta(1/n)$ variance floor of Corollary~\ref{cor:variance_floor}.
Our estimator introduces a second parameter that relaxes unbiasedness, converting that floor into an explicit, adaptively controlled bias--variance trade-off.
The semiparametric strand of alignment theory is complementary rather than competing: \citet{Xu2025DoublyRA} and \citet{Ye2025RobustRF} seek consistency under model misspecification and variance reduction from unlabeled responses, not a deliberate bias budget.

\paragraph{Semi-supervised and noise-robust preference learning.}
A parallel line keeps the estimator unbiased and instead cleans the labels.
Robust DPO~\citep{Chowdhury2024ProvablyRD} debiases the DPO loss under a \emph{known} symmetric flip rate; semi-supervised reward modeling~\citep{He2024SemiSupervisedRM} and semi-supervised preference optimization~\citep{Lee2025SemiSupervisedPO} pseudo-label with confidence thresholds, inheriting the selection heuristics of classical SSL~\citep{Lee2013PseudoLabelT,Sohn2020FixMatchSS,Zhang2021FlexMatchBS}; and hybrid pipelines route instances between human and LLM annotators to spend a fixed human budget well~\citep{Miranda2024HybridPL,Xu2025RLTHFTH}.
These methods either assume a noise model we do not require, or \emph{trust} pseudo labels after filtering rather than \emph{correcting} them against ground truth.
Our human-labeled set plays a different role: it is not extra training signal to be augmented, but a measuring instrument for the teacher's gradient bias.

\paragraph{Biased gradient estimators.}
That a biased estimator can dominate an unbiased one in mean-squared error is well established in stochastic optimization: biased variance-reduced methods attain lower gradient MSE and faster rates than their unbiased counterparts~\citep{Driggs2019OnBS}, and convergence under biased oracles is standard for both SGD~\citep{Ajalloeian2020OnTC,Demidovich2023AGT} and adaptive methods~\citep{Surendran2024NonasymptoticAO}.
In RLHF, control variates appear as variance-reducing \emph{baselines}~\citep{Ahmadian2024BackTB}, including shrinkage baselines that knowingly trade bias for variance in the GRPO advantage~\citep{Zeng2025ShrinkingTV}; GRPO's group-relative advantage is itself a biased gradient estimator~\citep{Liu2025UnderstandingR1ZeroLT,Zhou2026DemystifyingGR}.
What neither literature provides is a bias that is \emph{measured against trusted human labels and controlled online}, which is what ABC-SSL contributes.

\section{Conclusion}

Prediction powered inference (PPI) techniques have emerged as a powerful tool for semi-supervised learning when ($i$) we have access to a noisy and biased teacher model and ($ii$) the number of labeled examples $n$ is smaller than the number of unlabeled examples $N$.
The standard PPI-style model training methods like PP-SSL use a {\em single parameter} to control variance while keeping the gradient estimate {\em unbiased}.
Such unbiasedness forces the resulting estimate to suffer from a variance floor that scales as $\Theta(1/n)$.
To this end, we introduced adaptive bias control for semi-supervised learning (ABC-SSL, Algorithm~\ref{alg:abcssl}), an iterative algorithm with a {\em two-parameter} gradient estimator that improve the bias-variance trade-off beyond the $\Theta(1/n)$ variance floor by admitting {\em a controlled bias} decoupled from variance reduction. Our analysis (Corollary~\ref{cor:end_to_end}) gives a joint regret guarantee in $(a_t, b_t)$ via the corrected-norm surrogate $\ell_t$ and the model parameter $\theta$
(Section~\ref{sec:plugin}), where the parameters $(a_t, b_t)$ are adaptively learned to control the bias--variance trade-off.
This is shown to achieve a strict improvement over the best statically tuned PP-SSL estimator at rate $O(T^{-1/2})$ for a broad range of parameters.

Empirically, our goal is to improve LLM alignment by better utilizing abundant unlabeled text together with easily accessible teacher models.
To empirically test the gain of our new approach, we present three instantiations of ABC-Align, i.e., the ABC-SSL algorithm applied to solve LLM alignment problems: reward model training for RLHF, offline preference optimization using DPO, and on-policy RL with GRPO.
We show that ABC-Align dominates common baselines of labeled-only, naively combined, and PP-SSL across all three scenarios.

\section*{Acknowledgement}

Frankel is supported in part by the NSF graduate research fellowship program;
Ratliff is supported in part by NSF award AF-2312775; Oh is  supported in part by NSF awards CNS-2112471, IIS-2229876, CCF-2505865, and DMS-2502281.

\bibliographystyle{plainnat}
\bibliography{references}

\begin{thebibliography}{106}
\providecommand{\natexlab}[1]{#1}
\providecommand{\url}[1]{\texttt{#1}}
\expandafter\ifx\csname urlstyle\endcsname\relax
  \providecommand{\doi}[1]{doi: #1}\else
  \providecommand{\doi}{doi: \begingroup \urlstyle{rm}\Url}\fi

\bibitem[Ahmadian et~al.(2024)Ahmadian, Cremer, Gall{\'e}, Fadaee, Kreutzer,
  Pietquin, {\"U}st{\"u}n, and Hooker]{Ahmadian2024BackTB}
Arash Ahmadian, Chris Cremer, Matthias Gall{\'e}, Marzieh Fadaee, Julia
  Kreutzer, Olivier Pietquin, Ahmet {\"U}st{\"u}n, and Sara Hooker.
\newblock Back to basics: Revisiting {REINFORCE}-style optimization for
  learning from human feedback in {LLM}s.
\newblock In \emph{Proceedings of the 62nd Annual Meeting of the Association
  for Computational Linguistics (ACL)}, 2024.
\newblock arXiv:2402.14740.

\bibitem[Ajalloeian and Stich(2020)]{Ajalloeian2020OnTC}
Ahmad Ajalloeian and Sebastian~U. Stich.
\newblock On the convergence of {SGD} with biased gradients.
\newblock \emph{arXiv preprint arXiv:2008.00051}, 2020.

\bibitem[Angelopoulos et~al.(2023{\natexlab{a}})Angelopoulos, Bates, Fannjiang,
  Jordan, and Zrnic]{Angelopoulos2023PredictionpoweredI}
Anastasios~Nikolas Angelopoulos, Stephen Bates, Clara Fannjiang, Michael~I.
  Jordan, and Tijana Zrnic.
\newblock Prediction-powered inference.
\newblock \emph{Science}, 382:\penalty0 669 -- 674, 2023{\natexlab{a}}.

\bibitem[Angelopoulos et~al.(2023{\natexlab{b}})Angelopoulos, Duchi, and
  Zrnic]{Angelopoulos2023PPIEP}
Anastasios~Nikolas Angelopoulos, John~C. Duchi, and Tijana Zrnic.
\newblock {PPI++}: Efficient prediction-powered inference.
\newblock \emph{ArXiv}, abs/2311.01453, 2023{\natexlab{b}}.

\bibitem[Ao et~al.(2026)Ao, Chen, Liu, Simchi-Levi, and Sun]{Ao2026PPISVRGUP}
Ruicheng Ao, Hongyu Chen, Haoyang Liu, David Simchi-Levi, and Will~Wei Sun.
\newblock {PPI-SVRG}: Unifying prediction-powered inference and variance
  reduction for semi-supervised optimization.
\newblock \emph{ArXiv}, abs/2601.21470, 2026.

\bibitem[Arazo et~al.(2019)Arazo, Ortego, Albert, O'Connor, and
  McGuinness]{Arazo2019PseudoLabelingAC}
Eric Arazo, Diego Ortego, Paul Albert, Noel~E. O'Connor, and Kevin McGuinness.
\newblock Pseudo-labeling and confirmation bias in deep semi-supervised
  learning.
\newblock \emph{2020 International Joint Conference on Neural Networks
  (IJCNN)}, pages 1--8, 2019.

\bibitem[Azar et~al.(2023)Azar, Rowland, Piot, Guo, Calandriello, Valko, and
  Munos]{Azar2023AGT}
Mohammad~Gheshlaghi Azar, Mark Rowland, Bilal Piot, Daniel Guo, Daniele
  Calandriello, Michal Valko, and R{\'e}mi Munos.
\newblock A general theoretical paradigm to understand learning from human
  preferences.
\newblock \emph{ArXiv}, abs/2310.12036, 2023.

\bibitem[Bai et~al.(2022)Bai, Kadavath, Kundu, Askell, Kernion, Jones, Chen,
  Goldie, Mirhoseini, McKinnon, Chen, Olsson, Olah, Hernandez, Drain, Ganguli,
  Li, Tran-Johnson, Perez, Kerr, Mueller, Ladish, Landau, Ndousse,
  Luko{\v{s}}i{\=u}t\.{e}, Lovitt, Sellitto, Elhage, Schiefer, Mercado,
  Dassarma, Lasenby, Larson, Ringer, Johnston, Kravec, Showk, Fort, Lanham,
  Telleen-Lawton, Conerly, Henighan, Hume, Bowman, Hatfield-Dodds, Mann,
  Amodei, Joseph, McCandlish, Brown, and Kaplan]{Bai2022ConstitutionalAH}
Yuntao Bai, Saurav Kadavath, Sandipan Kundu, Amanda Askell, John Kernion, Andy
  Jones, Anna Chen, Anna Goldie, Azalia Mirhoseini, Cameron McKinnon, Carol
  Chen, Catherine Olsson, Chris Olah, Danny Hernandez, Dawn Drain, Deep
  Ganguli, Dustin Li, Eli Tran-Johnson, E~Perez, Jamie Kerr, Jared Mueller,
  Jeffrey Ladish, J~Landau, Kamal Ndousse, Kamil\.{e} Luko{\v{s}}i{\=u}t\.{e},
  Liane Lovitt, Michael Sellitto, Nelson Elhage, Nicholas Schiefer, Noem'i
  Mercado, Nova Dassarma, Robert Lasenby, Robin Larson, Sam Ringer, Scott
  Johnston, Shauna Kravec, Sheer~El Showk, Stanislav Fort, Tamera Lanham,
  Timothy Telleen-Lawton, Tom Conerly, Thomas Henighan, Tristan Hume, Sam
  Bowman, Zac Hatfield-Dodds, Benjamin Mann, Dario Amodei, Nicholas Joseph, Sam
  McCandlish, Tom~B. Brown, and Jared Kaplan.
\newblock Constitutional {AI}: Harmlessness from {AI} feedback.
\newblock \emph{ArXiv}, abs/2212.08073, 2022.

\bibitem[Borkar(1997)]{Borkar1997StochasticAW}
Vivek~S. Borkar.
\newblock Stochastic approximation with two time scales.
\newblock \emph{Systems \& Control Letters}, 29\penalty0 (5):\penalty0
  291--294, 1997.

\bibitem[Borkar(2008)]{Borkar2008StochasticAA}
Vivek~S. Borkar.
\newblock \emph{Stochastic Approximation: A Dynamical Systems Viewpoint}.
\newblock Cambridge University Press and Hindustan Book Agency, 2008.

\bibitem[Boyeau et~al.(2025)Boyeau, Angelopoulos, Yosef, Malik, and
  Jordan]{Boyeau2024AutoEvalDR}
Pierre Boyeau, Anastasios~N. Angelopoulos, Nir Yosef, Jitendra Malik, and
  Michael~I. Jordan.
\newblock Autoeval done right: Using synthetic data for model evaluation.
\newblock In \emph{International Conference on Machine Learning}, 2025.
\newblock arXiv:2403.07008.

\bibitem[Bradley and Terry(1952)]{Bradley1952RankAO}
Ralph~Allan Bradley and Milton~E. Terry.
\newblock Rank analysis of incomplete block designs: I. the method of paired
  comparisons.
\newblock \emph{Biometrika}, 39:\penalty0 324, 1952.

\bibitem[Broska et~al.(2025)Broska, Howes, and van Loon]{broska2025mixed}
David Broska, Michael Howes, and Austin van Loon.
\newblock The mixed subjects design: Treating large language models as
  potentially informative observations.
\newblock \emph{Sociological Methods \& Research}, 54\penalty0 (3):\penalty0
  1074--1109, 2025.

\bibitem[Burns et~al.(2023)Burns, Izmailov, Kirchner, Baker, Gao,
  Aschenbrenner, Chen, Ecoffet, Joglekar, Leike, Sutskever, and
  Wu]{Burns2023WeaktoStrongGE}
Collin Burns, Pavel Izmailov, Jan~Hendrik Kirchner, Bowen Baker, Leo Gao,
  Leopold Aschenbrenner, Yining Chen, Adrien Ecoffet, Manas Joglekar, Jan
  Leike, Ilya Sutskever, and Jeffrey Wu.
\newblock Weak-to-strong generalization: Eliciting strong capabilities with
  weak supervision.
\newblock \emph{ArXiv}, abs/2312.09390, 2023.

\bibitem[Carlson and Dell(2025)]{carlson2025unifying}
Jacob Carlson and Melissa Dell.
\newblock A unifying framework for robust and efficient inference with
  unstructured data.
\newblock \emph{arXiv preprint arXiv:2505.00282}, 2025.

\bibitem[Casper et~al.(2023)Casper, Davies, Shi, Gilbert, Scheurer, Rando,
  Freedman, Korbak, Lindner, Freire, Wang, Marks, S{\'e}gerie, Carroll, Peng,
  Christoffersen, Damani, Slocum, Anwar, Siththaranjan, Nadeau, Michaud, Pfau,
  Krasheninnikov, Chen, di~Langosco, Hase, Biyik, Dragan, Krueger, Sadigh, and
  Hadfield-Menell]{Casper2023OpenPA}
Stephen Casper, Xander Davies, Claudia Shi, Thomas~Krendl Gilbert, J'er'emy
  Scheurer, Javier Rando, Rachel Freedman, Tomasz Korbak, David Lindner,
  Pedro~J Freire, Tony Wang, Samuel Marks, Charbel-Rapha{\"e}l S{\'e}gerie,
  Micah Carroll, Andi Peng, Phillip J.~K. Christoffersen, Mehul Damani, Stewart
  Slocum, Usman Anwar, Anand Siththaranjan, Max Nadeau, Eric~J. Michaud, Jacob
  Pfau, Dmitrii Krasheninnikov, Xin Chen, Lauro~Langosco di~Langosco, Peter
  Hase, Erdem Biyik, Anca~D. Dragan, David Krueger, Dorsa Sadigh, and Dylan
  Hadfield-Menell.
\newblock Open problems and fundamental limitations of reinforcement learning
  from human feedback.
\newblock \emph{Transactions on Machine Learning Research}, 2023.

\bibitem[Charikar et~al.(2024)Charikar, Pabbaraju, and
  Shiragur]{Charikar2024QuantifyingTG}
Moses Charikar, Chirag Pabbaraju, and Kirankumar Shiragur.
\newblock Quantifying the gain in weak-to-strong generalization.
\newblock In \emph{Advances in Neural Information Processing Systems}, 2024.
\newblock arXiv:2405.15116.

\bibitem[Chen et~al.(2022)Chen, Jiang, Wang, Wan, Long, and
  Wang]{Chen2022DebiasedST}
Baixu Chen, Junguang Jiang, Ximei Wang, Pengfei Wan, Mingsheng Long, and
  Jianmin Wang.
\newblock Debiased self-training for semi-supervised learning.
\newblock In \emph{Advances in Neural Information Processing Systems},
  volume~35, 2022.

\bibitem[Chen et~al.(2025{\natexlab{a}})Chen, Wang, Lumley, Dai, and
  Chen]{chen2025surrogate}
Jianmin Chen, Huiyuan Wang, Thomas Lumley, Xiaowu Dai, and Yong Chen.
\newblock Surrogate-powered inference: Regularization and adaptivity.
\newblock \emph{arXiv preprint arXiv:2512.21826}, 2025{\natexlab{a}}.

\bibitem[Chen et~al.(2025{\natexlab{b}})Chen, McCormick, Mukherjee, and
  Wu]{chen2025unified}
Xingran Chen, Tyler McCormick, Bhramar Mukherjee, and Zhenke Wu.
\newblock A unified framework for inference with general missingness patterns
  and machine learning imputation.
\newblock \emph{arXiv preprint arXiv:2508.15162}, 2025{\natexlab{b}}.

\bibitem[Chen et~al.(2024)Chen, Deng, Yuan, Ji, and Gu]{Chen2024SelfPlayFC}
Zixiang Chen, Yihe Deng, Huizhuo Yuan, Kaixuan Ji, and Quanquan Gu.
\newblock Self-play fine-tuning converts weak language models to strong
  language models.
\newblock In \emph{International Conference on Machine Learning}, 2024.

\bibitem[Chowdhury et~al.(2024)Chowdhury, Kini, and
  Natarajan]{Chowdhury2024ProvablyRD}
Sayak~Ray Chowdhury, Anush Kini, and Nagarajan Natarajan.
\newblock Provably robust {DPO}: Aligning language models with noisy feedback.
\newblock \emph{ArXiv}, abs/2403.00409, 2024.

\bibitem[Christiano et~al.(2017)Christiano, Leike, Brown, Martic, Legg, and
  Amodei]{Christiano2017DeepRL}
Paul~Francis Christiano, Jan Leike, Tom~B. Brown, Miljan Martic, Shane Legg,
  and Dario Amodei.
\newblock Deep reinforcement learning from human preferences.
\newblock \emph{Advances in neural information processing systems}, 30, 2017.

\bibitem[Cutkosky and Orabona(2019)]{CutkoskyOrabona2019MomentumBV}
Ashok Cutkosky and Francesco Orabona.
\newblock Momentum-based variance reduction in non-convex {SGD}.
\newblock In \emph{Advances in Neural Information Processing Systems},
  volume~32, 2019.

\bibitem[Demidovich et~al.(2023)Demidovich, Malinovsky, Sokolov, and
  Richt{\'a}rik]{Demidovich2023AGT}
Yury Demidovich, Grigory Malinovsky, Igor Sokolov, and Peter Richt{\'a}rik.
\newblock A guide through the zoo of biased sgd.
\newblock \emph{ArXiv}, abs/2305.16296, 2023.

\bibitem[Doan(2021)]{Doan2021NonlinearTT}
Thinh~T. Doan.
\newblock Nonlinear two-time-scale stochastic approximation: Convergence and
  finite-time performance.
\newblock In \emph{Proceedings of the 3rd Conference on Learning for Dynamics
  and Control (L4DC)}, volume 144 of \emph{Proceedings of Machine Learning
  Research}, pages 47--57. PMLR, 2021.

\bibitem[Doan(2026)]{Doan2024FastNonlinearTT}
Thinh~T Doan.
\newblock Fast nonlinear two-time-scale stochastic approximation: Achieving
  ${O}(1/k) $ finite-sample complexity.
\newblock \emph{IEEE Transactions on Automatic Control}, 71\penalty0
  (1):\penalty0 230--242, 2026.

\bibitem[Dong et~al.(2024)Dong, Xiong, Pang, Wang, Zhao, Zhou, Jiang, Sahoo,
  Xiong, and Zhang]{Dong2024RLHFWF}
Hanze Dong, Wei Xiong, Bo~Pang, Haoxiang Wang, Han Zhao, Yingbo Zhou, Nan
  Jiang, Doyen Sahoo, Caiming Xiong, and Tong Zhang.
\newblock {RLHF} workflow: From reward modeling to online {RLHF}.
\newblock \emph{Transactions on Machine Learning Research}, 2024.
\newblock ISSN 2835-8856.

\bibitem[Dorner et~al.(2025)Dorner, Nastl, and Hardt]{Dorner2024LimitsTS}
Florian~E. Dorner, Vivian~Y. Nastl, and Moritz Hardt.
\newblock Limits to scalable evaluation at the frontier: {LLM} as judge won't
  beat twice the data.
\newblock In \emph{International Conference on Learning Representations}, 2025.
\newblock arXiv:2410.13341.

\bibitem[Driggs et~al.(2022)Driggs, Liang, and Sch{\"o}nlieb]{Driggs2019OnBS}
Derek Driggs, Jingwei Liang, and Carola-Bibiane Sch{\"o}nlieb.
\newblock On biased stochastic gradient estimation.
\newblock \emph{Journal of Machine Learning Research}, 23\penalty0
  (24):\penalty0 1--43, 2022.
\newblock arXiv:1906.01133.

\bibitem[Dubey et~al.(2024)Dubey, Jauhri, Pandey, Kadian, Al-Dahle, Letman,
  Mathur, Schelten, Yang, Fan, Goyal, Hartshorn, Yang, Mitra, Sravankumar,
  Korenev, Hinsvark, Rao, Zhang, Rodriguez, Gregerson, Spataru, Rozi{\`e}re,
  Biron, Tang, Chern, Caucheteux, Nayak, Bi, Marra, McConnell, Keller, Touret,
  Wu, Wong, Ferrer, Nikolaidis, Allonsius, Song, Pintz, Livshits, Esiobu,
  Choudhary, Mahajan, Garcia-Olano, Perino, Hupkes, Lakomkin, AlBadawy,
  Lobanova, Dinan, Smith, Radenovic, Zhang, Synnaeve, Lee, Anderson, Nail,
  Mialon, Pang, Cucurell, Nguyen, Korevaar, Xu, Touvron, Zarov, Ibarra,
  Kloumann, Misra, Evtimov, Copet, Lee, Geffert, Vranes, Park, Mahadeokar,
  Shah, van~der Linde, Billock, Hong, Lee, Fu, Chi, Huang, Liu, Wang, Yu,
  Bitton, Spisak, Park, Rocca, Johnstun, Saxe, Jia, Alwala, Upasani, Plawiak,
  Li, Heafield, Stone, El-Arini, Iyer, Malik, ley Chiu, Bhalla, Rantala-Yeary,
  van~der Maaten, Chen, Tan, Jenkins, Martin, Madaan, Malo, Blecher, Landzaat,
  de~Oliveira, Muzzi, hesh Pasupuleti, Singh, Paluri, Kardas, Oldham, Rita,
  Pavlova, Kambadur, Lewis, Si, Singh, Hassan, Goyal, Torabi, lay Bashlykov,
  Bogoychev, Chatterji, Duchenne, cCelebi, Alrassy, Zhang, Li, Vasi{\'c}, Weng,
  Bhargava, Dubal, Krishnan, Koura, Xu, He, Dong, Srinivasan, Ganapathy,
  Calderer, veira Cabral, Stojnic, Raileanu, Girdhar, Patel, Sauvestre, nie
  Polidoro, Sumbaly, Taylor, Silva, Hou, Wang, Hosseini, hana Chennabasappa,
  Singh, Bell, Kim, Edunov, Nie, Narang, Raparthy, Shen, Wan, Bhosale, Zhang,
  Vandenhende, Batra, Whitman, Sootla, Collot, Gururangan, Borodinsky, Herman,
  Fowler, Sheasha, Georgiou, Scialom, Speckbacher, Mihaylov, Xiao, Karn,
  Goswami, Gupta, Ramanathan, Kerkez, Gonguet, Do, Vogeti, Petrovic, Chu,
  Xiong, Fu, ney Meers, Martinet, Wang, Tan, Xie, Jia, Wang, Goldschlag, Gaur,
  Babaei, Wen, Song, Zhang, Li, Mao, Coudert, Yan, Chen, Papakipos, Singh,
  Grattafiori, Jain, Kelsey, Shajnfeld, Gangidi, Victoria, Goldstand, Menon,
  Sharma, Boesenberg, Vaughan, Baevski, Feinstein, Kallet, Sangani, Yunus,
  Lupu, Alvarado, Caples, Gu, Ho, Poulton, Ryan, Ramchandani, Franco, Saraf,
  Chowdhury, Gabriel, Bharambe, Eisenman, Yazdan, James, Maurer, Leonhardi,
  Huang, Loyd, de~Paola, Paranjape, Liu, Wu, Ni, Hancock, Wasti, Spence,
  Stojkovic, Gamido, Montalvo, Parker, Burton, Mejia, Wang, Kim, Zhou, Hu, Chu,
  Cai, Tindal, Feichtenhofer, Civin, Beaty, Kreymer, Li, Wyatt, Adkins, Xu,
  Testuggine, David, Parikh, Liskovich, Foss, Wang, Le, Holland, Dowling,
  Jamil, Montgomery, Presani, Hahn, Wood, Brinkman, Arcaute, Dunbar, Smothers,
  Sun, Kreuk, Tian, Ozgenel, Caggioni, Guzm{\'a}n, Kanayet, Seide, Florez,
  Schwarz, Badeer, Swee, Halpern, Thattai, Herman, Sizov, Zhang,
  Lakshminarayanan, Shojanazeri, Zou, Wang, Zha, Habeeb, Rudolph, Suk,
  Aspegren, Goldman, Molybog, Tufanov, Veliche, Gat, Weissman, Geboski, Kohli,
  Asher, Gaya, Marcus, Tang, Chan, Zhen, Reizenstein, Teboul, Zhong, Jin, Yang,
  Cummings, Carvill, Shepard, McPhie, Torres, Ginsburg, Wang, Wu, KamHou,
  Saxena, Prasad, Khandelwal, Zand, Matosich, Veeraraghavan, Michelena, Li,
  Huang, Chawla, Lakhotia, Huang, Chen, Garg, Lavender, Silva, Bell, Zhang,
  Guo, Yu, Moshkovich, Wehrstedt, Khabsa, Avalani, Bhatt, Tsimpoukelli, Mankus,
  Hasson, Lennie, Reso, Groshev, Naumov, Lathi, Keneally, Seltzer, Valko,
  Restrepo, Patel, Vyatskov, Samvelyan, Clark, Macey, Wang, Hermoso, Metanat,
  Rastegari, ish Bansal, Santhanam, Parks, White, ata Bawa, Singhal, Egebo,
  Usunier, Laptev, Dong, Zhang, Cheng, Chernoguz, Hart, Salpekar, Kalinli,
  Kent, Parekh, Saab, Balaji, dro Rittner, Bontrager, Roux, Doll{\'a}r,
  Zvyagina, Ratanchandani, Yuvraj, Liang, Alao, Rodriguez, Ayub, Murthy,
  Nayani, Mitra, Li, Hogan, Battey, Wang, han Maheswari, Howes, Rinott, Bondu,
  Datta, Chugh, Hunt, Dhillon, Sidorov, Pan, Verma, Yamamoto, Ramaswamy,
  Lindsay, Feng, Lin, Zha, Shankar, Zhang, Wang, Agarwal, Sajuyigbe, Chintala,
  Max, Chen, Kehoe, Satterfield, Govindaprasad, Gupta, Cho, Virk, Subramanian,
  Choudhury, Goldman, Remez, Glaser, Best, Kohler, Robinson, Li, Zhang,
  Matthews, Chou, Shaked, Vontimitta, Ajayi, Montanez, Mohan, Kumar, Mangla,
  Ionescu, Poenaru, Mihailescu, Ivanov, Li, Wang, Jiang, Bouaziz, Constable,
  Tang, Wang, Wu, Wang, Xia, Wu, Gao, Chen, Hu, Jia, Qi, Li, Zhang, Zhang, Adi,
  Nam, Wang, Hao, Qian, He, Rait, DeVito, Rosnbrick, Wen, Yang, and
  Zhao]{Dubey2024TheL3}
Abhimanyu Dubey, Abhinav Jauhri, Abhinav Pandey, Abhishek Kadian, Ahmad
  Al-Dahle, Aiesha Letman, Akhil Mathur, Alan Schelten, Amy Yang, Angela Fan,
  Anirudh Goyal, Anthony~S. Hartshorn, Aobo Yang, Archi Mitra, Archie
  Sravankumar, Artem Korenev, Arthur Hinsvark, Arun Rao, Aston Zhang, Aur'elien
  Rodriguez, Austen Gregerson, Ava Spataru, Baptiste Rozi{\`e}re, Bethany~M.
  Biron, Binh Tang, Bobbie Chern, Charlotte Caucheteux, Chaya Nayak, Chloe Bi,
  Chris Marra, Chris McConnell, Christian Keller, Christophe Touret, Chunyang
  Wu, Corinne Wong, Cristian~Canton Ferrer, Cyrus Nikolaidis, Damien Allonsius,
  Daniel~J. Song, Danielle Pintz, Danny Livshits, David Esiobu, Dhruv
  Choudhary, Dhruv Mahajan, Diego Garcia-Olano, Diego Perino, Dieuwke Hupkes,
  Egor Lakomkin, Ehab~A. AlBadawy, E~I Lobanova, Emily Dinan, Eric~Michael
  Smith, Filip Radenovic, Frank Zhang, Gabriel Synnaeve, Gabrielle Lee,
  Georgia~Lewis Anderson, Graeme Nail, Gr{\'e}goire Mialon, Guanglong Pang,
  Guillem Cucurell, Hailey Nguyen, Hannah Korevaar, Hu~Xu, Hugo Touvron, Iliyan
  Zarov, Imanol~Arrieta Ibarra, Isabel~M. Kloumann, Ishan Misra, Ivan Evtimov,
  Jade Copet, Jaewon Lee, Jan Geffert, Jana Vranes, Jason Park, Jay Mahadeokar,
  Jeet Shah, Jelmer van~der Linde, Jennifer Billock, Jenny Hong, Jenya Lee,
  Jeremy Fu, Jianfeng Chi, Jianyu Huang, Jiawen Liu, Jie Wang, Jiecao Yu,
  Joanna Bitton, Joe Spisak, Jongsoo Park, Joseph Rocca, Joshua Johnstun,
  Joshua Saxe, Ju-Qing Jia, Kalyan~Vasuden Alwala, K.~Upasani, Kate Plawiak,
  Keqian Li, Kenneth Heafield, Kevin~R. Stone, Khalid El-Arini, Krithika Iyer,
  Kshitiz Malik, Kuen ley Chiu, Kunal Bhalla, Lauren Rantala-Yeary, Laurens
  van~der Maaten, Lawrence Chen, Liang Tan, Liz Jenkins, Louis Martin, Lovish
  Madaan, Lubo Malo, Lukas Blecher, Lukas Landzaat, Luke de~Oliveira, Madeline
  Muzzi, Ma~hesh Pasupuleti, Mannat Singh, Manohar Paluri, Marcin Kardas,
  Mathew Oldham, Mathieu Rita, Maya Pavlova, Melissa Hall~Melanie Kambadur,
  Mike Lewis, Min Si, Mitesh~Kumar Singh, Mona Hassan, Naman Goyal, Narjes
  Torabi, Niko lay Bashlykov, Nikolay Bogoychev, Niladri~S. Chatterji, Olivier
  Duchenne, Onur cCelebi, Patrick Alrassy, Pengchuan Zhang, Pengwei Li, Petar
  Vasi{\'c}, Peter Weng, Prajjwal Bhargava, Pratik Dubal, Praveen Krishnan,
  Punit~Singh Koura, Puxin Xu, Qing He, Qingxiao Dong, Ragavan Srinivasan, Raj
  Ganapathy, Ramon Calderer, Ricardo~Sil veira Cabral, Robert Stojnic, Roberta
  Raileanu, Rohit Girdhar, Rohit Patel, Romain Sauvestre, Ron nie Polidoro,
  Roshan Sumbaly, Ross Taylor, Ruan Silva, Rui Hou, Rui Wang, Saghar Hosseini,
  Sa~hana Chennabasappa, Sanjay Singh, Sean Bell, Seohyun~Sonia Kim, Sergey
  Edunov, Shaoliang Nie, Sharan Narang, Sharath~Chandra Raparthy, Sheng Shen,
  Shengye Wan, Shruti Bhosale, Shun Zhang, Simon Vandenhende, Soumya Batra,
  Spencer Whitman, Sten Sootla, St{\'e}phane Collot, Suchin Gururangan, Sydney
  Borodinsky, Tamar Herman, Tara Fowler, Tarek Sheasha, Thomas Georgiou, Thomas
  Scialom, Tobias Speckbacher, Todor Mihaylov, Tong Xiao, Ujjwal Karn, Vedanuj
  Goswami, Vibhor Gupta, Vignesh Ramanathan, Viktor Kerkez, Vincent Gonguet,
  Virginie Do, Vish Vogeti, Vladan Petrovic, Weiwei Chu, Wenhan Xiong, Wenyin
  Fu, Whit ney Meers, Xavier Martinet, Xiaodong Wang, Xiaoqing~Ellen Tan,
  Xinfeng Xie, Xuchao Jia, Xuewei Wang, Yaelle Goldschlag, Yashesh Gaur,
  Yasmine Babaei, Yiqian Wen, Yiwen Song, Yuchen Zhang, Yue Li, Yuning Mao,
  Zacharie~Delpierre Coudert, Zhengxu Yan, Zhengxing Chen, Zoe Papakipos,
  Aaditya~K. Singh, Aaron Grattafiori, Abha Jain, Adam Kelsey, Adam Shajnfeld,
  Adi Gangidi, Adolfo Victoria, Ahuva Goldstand, Ajay Menon, Ajay Sharma, Alex
  Boesenberg, Alex Vaughan, Alexei Baevski, Allie Feinstein, Amanda Kallet,
  Amit Sangani, Anam Yunus, Andrei Lupu, Andres Alvarado, Andrew Caples, Andrew
  Gu, Andrew Ho, Andrew Poulton, Andrew Ryan, Ankit Ramchandani, Annie Franco,
  Aparajita Saraf, Arkabandhu Chowdhury, Ashley Gabriel, Ashwin Bharambe, Assaf
  Eisenman, Azadeh Yazdan, Beau James, Ben Maurer, Benjamin Leonhardi,
  Po-Yao~(Bernie) Huang, Beth Loyd, Beto de~Paola, Bhargavi Paranjape, Bing
  Liu, Bo~Wu, Boyu Ni, Braden Hancock, Bram Wasti, Brandon Spence, Brani
  Stojkovic, Brian Gamido, Britt Montalvo, Carl Parker, Carly Burton, Catalina
  Mejia, Changhan Wang, Changkyu Kim, Chao Zhou, Chester Hu, Ching-Hsiang Chu,
  Chris Cai, Chris Tindal, Christoph Feichtenhofer, Damon Civin, Dana Beaty,
  Daniel Kreymer, Shang-Wen Li, Danny Wyatt, David Adkins, David Xu, Davide
  Testuggine, Delia David, Devi Parikh, Diana Liskovich, Didem Foss, Dingkang
  Wang, Duc Le, Dustin Holland, Edward Dowling, Eissa Jamil, Elaine Montgomery,
  Eleonora Presani, Emily Hahn, Emily Wood, Erik Brinkman, Esteban Arcaute,
  Evan Dunbar, Evan Smothers, Fei Sun, Felix Kreuk, Feng Tian, Firat Ozgenel,
  Francesco Caggioni, Francisco~(Paco) Guzm{\'a}n, Frank~J. Kanayet, Frank
  Seide, Gabriela~Medina Florez, Gabriella Schwarz, Gada Badeer, Georgia Swee,
  Gil Halpern, Govind Thattai, Grant Herman, Grigory Sizov, Guangyi Zhang, Guna
  Lakshminarayanan, Hamid Shojanazeri, Han Zou, Hannah Wang, Han Zha, Haroun
  Habeeb, Harrison Rudolph, Helen Suk, Henry Aspegren, Hunter Goldman, Igor
  Molybog, Igor Tufanov, Irina-Elena Veliche, Itai Gat, Jake Weissman, James
  Geboski, James Kohli, Japhet Asher, Jean-Baptiste Gaya, Jeff Marcus, Jeff
  Tang, Jennifer Chan, Jenny Zhen, Jeremy Reizenstein, Jeremy Teboul, Jessica
  Zhong, Jian Jin, Jingyi Yang, Joe Cummings, Jon Carvill, Jon Shepard,
  Jonathan McPhie, Jonathan Torres, Josh Ginsburg, Junjie Wang, Kaixing(Kai)
  Wu, U~KamHou, Karan Saxena, Karthik Prasad, Kartikay Khandelwal, Katayoun
  Zand, Kathy Matosich, Kaushik Veeraraghavan, Kelly Michelena, Keqian Li, Kun
  Huang, Kunal Chawla, Kushal Lakhotia, Kyle Huang, Lailin Chen, Lakshya Garg,
  A~Lavender, Leandro Silva, Lee Bell, Lei Zhang, Liangpeng Guo, Licheng Yu,
  Liron Moshkovich, Luca Wehrstedt, Madian Khabsa, Manav Avalani, Manish Bhatt,
  Maria Tsimpoukelli, Martynas Mankus, Matan Hasson, Matthias Lennie, Matthias
  Reso, Maxim Groshev, Maxim Naumov, Maya Lathi, Meghan Keneally, Michael~L.
  Seltzer, Michal Valko, Michelle Restrepo, Mihir Patel, Mik Vyatskov, Mikayel
  Samvelyan, Mike Clark, Mike Macey, Mike Wang, Miquel~Jubert Hermoso,
  Mo~Metanat, Mohammad Rastegari, Mun ish Bansal, Nandhini Santhanam, Natascha
  Parks, Natasha White, Navy ata Bawa, Nayan Singhal, Nick Egebo, Nicolas
  Usunier, Nikolay~Pavlovich Laptev, Ning Dong, Ning Zhang, Norman Cheng, Oleg
  Chernoguz, Olivia Hart, Omkar Salpekar, Ozlem Kalinli, Parkin Kent, Parth
  Parekh, Paul Saab, Pavan Balaji, Pe~dro Rittner, Philip Bontrager, Pierre
  Roux, Piotr Doll{\'a}r, Polina Zvyagina, Prashant Ratanchandani, Pritish
  Yuvraj, Qian Liang, Rachad Alao, Rachel Rodriguez, Rafi Ayub, Raghotham
  Murthy, Raghu Nayani, Rahul Mitra, Raymond Li, Rebekkah Hogan, Robin Battey,
  Rocky Wang, Ro~han Maheswari, Russ Howes, Ruty Rinott, Sai~Jayesh Bondu,
  Samyak Datta, Sara Chugh, Sara Hunt, Sargun Dhillon, S.~Yu. Sidorov, Satadru
  Pan, Saurabh Verma, Seiji Yamamoto, Sharadh Ramaswamy, Shaun Lindsay, Sheng
  Feng, Shenghao Lin, Shengxin Zha, Shiva Shankar, Shuqiang Zhang, Sinong Wang,
  Sneha Agarwal, Soji Sajuyigbe, Soumith Chintala, Stephanie Max, Stephen Chen,
  Steve Kehoe, Steve Satterfield, Sudarshan Govindaprasad, Sumit~Kumar Gupta,
  Sung-Bae Cho, Sunny Virk, Suraj Subramanian, Sy~Choudhury, Sydney Goldman,
  Tal Remez, Tamar Glaser, Tamara Best, Thilo Kohler, Thomas Robinson, Tianhe
  Li, Tianjun Zhang, Tim Matthews, Timothy Chou, Tzook Shaked, Varun
  Vontimitta, Victoria~O Ajayi, Victoria Montanez, Vijai Mohan, Vinay Kumar,
  Vishal Mangla, Vlad Ionescu, Vlad~Andrei Poenaru, Vlad~T. Mihailescu,
  Vladimir Ivanov, Wei Li, Wenchen Wang, Wenwen Jiang, Wes Bouaziz, Will
  Constable, Xia Tang, Xiaofang Wang, Xiaojian Wu, Xiaolan Wang, Xide Xia,
  Xilun Wu, Xinbo Gao, Yanjun Chen, Ye~Hu, Ye~Jia, Ye~Qi, Yenda Li, Yilin
  Zhang, Ying Zhang, Yossi Adi, Youngjin Nam, Yu~Wang, Yuchen Hao, Yundi Qian,
  Yuzi He, Zach Rait, Zachary DeVito, Zef Rosnbrick, Zhaoduo Wen, Zhenyu Yang,
  and Zhiwei Zhao.
\newblock The {Llama 3} herd of models.
\newblock \emph{arXiv preprint arXiv:2407.21783}, 2024.

\bibitem[Dubois et~al.(2024)Dubois, Galambosi, Liang, and
  Hashimoto]{Dubois2024LengthControlledAA}
Yann Dubois, Bal'azs Galambosi, Percy Liang, and Tatsunori Hashimoto.
\newblock Length-controlled {AlpacaEval}: A simple way to debias automatic
  evaluators.
\newblock \emph{ArXiv}, abs/2404.04475, 2024.

\bibitem[Emmenegger et~al.(2026)Emmenegger, Stahler, and
  Podimata]{emmenegger2026prediction}
Nicolas Emmenegger, Ellery Stahler, and Chara Podimata.
\newblock Prediction-powered inference across many tasks for ai evaluation \&
  social science research.
\newblock \emph{arXiv preprint arXiv:2605.29249}, 2026.

\bibitem[Ethayarajh et~al.(2024)Ethayarajh, Xu, Muennighoff, Jurafsky, and
  Kiela]{Ethayarajh2024KTOMA}
Kawin Ethayarajh, Winnie Xu, Niklas Muennighoff, Dan Jurafsky, and Douwe Kiela.
\newblock Kto: Model alignment as prospect theoretic optimization.
\newblock In \emph{International Conference on Machine Learning}, 2024.

\bibitem[Fisch et~al.(2024)Fisch, Maynez, Hofer, Dhingra, Globerson, and
  Cohen]{Fisch2024StratifiedPI}
Adam Fisch, Joshua Maynez, R.~Alex Hofer, Bhuwan Dhingra, Amir Globerson, and
  William~W. Cohen.
\newblock Stratified prediction-powered inference for hybrid language model
  evaluation.
\newblock \emph{ArXiv}, abs/2406.04291, 2024.

\bibitem[Frick et~al.(2024)Frick, Li, Chen, Chiang, Angelopoulos, Jiao, Zhu,
  Gonzalez, and Stoica]{Frick2024HowTE}
Evan Frick, Tianle Li, Connor Chen, Wei-Lin Chiang, Anastasios~Nikolas
  Angelopoulos, Jiantao Jiao, Banghua Zhu, Joseph~E. Gonzalez, and Ion Stoica.
\newblock How to evaluate reward models for {RLHF}.
\newblock In \emph{The Thirteenth International Conference on Learning
  Representations}, 2024.

\bibitem[Gao et~al.(2024)Gao, Alon, and Metzler]{Gao2024ImpactOP}
Yang Gao, Dana Alon, and Donald Metzler.
\newblock Impact of preference noise on the alignment performance of generative
  language models.
\newblock \emph{ArXiv}, abs/2404.09824, 2024.

\bibitem[Gerstgrasser et~al.(2024)Gerstgrasser, Schaeffer, Dey, Rafailov,
  Sleight, Hughes, Korbak, Agrawal, Pai, Gromov, Roberts, Yang, Donoho, and
  Koyejo]{Gerstgrasser2024IsMC}
Matthias Gerstgrasser, Rylan Schaeffer, Apratim Dey, Rafael Rafailov, Henry
  Sleight, John Hughes, Tomasz Korbak, Rajashree Agrawal, Dhruv Pai, Andrey
  Gromov, Daniel~A. Roberts, Diyi Yang, David~L. Donoho, and Sanmi Koyejo.
\newblock Is model collapse inevitable? breaking the curse of recursion by
  accumulating real and synthetic data.
\newblock \emph{ArXiv}, abs/2404.01413, 2024.

\bibitem[Ghadimi and Lan(2013)]{Ghadimi2013StochasticFA}
Saeed Ghadimi and Guanghui Lan.
\newblock Stochastic first- and zeroth-order methods for nonconvex stochastic
  programming.
\newblock \emph{SIAM J. Optim.}, 23:\penalty0 2341--2368, 2013.

\bibitem[Gligori{\'c} et~al.(2025)Gligori{\'c}, Zrnic, Lee, Cand{\`e}s, and
  Jurafsky]{Gligoric2024CanUL}
Kristina Gligori{\'c}, Tijana Zrnic, Cinoo Lee, Emmanuel~J. Cand{\`e}s, and Dan
  Jurafsky.
\newblock Can unconfident {LLM} annotations be used for confident conclusions?
\newblock In \emph{Proceedings of the 2025 Conference of the North American
  Chapter of the Association for Computational Linguistics (NAACL)}, 2025.
\newblock arXiv:2408.15204.

\bibitem[Han et~al.(2024)Han, Li, and Zhang]{HanLiZhang2024FinitetimeDC}
Yuze Han, Xiang Li, and Zhihua Zhang.
\newblock Finite-time decoupled convergence in nonlinear two-time-scale
  stochastic approximation.
\newblock \emph{arXiv preprint arXiv:2401.03893}, 2024.

\bibitem[Hazan(2016)]{Hazan2016OCO}
Elad Hazan.
\newblock Introduction to online convex optimization.
\newblock \emph{Foundations and Trends in Optimization}, 2\penalty0
  (3-4):\penalty0 157--325, 2016.

\bibitem[He et~al.(2024)He, Wang, Jiang, Papangelis, and
  Zhao]{He2024SemiSupervisedRM}
Yifei He, Haoxiang Wang, Ziyan Jiang, Alexandros Papangelis, and Han Zhao.
\newblock Semi-supervised reward modeling via iterative self-training.
\newblock In \emph{Findings of the Association for Computational Linguistics:
  EMNLP 2024}, 2024.
\newblock arXiv:2409.06903.

\bibitem[Hong et~al.(2023)Hong, Wai, Wang, and Yang]{HongWaiWangYang2023ATS}
Mingyi Hong, Hoi-To Wai, Zhaoran Wang, and Zhuoran Yang.
\newblock A two-timescale stochastic algorithm framework for bilevel
  optimization: Complexity analysis and application to actor-critic.
\newblock \emph{SIAM Journal on Optimization}, 33\penalty0 (1):\penalty0
  147--180, 2023.

\bibitem[Huang et~al.(2024)Huang, Block, Foster, Rohatgi, Zhang, Simchowitz,
  Ash, and Krishnamurthy]{Huang2024SelfImprovementIL}
Audrey Huang, Adam Block, Dylan~J. Foster, Dhruv Rohatgi, Cyril Zhang, Max
  Simchowitz, Jordan~T. Ash, and Akshay Krishnamurthy.
\newblock Self-improvement in language models: The sharpening mechanism.
\newblock \emph{ArXiv}, abs/2412.01951, 2024.

\bibitem[Ji et~al.(2025)Ji, Lei, and Zrnic]{ji2025predictions}
Wenlong Ji, Lihua Lei, and Tijana Zrnic.
\newblock Predictions as surrogates: Revisiting surrogate outcomes in the age
  of ai.
\newblock \emph{arXiv preprint arXiv:2501.09731}, 2025.

\bibitem[Kingma and Ba(2014)]{Kingma2014AdamAM}
Diederik~P. Kingma and Jimmy Ba.
\newblock Adam: A method for stochastic optimization.
\newblock \emph{CoRR}, abs/1412.6980, 2014.

\bibitem[Kluger and Bates(2026)]{kluger2026m}
Dan~M Kluger and Stephen Bates.
\newblock M-estimation under two-phase multiwave sampling with applications to
  prediction-powered inference.
\newblock \emph{arXiv preprint arXiv:2602.16933}, 2026.

\bibitem[Kluger et~al.(2025)Kluger, Lu, Zrnic, Wang, and
  Bates]{kluger2025prediction}
Dan~M Kluger, Kerri Lu, Tijana Zrnic, Sherrie Wang, and Stephen Bates.
\newblock Prediction-powered inference with imputed covariates and nonuniform
  sampling.
\newblock \emph{arXiv preprint arXiv:2501.18577}, 2025.

\bibitem[Konda and Tsitsiklis(2003)]{KondaTsitsiklis2003OnActorCritic}
Vijay~R. Konda and John~N. Tsitsiklis.
\newblock On actor-critic algorithms.
\newblock \emph{SIAM Journal on Control and Optimization}, 42\penalty0
  (4):\penalty0 1143--1166, 2003.

\bibitem[Kwon et~al.(2025)Kwon, Dotson, Chen, and Xie]{Kwon2025TwoTimescaleLSA}
Jeongyeol Kwon, Luke Dotson, Yudong Chen, and Qiaomin Xie.
\newblock Two-timescale linear stochastic approximation: Constant stepsizes go
  a long way.
\newblock In \emph{Proceedings of the 28th International Conference on
  Artificial Intelligence and Statistics (AISTATS)}, Proceedings of Machine
  Learning Research. PMLR, 2025.

\bibitem[Lambert et~al.(2025)Lambert, Pyatkin, Morrison, Miranda, Lin, Chandu,
  Dziri, Kumar, Zick, Choi, Smith, and Hajishirzi]{Lambert2024RewardBenchER}
Nathan Lambert, Valentina Pyatkin, Jacob~Daniel Morrison, Lester James~Validad
  Miranda, Bill~Yuchen Lin, Khyathi~Raghavi Chandu, Nouha Dziri, Sachin Kumar,
  Tom Zick, Yejin Choi, Noah~A. Smith, and Hanna Hajishirzi.
\newblock Rewardbench: Evaluating reward models for language modeling.
\newblock In \emph{Findings of the Association for Computational Linguistics:
  NAACL 2025}, pages 1755--1797, 2025.

\bibitem[Lee(2013)]{Lee2013PseudoLabelT}
Dong-Hyun Lee.
\newblock Pseudo-label : The simple and efficient semi-supervised learning
  method for deep neural networks.
\newblock 2013.

\bibitem[Lee et~al.(2023)Lee, Phatale, Mansoor, Lu, Mesnard, Bishop, Carbune,
  and Rastogi]{Lee2023RLAIFVR}
Harrison Lee, Samrat Phatale, Hassan Mansoor, Kellie Lu, Thomas Mesnard, Colton
  Bishop, Victor Carbune, and Abhinav Rastogi.
\newblock {RLAIF} vs. {RLHF}: Scaling reinforcement learning from human
  feedback with {AI} feedback.
\newblock In \emph{International Conference on Machine Learning}, 2023.

\bibitem[Lee et~al.(2025)Lee, Lim, Park, Cheon, and
  Song]{Lee2025SemiSupervisedPO}
Seonggyun Lee, Sungjun Lim, Seojin Park, Soeun Cheon, and Kyungwoo Song.
\newblock Semi-supervised preference optimization with limited feedback.
\newblock \emph{ArXiv}, abs/2511.00040, 2025.

\bibitem[Li et~al.(2025)Li, Sun, Huang, Zhong, Jiang, Han, Zhang, Wang, and
  Liu]{Li2025PreferenceLA}
Dawei Li, Renliang Sun, Yue Huang, Ming Zhong, Bohan Jiang, Jiawei Han,
  Xiangliang Zhang, Wei Wang, and Huan Liu.
\newblock Preference leakage: A contamination problem in {LLM}-as-a-judge.
\newblock \emph{ArXiv}, abs/2502.01534, 2025.

\bibitem[Liu et~al.(2024)Liu, Zeng, Liu, Yan, He, Wang, Yan, Liu, and
  Zhou]{Liu2024SkyworkRewardBO}
Chris Liu, Liang Zeng, Jiacai Liu, Rui Yan, Jujie He, Chaojie Wang, Shuicheng
  Yan, Yang Liu, and Yahui Zhou.
\newblock Skywork-reward: Bag of tricks for reward modeling in llms.
\newblock \emph{ArXiv}, abs/2410.18451, 2024.

\bibitem[Liu et~al.(2025{\natexlab{a}})Liu, Zeng, Xiao, He, Liu, Wang, Yan,
  Shen, Zhang, Xu, Liu, and Zhou]{Liu2025SkyworkRewardV2SP}
Chris Liu, Liang Zeng, Yuzhen Xiao, Jujie He, Jiacai Liu, Chaojie Wang, Rui
  Yan, Wei Shen, Fuxiang Zhang, Jiacheng Xu, Yang Liu, and Yahui Zhou.
\newblock Skywork-reward-v2: Scaling preference data curation via human-ai
  synergy.
\newblock \emph{ArXiv}, abs/2507.01352, 2025{\natexlab{a}}.

\bibitem[Liu et~al.(2025{\natexlab{b}})Liu, Chen, Li, Qi, Pang, Du, Lee, and
  Lin]{Liu2025UnderstandingR1ZeroLT}
Zichen Liu, Changyu Chen, Wenjun Li, Penghui Qi, Tianyu Pang, Chao Du, Wee~Sun
  Lee, and Min Lin.
\newblock Understanding r1-zero-like training: A critical perspective.
\newblock \emph{ArXiv}, abs/2503.20783, 2025{\natexlab{b}}.

\bibitem[Malik et~al.(2025)Malik, Pyatkin, Land, Morrison, Smith, Hajishirzi,
  and Lambert]{Malik2025RewardBench2A}
Saumya Malik, Valentina Pyatkin, Sander Land, Jacob~Daniel Morrison, Noah~A.
  Smith, Hanna Hajishirzi, and Nathan Lambert.
\newblock Rewardbench 2: Advancing reward model evaluation.
\newblock \emph{arXiv preprint arXiv:2506.01937}, 2025.

\bibitem[Mani et~al.(2025)Mani, Xu, Lipton, and Oberst]{Mani2025NoFL}
Pranav Mani, Peng Xu, Zachary~C. Lipton, and Michael Oberst.
\newblock No free lunch: Non-asymptotic analysis of prediction-powered
  inference.
\newblock \emph{ArXiv}, abs/2505.20178, 2025.

\bibitem[Miranda et~al.(2024)Miranda, Wang, Elazar, Kumar, Pyatkin, Brahman,
  Smith, Hajishirzi, and Dasigi]{Miranda2024HybridPL}
Lester James~V. Miranda, Yizhong Wang, Yanai Elazar, Sachin Kumar, Valentina
  Pyatkin, Faeze Brahman, Noah~A. Smith, Hannaneh Hajishirzi, and Pradeep
  Dasigi.
\newblock Hybrid preferences: Learning to route instances for human vs. {AI}
  feedback.
\newblock \emph{ArXiv}, abs/2410.19133, 2024.

\bibitem[Mokkadem and Pelletier(2006)]{MokkademPelletier2006ConvergenceRO}
Abdelkader Mokkadem and Mariane Pelletier.
\newblock Convergence rate and averaging of nonlinear two-time-scale stochastic
  approximation algorithms.
\newblock \emph{The Annals of Applied Probability}, 16\penalty0 (3):\penalty0
  1671--1702, 2006.

\bibitem[Ouyang et~al.(2022)Ouyang, Wu, Jiang, Almeida, Wainwright, Mishkin,
  Zhang, Agarwal, Slama, Ray, Schulman, Hilton, Kelton, Miller, Simens, Askell,
  Welinder, Christiano, Leike, and Lowe]{Ouyang2022TrainingLM}
Long Ouyang, Jeff Wu, Xu~Jiang, Diogo Almeida, Carroll~L. Wainwright, Pamela
  Mishkin, Chong Zhang, Sandhini Agarwal, Katarina Slama, Alex Ray, John
  Schulman, Jacob Hilton, Fraser Kelton, Luke~E. Miller, Maddie Simens, Amanda
  Askell, Peter Welinder, Paul~Francis Christiano, Jan Leike, and Ryan~J. Lowe.
\newblock Training language models to follow instructions with human feedback.
\newblock \emph{Advances in neural information processing systems},
  35:\penalty0 27730--27744, 2022.

\bibitem[Pace et~al.(2024)Pace, Mallinson, Malmi, Krause, and
  Severyn]{Pace2024WestofNSP}
Aliz{\'e}e Pace, Jonathan Mallinson, Eric Malmi, Sebastian Krause, and Aliaksei
  Severyn.
\newblock West-of-n: Synthetic preferences for self-improving reward models.
\newblock \emph{ArXiv}, abs/2401.12086, 2024.

\bibitem[Polyak and Juditsky(1992)]{PolyakJuditsky1992AccelerationOS}
Boris~T. Polyak and Anatoli~B. Juditsky.
\newblock Acceleration of stochastic approximation by averaging.
\newblock \emph{SIAM Journal on Control and Optimization}, 30\penalty0
  (4):\penalty0 838--855, 1992.

\bibitem[Rafailov et~al.(2023)Rafailov, Sharma, Mitchell, Ermon, Manning, and
  Finn]{Rafailov2023DirectPO}
Rafael Rafailov, Archit Sharma, Eric Mitchell, Stefano Ermon, Christopher~D.
  Manning, and Chelsea Finn.
\newblock Direct preference optimization: Your language model is secretly a
  reward model.
\newblock \emph{Advances in neural information processing systems},
  36:\penalty0 53728--53741, 2023.

\bibitem[Rakhlin and Sridharan(2013)]{Rakhlin2012OnlineLW}
Alexander Rakhlin and Karthik Sridharan.
\newblock Online learning with predictable sequences.
\newblock In \emph{Conference on Learning Theory}, pages 993--1019. PMLR, 2013.

\bibitem[Sfyraki and Wang(2026)]{sfyrakirevisiting}
Maria-Eleni Sfyraki and Jun-Kun Wang.
\newblock Revisiting active sequential prediction-powered mean estimation.
\newblock In \emph{The Fourteenth International Conference on Learning
  Representations}, 2026.

\bibitem[Shao et~al.(2024)Shao, Wang, Zhu, Xu, Song, Zhang, Li, Wu, and
  Guo]{Shao2024DeepSeekMathPT}
Zhihong Shao, Peiyi Wang, Qihao Zhu, Runxin Xu, Junxiao Song, Mingchuan Zhang,
  Y.~K. Li, Yu~Wu, and Daya Guo.
\newblock Deepseekmath: Pushing the limits of mathematical reasoning in open
  language models.
\newblock \emph{arXiv preprint arXiv:2402.03300}, 2024.

\bibitem[Shi et~al.(2024)Shi, Ma, Liang, Diao, Ma, and
  Vosoughi]{Shi2024JudgingTJ}
Lin Shi, Chiyu Ma, Wenhua Liang, Xingjian Diao, Weicheng Ma, and Soroush
  Vosoughi.
\newblock Judging the judges: A systematic study of position bias in
  {LLM}-as-a-judge.
\newblock \emph{ArXiv}, abs/2406.07791, 2024.

\bibitem[Shoham et~al.(2025)Shoham, Dorfman, Shaer, Levy, and
  Romano]{Shoham2025PredictionPoweredSL}
Noa Shoham, Ron Dorfman, Shalev Shaer, Kfir~Yehuda Levy, and Yaniv Romano.
\newblock Prediction-powered semi-supervised learning with online power tuning.
\newblock In \emph{The Thirty-ninth Annual Conference on Neural Information
  Processing Systems}, 2025.

\bibitem[Shumailov et~al.(2023)Shumailov, Shumaylov, Zhao, Gal, Papernot, and
  Anderson]{Shumailov2023TheCO}
Ilia Shumailov, Zakhar Shumaylov, Yiren Zhao, Yarin Gal, Nicolas Papernot, and
  Ross Anderson.
\newblock The curse of recursion: Training on generated data makes models
  forget.
\newblock \emph{ArXiv}, abs/2305.17493, 2023.

\bibitem[Sohn et~al.(2020)Sohn, Berthelot, Li, Zhang, Carlini, Cubuk, Kurakin,
  Zhang, and Raffel]{Sohn2020FixMatchSS}
Kihyuk Sohn, David Berthelot, Chun-Liang Li, Zizhao Zhang, Nicholas Carlini,
  Ekin~Dogus Cubuk, Alexey Kurakin, Han Zhang, and Colin Raffel.
\newblock Fixmatch: Simplifying semi-supervised learning with consistency and
  confidence.
\newblock \emph{Advances in neural information processing systems},
  33:\penalty0 596--608, 2020.

\bibitem[Somerstep et~al.(2024)Somerstep, Polo, Banerjee, Ritov, Yurochkin, and
  Sun]{Somerstep2024ATL}
Seamus Somerstep, Felipe~Maia Polo, Moulinath Banerjee, Ya'acov Ritov, Mikhail
  Yurochkin, and Yuekai Sun.
\newblock A transfer learning framework for weak-to-strong generalization.
\newblock \emph{ArXiv}, abs/2405.16236, 2024.

\bibitem[Song et~al.(2026)Song, Kluger, Parikh, and Gu]{song2026demystifying}
Yilin Song, Dan~M Kluger, Harsh Parikh, and Tian Gu.
\newblock Demystifying prediction powered inference.
\newblock \emph{arXiv preprint arXiv:2601.20819}, 2026.

\bibitem[Song et~al.(2025)Song, Zhang, Eisenach, Kakade, Foster, and
  Ghai]{Song2024MindTG}
Yuda Song, Hanlin Zhang, Carson Eisenach, Sham Kakade, Dean Foster, and Udaya
  Ghai.
\newblock Mind the gap: Examining the self-improvement capabilities of large
  language models.
\newblock In \emph{International Conference on Learning Representations}, 2025.
\newblock arXiv:2412.02674.

\bibitem[Stiennon et~al.(2020)Stiennon, Ouyang, Wu, Ziegler, Lowe, Voss,
  Radford, Amodei, and Christiano]{Stiennon2020LearningTS}
Nisan Stiennon, Long Ouyang, Jeff Wu, Daniel~M. Ziegler, Ryan~J. Lowe, Chelsea
  Voss, Alec Radford, Dario Amodei, and Paul Christiano.
\newblock Learning to summarize from human feedback.
\newblock \emph{ArXiv}, abs/2009.01325, 2020.

\bibitem[Surendran et~al.(2024)Surendran, Godichon-Baggioni, Fermanian, and
  Corff]{Surendran2024NonasymptoticAO}
Sobihan Surendran, Antoine Godichon-Baggioni, Adeline Fermanian, and Sylvain~Le
  Corff.
\newblock Non-asymptotic analysis of biased adaptive stochastic approximation.
\newblock \emph{ArXiv}, abs/2402.02857, 2024.

\bibitem[Tao and Li(2025)]{Tao2024YourWL}
Leitian Tao and Yixuan Li.
\newblock Your weak {LLM} is secretly a strong teacher for alignment.
\newblock In \emph{International Conference on Learning Representations}, 2025.
\newblock arXiv:2409.08813.

\bibitem[Testa et~al.(2025)Testa, Xu, Lei, and Roeder]{testa2025semiparametric}
Lorenzo Testa, Qi~Xu, Jing Lei, and Kathryn Roeder.
\newblock Semiparametric semi-supervised learning for general targets under
  distribution shift and decaying overlap.
\newblock \emph{arXiv preprint arXiv:2505.06452}, 2025.

\bibitem[Tunstall et~al.(2024)Tunstall, Beeching, Lambert, Rajani, Rasul,
  Belkada, Huang, Werra, Fourrier, Habib, Sarrazin, Sanseviero, Rush, and
  Wolf]{Tunstall2023ZephyrDD}
Lewis Tunstall, Edward~Emanuel Beeching, Nathan Lambert, Nazneen Rajani, Kashif
  Rasul, Younes Belkada, Shengyi Huang, Leandro~Von Werra, Cl{\'e}mentine
  Fourrier, Nathan Habib, Nathan Sarrazin, Omar Sanseviero, Alexander~M Rush,
  and Thomas Wolf.
\newblock Zephyr: Direct distillation of {LM} alignment.
\newblock In \emph{First Conference on Language Modeling}, 2024.

\bibitem[van~der Laan and Van Der~Laan(2026)]{van2026calibeating}
Lars van~der Laan and Mark Van Der~Laan.
\newblock Calibeating prediction-powered inference.
\newblock \emph{arXiv preprint arXiv:2604.21260}, 2026.

\bibitem[Wang et~al.(2024)Wang, Li, Chen, Zhu, Lin, Cao, Liu, Liu, and
  Sui]{Wang2023LargeLM}
Peiyi Wang, Lei Li, Liang Chen, Dawei Zhu, Binghuai Lin, Yunbo Cao, Qi~Liu,
  Tianyu Liu, and Zhifang Sui.
\newblock Large language models are not fair evaluators.
\newblock In \emph{Proceedings of the 62nd Annual Meeting of the Association
  for Computational Linguistics (Volume 1: Long Papers)}, pages 9440--9450,
  2024.

\bibitem[Ward et~al.(2018)Ward, Wu, and Bottou]{Ward2018AdaGradSS}
Rachel~A. Ward, Xiaoxia Wu, and L{\'e}on Bottou.
\newblock Adagrad stepsizes: Sharp convergence over nonconvex landscapes, from
  any initialization.
\newblock In \emph{International Conference on Machine Learning}, 2018.

\bibitem[Wataoka et~al.(2024)Wataoka, Takahashi, and
  Ri]{Wataoka2024SelfPreferenceBI}
Koki Wataoka, Tsubasa Takahashi, and Ryokan Ri.
\newblock Self-preference bias in {LLM}-as-a-judge.
\newblock \emph{ArXiv}, abs/2410.21819, 2024.

\bibitem[Wei et~al.(2023)Wei, Huang, Lu, Zhou, and Le]{Wei2023SimpleSD}
Jerry~W. Wei, Da~Huang, Yifeng Lu, Denny Zhou, and Quoc~V. Le.
\newblock Simple synthetic data reduces sycophancy in large language models.
\newblock \emph{ArXiv}, abs/2308.03958, 2023.

\bibitem[Wu et~al.(2024)Wu, Yuan, Golovneva, Xu, Tian, Jiao, Weston, and
  Sukhbaatar]{Wu2024MetaRewardingLM}
Tianhao Wu, Weizhe Yuan, Olga Golovneva, Jing Xu, Yuandong Tian, Jiantao Jiao,
  Jason Weston, and Sainbayar Sukhbaatar.
\newblock Meta-rewarding language models: Self-improving alignment with
  llm-as-a-meta-judge.
\newblock \emph{ArXiv}, abs/2407.19594, 2024.

\bibitem[Xu et~al.(2025{\natexlab{a}})Xu, Ye, Zhou, Zhu, Quinzan, and
  Shi]{Xu2025DoublyRA}
Erhan Xu, Kai Ye, Hongyi Zhou, Luhan Zhu, Francesco Quinzan, and Chengchun Shi.
\newblock Doubly robust alignment for large language models.
\newblock In \emph{Advances in Neural Information Processing Systems},
  2025{\natexlab{a}}.
\newblock arXiv:2506.01183.

\bibitem[Xu et~al.(2025{\natexlab{b}})Xu, Chakraborty, K{\i}c{\i}man, Aryal,
  Rodrigues, Sharma, Estevao, de~Luis~Balaguer, Wolk, Padilha, Nunes,
  Balakrishnan, Lu, and Chandra]{Xu2025RLTHFTH}
Yifei Xu, Tusher Chakraborty, Emre K{\i}c{\i}man, Bibek Aryal, Eduardo
  Rodrigues, Srinagesh Sharma, Roberto Estevao, Maria~Angels de~Luis~Balaguer,
  Jessica Wolk, Rafael Padilha, Leonardo Nunes, Shobana Balakrishnan, Songwu
  Lu, and Ranveer Chandra.
\newblock {RLTHF}: Targeted human feedback for {LLM} alignment.
\newblock In \emph{International Conference on Machine Learning},
  2025{\natexlab{b}}.
\newblock arXiv:2502.13417.

\bibitem[Ye et~al.(2024)Ye, Wang, Huang, Chen, Zhang, Moniz, Gao, Geyer, Huang,
  Chen, Chawla, and Zhang]{Ye2024JusticeOP}
Jiayi Ye, Yanbo Wang, Yue Huang, Dongping Chen, Qihui Zhang, Nuno Moniz, Tian
  Gao, Werner Geyer, Chao Huang, Pin-Yu Chen, Nitesh~V. Chawla, and Xiangliang
  Zhang.
\newblock Justice or prejudice? quantifying biases in {LLM}-as-a-judge.
\newblock \emph{ArXiv}, abs/2410.02736, 2024.

\bibitem[Ye et~al.(2025)Ye, Zhou, Zhu, Quinzan, and Shi]{Ye2025RobustRF}
Kai Ye, Hongyi Zhou, Jin Zhu, Francesco Quinzan, and Chengchun Shi.
\newblock Robust reinforcement learning from human feedback for large language
  models fine-tuning.
\newblock \emph{ArXiv}, abs/2504.03784, 2025.

\bibitem[Yu et~al.(2025)Yu, Zhang, Zhu, Yuan, Zuo, Yue, Dai, Fan, Liu, Liu,
  Liu, Lin, Lin, Ma, Sheng, Tong, Zhang, Zhang, Zhang, Zhu, Zhu, Chen, Chen,
  Wang, Yu, Dai, Song, Wei, Zhou, Liu, Ma, Wang, Yan, Qiao, Wu, and
  Wang]{Yu2025DAPOAO}
Qiying Yu, Zheng Zhang, Ruofei Zhu, Yufeng Yuan, Xiaochen Zuo, Yu~Yue, Weinan
  Dai, Tiantian Fan, Gaohong Liu, Lingjun Liu, Xin Liu, Haibin Lin, Zhiqi Lin,
  Bole Ma, Guangming Sheng, Yuxuan Tong, Chi Zhang, Mofan Zhang, Wang Zhang,
  Hang Zhu, Jinhua Zhu, Jiaze Chen, Jiangjie Chen, Chengyi Wang, Hongli Yu,
  Weinan Dai, Yuxuan Song, Xiangpeng Wei, Hao Zhou, Jingjing Liu, Wei-Ying Ma,
  Yaqing Wang, Lin Yan, Mu~Qiao, Yonghui Wu, and Mingxuan Wang.
\newblock {DAPO}: An open-source {LLM} reinforcement learning system at scale.
\newblock \emph{ArXiv}, abs/2503.14476, 2025.

\bibitem[Yuan et~al.(2024)Yuan, Pang, Cho, Li, Sukhbaatar, Xu, and
  Weston]{Yuan2024SelfRewardingLM}
Weizhe Yuan, Richard~Yuanzhe Pang, Kyunghyun Cho, Xian Li, Sainbayar
  Sukhbaatar, Jing Xu, and Jason~E Weston.
\newblock Self-rewarding language models.
\newblock In \emph{Forty-first International Conference on Machine Learning},
  2024.

\bibitem[Zeng et~al.(2025)Zeng, Zhou, Arora, and Zanette]{Zeng2025ShrinkingTV}
Guanning Zeng, Zhaoyi Zhou, Daman Arora, and Andrea Zanette.
\newblock Shrinking the variance: Shrinkage baselines for reinforcement
  learning with verifiable rewards.
\newblock \emph{ArXiv}, abs/2511.03710, 2025.

\bibitem[Zhang et~al.(2021)Zhang, Wang, Hou, Wu, Wang, Okumura, and
  Shinozaki]{Zhang2021FlexMatchBS}
Bowen Zhang, Yidong Wang, Wenxin Hou, Hao Wu, Jindong Wang, Manabu Okumura, and
  Takahiro Shinozaki.
\newblock Flexmatch: Boosting semi-supervised learning with curriculum pseudo
  labeling.
\newblock \emph{Advances in neural information processing systems},
  34:\penalty0 18408--18419, 2021.

\bibitem[Zhang et~al.(2025)Zhang, Wang, Hwang, Dong, Delalleau, Choi, Choi,
  Ren, and Pyatkin]{Zhang2024DivergingPW}
Michael~Jq Zhang, Zhilin Wang, Jena~D Hwang, Yi~Dong, Olivier Delalleau, Yejin
  Choi, Eunsol Choi, Xiang Ren, and Valentina Pyatkin.
\newblock Diverging preferences: When do annotators disagree and do models
  know?
\newblock In \emph{International Conference on Machine Learning}, pages
  76193--76212. PMLR, 2025.

\bibitem[Zheng et~al.(2025)Zheng, Liu, Li, Chen, Yu, Gao, Dang, Liu, Men, Yang,
  Zhou, and Lin]{Zheng2025GroupSP}
Chujie Zheng, Shixuan Liu, Mingze Li, Xiong-Hui Chen, Bowen Yu, Chang Gao, Kai
  Dang, Yuqiong Liu, Rui Men, An~Yang, Jingren Zhou, and Junyang Lin.
\newblock Group sequence policy optimization.
\newblock \emph{ArXiv}, abs/2507.18071, 2025.

\bibitem[Zheng et~al.(2023)Zheng, Chiang, Sheng, Zhuang, Wu, Zhuang, Lin, Li,
  Li, Xing, Zhang, Gonzalez, and Stoica]{Zheng2023JudgingLW}
Lianmin Zheng, Wei-Lin Chiang, Ying Sheng, Siyuan Zhuang, Zhanghao Wu, Yonghao
  Zhuang, Zi~Lin, Zhuohan Li, Dacheng Li, Eric~P. Xing, Haotong Zhang,
  Joseph~E. Gonzalez, and Ion Stoica.
\newblock Judging llm-as-a-judge with mt-bench and chatbot arena.
\newblock \emph{Advances in neural information processing systems},
  36:\penalty0 46595--46623, 2023.

\bibitem[Zhou et~al.(2026)Zhou, Ye, Xu, Zhu, Yang, Gong, and
  Shi]{Zhou2026DemystifyingGR}
Hongyi Zhou, Kai Ye, Erhan Xu, Jin Zhu, Ying Yang, Shijin Gong, and Chengchun
  Shi.
\newblock Demystifying group relative policy optimization: Its policy gradient
  is a u-statistic.
\newblock \emph{ArXiv}, abs/2603.01162, 2026.

\bibitem[Zhou et~al.(2025)Zhou, Song, and Zanette]{Zhou2025AcceleratingUL}
Zhaoyi Zhou, Yuda Song, and Andrea Zanette.
\newblock Accelerating unbiased {LLM} evaluation via synthetic feedback.
\newblock \emph{ArXiv}, abs/2502.10563, 2025.

\bibitem[Zhu et~al.(2023)Zhu, Ding, Jacobson, Wu, Zhan, Jordan, and
  Jiao]{Zhu2023DoublyRS}
Banghua Zhu, Mingyu Ding, Philip Jacobson, Ming Wu, Wei Zhan, Michael Jordan,
  and Jiantao Jiao.
\newblock Doubly-robust self-training.
\newblock \emph{Advances in Neural Information Processing Systems},
  36:\penalty0 41413--41431, 2023.

\bibitem[Ziegler et~al.(2019)Ziegler, Stiennon, Wu, Brown, Radford, Amodei,
  Christiano, and Irving]{Ziegler2019FineTuningLM}
Daniel~M. Ziegler, Nisan Stiennon, Jeff Wu, Tom~B. Brown, Alec Radford, Dario
  Amodei, Paul Christiano, and Geoffrey Irving.
\newblock Fine-tuning language models from human preferences.
\newblock \emph{ArXiv}, abs/1909.08593, 2019.

\bibitem[Zrnic(2024)]{zrnic2024note}
Tijana Zrnic.
\newblock A note on the prediction-powered bootstrap.
\newblock \emph{arXiv preprint arXiv:2405.18379}, 2024.

\bibitem[Zrnic and
  Cand{\`e}s(2024{\natexlab{a}})]{Zrnic2023CrossPredictionpoweredI}
Tijana Zrnic and Emmanuel~J. Cand{\`e}s.
\newblock Cross-prediction-powered inference.
\newblock \emph{Proceedings of the National Academy of Sciences}, 121\penalty0
  (15), 2024{\natexlab{a}}.
\newblock arXiv:2309.16598.

\bibitem[Zrnic and Cand{\`e}s(2024{\natexlab{b}})]{Zrnic2024ActiveSI}
Tijana Zrnic and Emmanuel~J. Cand{\`e}s.
\newblock Active statistical inference.
\newblock In \emph{International Conference on Machine Learning},
  2024{\natexlab{b}}.
\newblock arXiv:2403.03208.

\end{thebibliography}

\newpage
\appendix
\section*{Appendix}

\section{Proofs for \Secref{sec:abc_ssl}}
\label{app:abc_ssl_proofs}

This appendix collects the proofs of all results stated in \Secref{sec:abc_ssl}, organized to mirror the main-text subsections.

\subsection{Estimator and bias--variance decomposition (\Secref{sec:estimator})}
\label{app:estimator_proofs}

\subsubsection{Proof of Lemma~\ref{lem:bias}}
\label{app:lem_bias}

\lembias*

\begin{proof}
By linearity of conditional expectation and predictability of $(a_t, b_t)$,
\[
    \E_t[\glam] = \E_t[\gnf_t] + a_t\,\E_t[\dt] + b_t\,\E_t[c_t].
\]
The inputs in each of $\glab_t, \gnf_t, \gtNf_t$ are i.i.d.\ with marginal $P_X$, and $f(x)$ is deterministic given $x$, so
\[
    \E_t[\glab_t] = \nabla \Ls(\theta_t), \qquad \E_t[\gnf_t] = \E_t[\gtNf_t] = \nabla \Lf(\theta_t).
\]
Hence $\E_t[\dt] = -B(\theta_t)$ and $\E_t[c_t] = 0$, giving $\E_t[\glam] = \nabla \Lf(\theta_t) - a_t B(\theta_t) = \nabla \Ls(\theta_t) + (1-a_t) B(\theta_t)$, i.e., $\Bias_t(\glam) = (1-a_t) B(\theta_t)$.
\end{proof}

\subsubsection{Proof of Proposition~\ref{prop:mse}}
\label{app:prop_mse}

\propmse*

\begin{proof}
The bias--variance identity $\E_t\norm{\glam - \nabla\Ls(\theta_t)}^2 = \norm{\Bias_t(\glam)}^2 + \Var_t(\glam)$ holds by definition of conditional bias and variance applied to $\glam$ relative to the deterministic target $\nabla\Ls(\theta_t)$. The squared-bias formula $\norm{\Bias_t(\glam)}^2 = (1-a_t)^2\Bt^2$ is immediate from Lemma~\ref{lem:bias}. It remains to compute $\Var_t(\glam)$.

Using the second form of~\eqref{eq:abc_grad}, write
\[
    \glam = (1-a_t-b_t)\,\gnf_t + a_t\,\glab_t + b_t\,\gtNf_t.
\]
The unlabeled batch $\{x^i_t\}_{i=n+1}^{n+N}$ is drawn independently of the labeled batch $\{(x^i_t, y^i_t)\}_{i=1}^{n}$, and only $\gtNf_t$ depends on the unlabeled batch, so
\begin{equation}
    \Var_t(\glam) = \Var_t\!\bigl((1-a_t-b_t)\,\gnf_t + a_t\,\glab_t\bigr) + b_t^2\,\Var_t(\gtNf_t)
    = \Var_t\!\bigl((1-a_t-b_t)\,\gnf_t + a_t\,\glab_t\bigr) + \frac{b_t^2\,\sigma_f^2}{N},
    \label{eq:var_split}
\end{equation}
using $\Var_t(\gtNf_t) = \sigma_f^2/N$ since $\gtNf_t$ is an average of $N$ i.i.d.\ pseudolabeled gradients.

The remaining term is the variance of an average over the $n$ labeled samples. The $i$th sample contributes
\[
    u_i := (1-a_t-b_t)\,\nabla\loss(\theta_t; x_i, f(x_i)) + a_t\,\nabla\loss(\theta_t; x_i, y_i),
\]
and $(1-a_t-b_t)\,\gnf_t + a_t\,\glab_t = (1/n)\sum_{i=1}^n u_i$. By the i.i.d.\ assumption, $\Var_t\!\bigl((1-a_t-b_t)\gnf_t + a_t\glab_t\bigr) = \mathrm{tr}\,\mathrm{Cov}(u_i)/n$. Expanding the covariance via $\mathrm{tr}\,\mathrm{Cov}(\alpha X + \beta Y) = \alpha^2\,\mathrm{tr}\,\mathrm{Cov}(X) + \beta^2\,\mathrm{tr}\,\mathrm{Cov}(Y) + 2\alpha\beta\,\mathrm{tr}\,\mathrm{Cov}(X,Y)$ and substituting the definitions of $\sigma^2$, $\sigma_f^2$, and $C$ from~\eqref{eq:sigma_sq}--\eqref{eq:C_def} gives
\[
    \mathrm{tr}\,\mathrm{Cov}(u_i) = (1-a_t-b_t)^2\,\sigma_f^2 + a_t^2\,\sigma^2 + 2\,a_t(1-a_t-b_t)\,C.
\]
Combining with~\eqref{eq:var_split} yields the variance formula in~\eqref{eq:mse}.
\end{proof}

\subsection{Oracle-optimal parameters (\Secref{sec:oracle})}
\label{app:oracle_proofs}

\subsubsection{Convexity of the per-step MSE}
\label{app:lem_convex}

\begin{lemma}[Convexity of the per-step MSE]
\label{lem:convex}
For any fixed $\theta_t$, the per-step MSE $V_t(a, b)$ from~\eqref{eq:mse_def}--\eqref{eq:mse} is a convex quadratic function of $(a, b) \in \R^2$.
\end{lemma}

\begin{proof}
Let $U := \nabla\loss(\theta_t; x, y) - \nabla\Ls(\theta_t)$ and $\tilde V := \nabla\loss(\theta_t; x, f(x)) - \nabla\Lf(\theta_t)$ be the centered labeled and pseudolabeled gradients on a shared input $(x,y) \sim P$; by~\eqref{eq:sigma_sq}--\eqref{eq:C_def}, $\E\norm{U}^2 = \sigma^2$, $\E\norm{\tilde V}^2 = \sigma_f^2$, and $\E\inner{U}{\tilde V} = C$. The $\mathrm{tr}\,\mathrm{Cov}(u_i)$ identity from the proof of Proposition~\ref{prop:mse} reads
\[
    \E\norm{(1{-}a{-}b)\,\tilde V + a\,U}^2 = (1{-}a{-}b)^2\,\sigma_f^2 + a^2\,\sigma^2 + 2a(1{-}a{-}b)\,C,
\]
so~\eqref{eq:mse} rewrites as
\begin{equation}
V_t(a,b) = (1-a)^2\,\Bt^2 \;+\; \frac{\E\norm{(1{-}a{-}b)\,\tilde V + a\,U}^2}{n} \;+\; \frac{b^2\,\sigma_f^2}{N}.
\label{eq:mse_sos}
\end{equation}
Each summand is a convex quadratic in $(a,b)$: the first and third are squared affine functions, and the second is $n^{-1}\norm{\phi(a,b)}_{L^2(P)}^2$ where $\phi(a,b) := (1{-}a{-}b)\,\tilde V + a\,U$ is an affine map from $\R^2$ into the Hilbert space $L^2(P;\R^d)$, hence the composition of a squared norm with an affine function. A sum of convex quadratics is a convex quadratic.
\end{proof}

\subsubsection{Proof of Proposition~\ref{prop:oracle}}
\label{app:prop_oracle}

\proporacle*

\begin{proof}
By Lemma~\ref{lem:convex}, $V_t(a,b)$ is a convex quadratic in $(a,b) \in \R^2$, so any stationary point is a global minimizer; Cauchy--Schwarz applied to $C = \E\inner{U}{\tilde V}$ also gives $C^2 \le \sigma^2 \sigma_f^2$ and hence $S \ge 0$.

Take $V_\infty(a, b) := \lim_{N \to \infty} V_t(a, b)$, which drops the $b^2\sigma_f^2/N$ term. For fixed $a$, unconstrained minimization in $b$ gives
\[
    \frac{\partial V_\infty}{\partial b} = \frac{-2(1{-}a{-}b)\sigma_f^2 - 2a\,C}{n} = 0
    \;\Longrightarrow\;
    b_\infty^\star(a) = (1-a) + a\,\frac{C}{\sigma_f^2},
\]
which matches the $b^\star$ component of~\eqref{eq:oracle_ab} at $a = a^\star$. At $b = b_\infty^\star(a)$, the residual $(1{-}a{-}b_\infty^\star) = -aC/\sigma_f^2$, so the numerator of the $1/n$ term in $V_\infty$ expands as
\[
    (1{-}a{-}b)^2\sigma_f^2 + 2a(1{-}a{-}b)\,C\,\Big|_{b = b_\infty^\star(a)}
    = a^2\frac{C^2}{\sigma_f^2} \,-\, 2a^2\frac{C^2}{\sigma_f^2}
    = -\,a^2\frac{C^2}{\sigma_f^2},
\]
and adding the $a^2\sigma^2$ contribution gives $a^2(\sigma^2 - C^2/\sigma_f^2) = a^2 S$. The reduced objective is therefore
\[
    V_\infty(a) = (1{-}a)^2\Bt^2 + a^2 S/n.
\]

Assume $\Bt^2 + S/n > 0$ (the degenerate case is handled below). Since $\Bt^2 \ge 0$ trivially and $S \ge 0$ as shown above, the univariate quadratic $V_\infty(a)$ has strictly positive leading coefficient, and its unique minimizer is
\[
    -2(1{-}a)\Bt^2 + 2a\,S/n = 0 \;\Longrightarrow\; a^\star = \frac{n\Bt^2}{n\Bt^2 + S} = \frac{\Bt^2}{\Bt^2 + S/n} \in [0, 1],
\]
with the $[0,1]$ range automatic from non-negativity of both coefficients. Back-substitution yields $b^\star = (1{-}a^\star) + a^\star C/\sigma_f^2$. The first-order condition rearranges to $(1{-}a^\star)\Bt^2 = a^\star \cdot S/n$, which gives
\[
    V^\star_\infty
    = (1{-}a^\star)^2\Bt^2 + (a^\star)^2 S/n
    = (1{-}a^\star)\bigl[a^\star S/n\bigr] + a^\star\bigl[a^\star S/n\bigr]
    = a^\star \cdot S/n.
\]
In the degenerate case $\Bt^2 = S = 0$, the bias vanishes for every $a$ and the variance $a^2 S/n$ vanishes along $b = b_\infty^\star(a)$, so $V_\infty \equiv 0$ along this curve and any such $(a,b)$ is optimal.
\end{proof}

\subsubsection{Finite-\texorpdfstring{$N$}{N} refinement}
\label{app:finite_N}

For completeness and to support the main-text claim that the asymptotic minimizer differs from the exact finite-$N$ minimizer by $O(1/N)$, we record the exact stationary point of $V_t(a, b)$.
Throughout this subsection we assume $\sigma_f^2(\theta_t) > 0$, so that the ratios $\rho := C/\sigma_f^2$ and $H := \sigma_f^2/(N{+}n)$ are well-defined; the degenerate case $\sigma_f^2 = 0$ implies $C = 0$ by Cauchy--Schwarz, $\gnf_t$ and $\gtNf_t$ are deterministic conditional on $\theta_t$, and the $b$ coordinate enters $\glam$ only through a deterministic shift.

For fixed $a$, the first-order condition $\partial V_t/\partial b = 0$ gives
\[
    \frac{-2(1{-}a{-}b)\sigma_f^2 - 2a\,C}{n} + \frac{2b\,\sigma_f^2}{N} = 0
    \;\Longrightarrow\;
    b_N^\star(a) = \frac{N}{N+n}\!\left((1-a) + a\,\rho\right).
\]
Substituting $b_N^\star(a)$ back into~\eqref{eq:mse} and collecting terms in $H$ reduces the bivariate MSE to a univariate quadratic in $a$,
\begin{equation}
    V_N(a) = (1-a)^2\,\Bt^2 \;+\; \frac{a^2\,S}{n} \;+\; H\,(1 - a + a\,\rho)^2,
    \qquad \rho := \frac{C}{\sigma_f^2}, \quad S := \sigma^2 - \frac{C^2}{\sigma_f^2},
    \label{eq:reduced_finite_N}
\end{equation}
with unique \emph{unconstrained} finite-$N$ minimizer
\begin{equation}
    a_N^\star = \frac{\Bt^2 + H(1-\rho)}{\Bt^2 + S/n + H(1-\rho)^2},
    \qquad b_N^\star(a_N^\star) = \frac{N}{N+n}\bigl((1 - a_N^\star) + a_N^\star\,\rho\bigr).
    \label{eq:exact_a}
\end{equation}
Unlike the asymptotic ratio in Proposition~\ref{prop:oracle}, $a_N^\star$ need not lie in $[0, 1]$ when $\rho \notin [0, 1]$; on a bounded feasible set $\mathcal K \subset \R^2$, the constrained minimizer is the unique minimizer of the convex quadratic $V_t$ over $\mathcal K$ and generally differs from $(a_N^\star, b_N^\star(a_N^\star))$.
Since $H = O(1/N)$, taking $N \to \infty$ in~\eqref{eq:exact_a} recovers $a^\star$ from Proposition~\ref{prop:oracle}, with discrepancy $|a_N^\star - a^\star| = O(1/N)$ for any fixed $\theta_t$.

\subsubsection{Proof of Corollary~\ref{cor:variance_floor}}
\label{app:cor_variance_floor}

\corvariancefloor*

\begin{proof}
Restricting~\eqref{eq:mse} to the unbiased slice $a_t \equiv 1$ collapses the bias term and gives
\[
    \Var_t(\glam)\big|_{a_t=1} \;=\; \frac{b_t^2\,\sigma_f^2 + \sigma^2 - 2\,b_t\,C}{n} \,+\, \frac{b_t^2\,\sigma_f^2}{N}
    \;=\; \frac{\sigma^2}{n} \,+\, b_t^2\,\sigma_f^2\!\left(\tfrac{1}{n} + \tfrac{1}{N}\right) \,-\, \frac{2\,b_t\,C}{n}.
\]
Setting the derivative in $b_t$ to zero gives the minimizer $b_t^\dagger = \frac{N}{N+n}\cdot\frac{C}{\sigma_f^2}$. Substituting back,
\[
    \inf_{b_t}\Var_t(\glam)\big|_{a_t=1}
    \;=\; \frac{\sigma^2}{n} \,-\, \frac{N}{N+n}\cdot\frac{C^2}{n\,\sigma_f^2}
    \;=\; \frac{S}{n} \,+\, \frac{n}{N+n}\cdot\frac{C^2}{n\sigma_f^2}
    \;=\; \frac{S}{n} \,+\, \frac{C^2}{(N+n)\,\sigma_f^2},
\]
where the second equality uses $\sigma^2 = S + C^2/\sigma_f^2$; this is the display of~\eqref{eq:variance_floor}, and in particular the infimum is bounded below by $S/n$. The strict-improvement claim follows by comparing with~\eqref{eq:oracle_ab}: the optimal ABC MSE $V_\infty^\star = a^\star S/n$ satisfies $V_\infty^\star \le S/n$ with strict inequality whenever $a^\star < 1$, equivalently whenever $S > 0$ and $\Bt^2 < \infty$; the gap is significant whenever $\Bt^2 < S/n$ (good-teacher regime by Corollary~\ref{cor:regimes}).
\end{proof}

\subsection{Online controller and end-to-end rate (\Secref{sec:plugin})}
\label{app:plugin_proofs}

The full pseudocode for ABC-SSL with the joint corrected-norm controller appears as Algorithm~\ref{alg:abcssl} at the end of this section.

\subsubsection{Proof of Theorem~\ref{thm:main}}
\label{app:thm_main}

The main statement gives the constant-step-size rate~\eqref{eq:main_bound}; we additionally prove a variance-tuned analogue that yields the same bound with $\bar V$ replaced by an a-posteriori variance term.

\thmmain*

\begin{proof}
We follow the standard smooth descent argument for biased SGD~\citep{Ajalloeian2020OnTC,Ghadimi2013StochasticFA}, adapted to the ABC parameterization via Lemma~\ref{lem:bias}. Write $\mu_t := \nabla\Ls(\theta_t)$ and $\bar{g}_t := \E_t[\glam]$, and note that $\bar{g}_t = \mu_t + (1-a_t)B(\theta_t)$ by Lemma~\ref{lem:bias}; since $\theta_t$ and $(a_t,b_t)$ are $\mathcal{F}_{t-1}$-measurable, both $\mu_t$ and $\bar{g}_t$ are deterministic given $\mathcal{F}_{t-1}$ and pass through $\E_t$ as constants.

By $\beta$-smoothness of $\Ls$ (Assumption~\ref{asm:smooth}) applied to the update $\theta_{t+1} = \theta_t - \eta\glam$,
\[
    \Ls(\theta_{t+1}) \;\le\; \Ls(\theta_t) \,-\, \eta\inner{\mu_t}{\glam} \,+\, \tfrac{\beta\eta^2}{2}\norm{\glam}^2.
\]
Taking $\E_t$ and applying the vector decomposition $\E_t\!\norm{\glam}^2 = \norm{\bar{g}_t}^2 + \Var_t(\glam)$ gives
\[
    \E_t\bigl[\Ls(\theta_{t+1})\bigr] \;\le\; \Ls(\theta_t) \,-\, \eta\inner{\mu_t}{\bar{g}_t} \,+\, \tfrac{\beta\eta^2}{2}\bigl(\norm{\bar{g}_t}^2 + \Var_t(\glam)\bigr).
\]
The polarization identity $\inner{\mu_t}{\bar{g}_t} = \tfrac{1}{2}(\norm{\mu_t}^2 + \norm{\bar{g}_t}^2 - \norm{\bar{g}_t - \mu_t}^2)$, combined with $\bar{g}_t - \mu_t = (1-a_t)B(\theta_t)$ and hence $\norm{\bar{g}_t - \mu_t}^2 = (1-a_t)^2\Bt^2$, rewrites the right-hand side as
\[
    \Ls(\theta_t) \,-\, \tfrac{\eta}{2}\norm{\mu_t}^2 \,-\, \tfrac{\eta}{2}(1-\eta\beta)\norm{\bar{g}_t}^2 \,+\, \tfrac{\eta}{2}(1-a_t)^2\Bt^2 \,+\, \tfrac{\beta\eta^2}{2}\Var_t(\glam).
\]
The step-size condition $\eta \le 1/\beta$ ensures $-\tfrac{\eta}{2}(1-\eta\beta)\norm{\bar{g}_t}^2 \le 0$, which we drop; rearranging yields the per-step bound
\[
    \norm{\mu_t}^2 \;\le\; \tfrac{2}{\eta}\bigl(\Ls(\theta_t) - \E_t[\Ls(\theta_{t+1})]\bigr) \,+\, (1-a_t)^2\Bt^2 \,+\, \eta\beta\,\Var_t(\glam).
\]
Taking total expectations, summing over $t = 1, \ldots, T$, and telescoping, the descent terms collapse to $\Ls(\theta_1) - \E[\Ls(\theta_{T+1})] \le \Delta_0$ (Assumption~\ref{asm:bounded}, using $\E[\Ls(\theta_{T+1})] \ge \inf_\theta\Ls(\theta)$). Dividing by $T$, defining $\bar B^2 := \frac{1}{T}\sum_t \E[(1-a_t)^2\Bt^2]$ and $\bar V^{\mathrm{var}} := \frac{1}{T}\sum_t \E[\Var_t(\glam)]$, gives
\begin{equation}
    \frac{1}{T}\sum_{t=1}^T \E\norm{\nabla\Ls(\theta_t)}^2 \;\le\; \frac{2\Delta_0}{\eta T} \,+\, \bar B^2 \,+\, \eta\beta\,\bar V^{\mathrm{var}}.
    \label{eq:main_bound_general}
\end{equation}
At $\eta = 1/\beta$, the right-hand side becomes $2\beta\Delta_0/T + \bar B^2 + \bar V^{\mathrm{var}} = 2\beta\Delta_0/T + \bar V$, recovering~\eqref{eq:main_bound}.

\paragraph{Variance-tuned step size.}
Let $\mathcal V > 0$ be any deterministic upper bound on $\bar V^{\mathrm{var}}$ and choose $\eta_{\mathcal V} := \min\bigl(1/\beta, \sqrt{2\Delta_0/(\beta\mathcal V T)}\bigr)$. Applying~\eqref{eq:main_bound_general} at $\eta = \eta_{\mathcal V}$, the inequality $1/\min(A,B) \le 1/A + 1/B$ for positive $A, B$ gives
\[
    \frac{2\Delta_0}{\eta_{\mathcal V} T} \;\le\; \frac{2\beta\Delta_0}{T} + \frac{2\Delta_0}{T}\sqrt{\frac{\beta\mathcal V T}{2\Delta_0}} = \frac{2\beta\Delta_0}{T} + \sqrt{\frac{2\beta\Delta_0\,\mathcal V}{T}},
\]
while $\eta_{\mathcal V}\beta\bar V^{\mathrm{var}} \le \sqrt{2\Delta_0/(\beta\mathcal V T)}\cdot\beta\bar V^{\mathrm{var}} \le \sqrt{2\beta\Delta_0\mathcal V/T}$ using $\bar V^{\mathrm{var}} \le \mathcal V$. Combining,
\begin{equation}
    \frac{1}{T}\sum_{t=1}^T \E\norm{\nabla\Ls(\theta_t)}^2 \;\le\; \bar B^2 \,+\, \frac{2\beta\Delta_0}{T} \,+\, 2\sqrt{\frac{2\beta\Delta_0\,\mathcal V}{T}}.
    \label{eq:main_bound_vtuned}
\end{equation}
The same bound holds with $\mathcal V$ replaced by $\bar V^{\mathrm{var}}$ a posteriori.
\end{proof}

\subsubsection{Proof of Lemma~\ref{lem:faithful}}
\label{app:lem_faithful}

\lemfaithful*

\begin{proof}
Substituting~\eqref{eq:norm_decomp} into~\eqref{eq:ell_def} and taking conditional expectation,
\[
    \E_t[\ell_t(a,b)]
    = \norm{\nabla\Ls(\theta_t)}^2 + 2(1-a)H_t + (1-a)^2\Bt^2 + \Var_t(\glam(a,b)) - 2(1-a)H_t,
\]
in which the $\pm 2(1-a)H_t$ terms cancel. The remainder is $\norm{\nabla\Ls(\theta_t)}^2 + (1-a)^2\Bt^2 + \Var_t(\glam(a,b)) = \norm{\nabla\Ls(\theta_t)}^2 + V_t(a,b)$ by the bias--variance decomposition of Proposition~\ref{prop:mse}.
\end{proof}

\subsubsection{Split-batch unbiased estimator of $H_t$}
\label{app:lem_split_H}

\begin{lemma}[Split-batch unbiased estimator of $H_t$]
\label{lem:split_H}
Partition the labeled batch $\mathcal B_t^{\mathrm{lab}}$ into two halves $A, B$ of size $n/2$ each (assume $n$ even; otherwise drop one sample), and let
\[
    g_A := \tfrac{2}{n}\sum_{i \in A}\nabla\loss(\theta_t; x_i, y_i),
    \quad
    g_A^f := \tfrac{2}{n}\sum_{i \in A}\nabla\loss(\theta_t; x_i, f(x_i)),
    \quad
    d_A := g_A - g_A^f,
\]
and analogously $g_B, g_B^f, d_B$. Then
\begin{equation}
    \widehat{H}_t := -\tfrac{1}{2}\bigl(\inner{g_A}{d_B} + \inner{g_B}{d_A}\bigr)
    \label{eq:H_hat}
\end{equation}
satisfies $\E_t[\widehat H_t] = H_t$.
\end{lemma}

\begin{proof}
Since the partition is deterministic in the indices, exchangeability of the i.i.d.\ labeled samples makes halves $A$ and $B$ independent draws from $P^{n/2}$, so $\E_t[g_A] = \E_t[g_B] = \nabla\Ls(\theta_t)$ and $\E_t[d_A] = \E_t[d_B] = \nabla\Ls(\theta_t) - \nabla\Lf(\theta_t) = -B(\theta_t)$. By independence,
\[
    \E_t\inner{g_A}{d_B} = \inner{\E_t[g_A]}{\E_t[d_B]} = -\inner{\nabla\Ls(\theta_t)}{B(\theta_t)} = -H_t,
\]
and the symmetric pair gives the same value. Averaging and negating yields $\E_t[\widehat H_t] = H_t$.
\end{proof}

\subsubsection{Proof of Proposition~\ref{prop:joint_regret}}
\label{app:prop_joint_regret}

The main statement gives the projected-AdaGrad $O(\sqrt T)$ regret.

\propjointregret*

\begin{proof}
$\ell_t$ is convex on $\mathcal K$ (its Hessian $2\mA_t$ is PSD), and $\norm{\nabla\ell_t(\vx)} = \norm{2\mA_t\vx + 2\vr_t}$ is uniformly bounded by some $G_\ell$ on $\mathcal K$ under the boundedness assumption. Projected OGD with step size $D_{\mathcal K}/(G_\ell\sqrt T)$ attains the standard diameter-times-gradient regret bound, and projected AdaGrad attains the same rate adaptively without knowing $T$ in advance~\citep{Ward2018AdaGradSS}, yielding~\eqref{eq:joint_regret_V} via Lemma~\ref{lem:faithful}: taking $\E$ of the surrogate-regret inequality and using $\E_t[\ell_t(a,b)] = V_t(a,b) + \norm{\nabla\Ls(\theta_t)}^2$ on both sides, the $\norm{\nabla\Ls(\theta_t)}^2$ terms cancel in the difference.
We note that a faster $O(\log T)$ rate via Online Newton Step is \emph{not} available in general: $\nabla^2\ell_t = 2\mA_t$ is the Gram matrix of $(\dt, c_t)$ and is singular whenever $\dt$ and $c_t$ are linearly dependent, in which case $\ell_t$ is affine (with slope driven by $\widehat H_t$) along the null direction and is exp-concave for no $\mu > 0$; recovering $O(\log T)$ would require an additional lower bound $\lambda_{\min}(\mA_t) \ge \alpha > 0$.
\end{proof}

\subsubsection{Proof of Corollary~\ref{cor:end_to_end}}
\label{app:cor_end_to_end}

\corendtoend*

\begin{proof}
By Theorem~\ref{thm:main} at $\eta = 1/\beta$, $\frac{1}{T}\sum_t \E\norm{\nabla\Ls(\theta_t)}^2 \le 2\beta\Delta_0/T + \bar V$, where $\bar V = \frac{1}{T}\sum_t \E[V_t(a_t, b_t)]$ is evaluated along the realized $(a_t, b_t)$ path. By Proposition~\ref{prop:joint_regret} divided by $T$, $\bar V \le \bar V_T^\dagger + R(T)/T$ with $R(T) = O(\sqrt T)$, giving the first inequality of~\eqref{eq:end_to_end}.

The second inequality is immediate: every pair $(1, b)$ with $b \in [0, b_{\max}]$ lies in $\mathcal K$, so minimizing the time-averaged MSE over all of $\mathcal K$ can only do better, i.e., $\bar V_T^\dagger \le \min_{b \in [0, b_{\max}]} \tfrac1T\sum_t \E[V_t(1, b)] = \bar V_T^{\textrm{PP-SSL}}$.

It remains to prove strictness under~\eqref{eq:end_to_end_strict}. Write $J(a, b) := \tfrac1T\sum_t \E[V_t(a, b)]$ for fixed (deterministic) $(a, b)$. By the sum-of-squares form~\eqref{eq:mse_sos},
\[
    J(a, b) \;=\; (1-a)^2\,\bar M \;+\; \frac1T\sum_t \frac{\E\norm{\phi_t(a, b)}^2}{n} \;+\; b^2\,\frac{\bar\sigma_f^2}{N},
    \qquad \phi_t(a, b) := (1{-}a{-}b)\,\tilde V_t + a\,U_t,
\]
where $U_t, \tilde V_t$ are the centered per-sample labeled and pseudolabeled gradients at $\theta_t$ (as in Lemma~\ref{lem:convex}), $\bar M := \tfrac1T\sum_t \E[\Bt^2]$, and $\bar\sigma_f^2 := \tfrac1T\sum_t \E[\sigma_{f,t}^2]$. Let $b^c \in [0, b_{\max}]$ attain $\bar V_T^{\textrm{PP-SSL}} = J(1, b^c)$ (the minimum exists: $J(1, \cdot)$ is a continuous quadratic on a compact interval). Consider the segment
\[
    (a_\varepsilon, b_\varepsilon) \;:=\; (1-\varepsilon)\,(1, b^c) + \varepsilon\,(0, 1) \;\in\; \mathcal K, \qquad \varepsilon \in [0, 1],
\]
feasible since $\mathcal K$ is convex and contains both endpoints ($b_{\max} \ge 1$). The map $(a, b) \mapsto \phi_t(a, b)$ is affine and $\phi_t(0, 1) = 0$, so $\phi_t(a_\varepsilon, b_\varepsilon) = (1-\varepsilon)\,\phi_t(1, b^c)$. Writing $P := \tfrac1{Tn}\sum_t \E\norm{\phi_t(1, b^c)}^2$ and $R := \bar\sigma_f^2/N$,
\[
    J(a_\varepsilon, b_\varepsilon)
    \;=\; \varepsilon^2\,\bar M \;+\; (1-\varepsilon)^2\,P \;+\; \bigl(b^c + \varepsilon(1 - b^c)\bigr)^2 R,
\]
and hence
\[
    J(a_\varepsilon, b_\varepsilon) - \bar V_T^{\textrm{PP-SSL}}
    \;=\; -2\varepsilon\bigl[P - b^c(1 - b^c)\,R\bigr] \;+\; \varepsilon^2\bigl[\bar M + P + (1 - b^c)^2 R\bigr].
\]
Per round, $\phi_t(1, b^c) = U_t - b^c \tilde V_t$ and $\min_{\rho \in \R} \E_t\norm{U_t - \rho \tilde V_t}^2 = \sigma_t^2 - C_t^2/\sigma_{f,t}^2 = S_t$ (using $\sigma_{f,t}^2 > 0$), so $P \ge \tfrac1{Tn}\sum_t \E[S_t]$; and $b^c(1 - b^c) \le 1/4$. The second condition of~\eqref{eq:end_to_end_strict} therefore gives $P - b^c(1-b^c)R \ge \tfrac1{Tn}\sum_t\E[S_t] - \bar\sigma_f^2/(4N) > 0$, while the first makes the $\varepsilon^2$ coefficient finite. Hence $J(a_\varepsilon, b_\varepsilon) < \bar V_T^{\textrm{PP-SSL}}$ for all sufficiently small $\varepsilon > 0$, and since $(a_\varepsilon, b_\varepsilon) \in \mathcal K$ we conclude $\bar V_T^\dagger < \bar V_T^{\textrm{PP-SSL}}$.
\end{proof}

\subsubsection{Algorithm pseudocode}
\label{app:alg}

\begin{algorithm}[h]
\caption{ABC-SSL: Adaptive Bias Control for Semi-Supervised Learning (with joint corrected-norm controller).}
\label{alg:abcssl}
\begin{algorithmic}[1]
\REQUIRE Teacher $f$; SGD step size $\eta$; OCO step size for $(a, b)$; feasible set $\mathcal K = [0, 1] \times [0, b_{\max}]$; iterations $T$
\STATE Initialize $\theta_1 \in \R^d$, $(a_1, b_1) \leftarrow (1, 0)$, OCO state \hfill\COMMENT{\textcolor{magenta}{labeled-only first step}}
\FOR{$t = 1, \ldots, T$}
    \STATE Draw $\mathcal{B}_t^{\mathrm{lab}} = \{(x_t^i, y_t^i)\}_{i=1}^n$ split into halves $A, B$ of size $n/2$; draw $\mathcal{B}_t^{\mathrm{unl}} = \{x_t^i\}_{i=n+1}^{n+N}$
    \STATE Compute aggregate component gradients $\glab_t, \gnf_t, \gtNf_t$ and half-batch gradients $g_A, g_A^f, g_B, g_B^f$
    \STATE $\dt \leftarrow \glab_t - \gnf_t$,\quad $c_t \leftarrow \gtNf_t - \gnf_t$,\quad $d_A \leftarrow g_A - g_A^f$,\quad $d_B \leftarrow g_B - g_B^f$
    \STATE $\glam \leftarrow \gnf_t + a_t\,\dt + b_t\,c_t$ \hfill\COMMENT{\textcolor{magenta}{ABC gradient~\eqref{eq:abc_grad}}}
    \STATE $\theta_{t+1} \leftarrow \theta_t - \eta\,\glam$ \hfill\COMMENT{\textcolor{magenta}{SGD update, Theorem~\ref{thm:main}}}
    \STATE \emph{Scalar primitives:}\; compute $\norm{\dt}^2, \norm{c_t}^2, \inner{\dt}{c_t}, \inner{\gnf_t}{\dt}, \inner{\gnf_t}{c_t}, \norm{\gnf_t}^2$
    \STATE $\widehat H_t \leftarrow -\tfrac{1}{2}\bigl(\inner{g_A}{d_B} + \inner{g_B}{d_A}\bigr)$ \hfill\eqref{eq:H_hat}
    \STATE Form quadratic $\ell_t(a, b) = a^2\norm{\dt}^2 + b^2\norm{c_t}^2 + 2ab\inner{\dt}{c_t} + 2a\inner{\gnf_t}{\dt} + 2b\inner{\gnf_t}{c_t} + \norm{\gnf_t}^2 - 2(1{-}a)\widehat H_t$
    \STATE Update OCO state on $\ell_t$ (projected AdaGrad); set $(a_{t+1}, b_{t+1}) \leftarrow$ projected iterate on $\mathcal K$ \hfill (Proposition~\ref{prop:joint_regret})
\ENDFOR
\RETURN $\theta_\tau$ for uniform random $\tau \in \{1, \ldots, T\}$
\end{algorithmic}
\end{algorithm}

The half-batch construction in Lines~3--5 produces the four aggregate gradients $g_A, g_A^f, g_B, g_B^f$ needed by~\eqref{eq:H_hat}; the labeled forward pass is shared between the labeled batch gradient $\glab_t$ and the labeled half-batches, so each step requires three batch backwards ($\glab_t, \gnf_t, \gtNf_t$) plus two extra labeled half-batch backwards. All scalar primitives are computable from the aggregate component gradients in a single reduction pass after the SGD step; the projected OCO step is in $\R^2$ and is negligible relative to the gradient computations.

\begin{remark}[EMA smoothing of $\widehat H_t$ in experiments]
\label{rem:ema_heuristic}
In the experiments of \Secref{sec:experiments} we additionally smooth the split-batch estimate with an exponential moving average, $\widehat H_t^{\mathrm{ema}} = (1 - \alpha_H)\,\widehat H_{t-1}^{\mathrm{ema}} + \alpha_H\,\widehat H_t$ (rate $\alpha_H = 0.05$, zero-initialized), and feed $\widehat H_t^{\mathrm{ema}}$ rather than $\widehat H_t$ into $\ell_t$, as a variance-reduction heuristic. The EMA forfeits the conditional unbiasedness required by Lemma~\ref{lem:faithful}: $\E_t[\widehat H_t^{\mathrm{ema}}] = \alpha_H H_t + (1 - \alpha_H)\widehat H_{t-1}^{\mathrm{ema}}$, so the expected surrogate acquires the $\mathcal F_{t-1}$-measurable, $a$-affine perturbation $2(1{-}a)\,\delta_t$ with $\delta_t := (1 - \alpha_H)\bigl(H_t - \widehat H_{t-1}^{\mathrm{ema}}\bigr)$, which shifts the per-round minimizer, and the regret bound of Proposition~\ref{prop:joint_regret} then holds only up to an additional drift term of order $\tfrac1T\sum_t \E|\delta_t|$. All theoretical statements (in particular Corollary~\ref{cor:end_to_end}) refer to Algorithm~\ref{alg:abcssl} as written, i.e., with the raw conditionally unbiased $\widehat H_t$.
\end{remark}

\section{Synthetic preference experiment: supplementary details}
\label{app:synth_details}

\subsection{Data generation}

The ground-truth reward function is $r^\star(x) = {w^\star}^\top x$ with $w^\star \sim \mathcal{N}(0, I_m)$ drawn once per trial and held fixed, where $m = 20$.
Features $x_1, x_2 \sim \mathcal{N}(0, I_m)$ are drawn independently for each comparison pair.
The ground-truth preference label is
\begin{equation}
    y = \mathbf{1}\!\bigl\{r^\star(x_1) - r^\star(x_2) + \epsilon > 0\bigr\}, \qquad \epsilon \sim \mathcal{N}(0, \sigma_\epsilon^2),
    \label{eq:synth_pref}
\end{equation}
with $\sigma_\epsilon = 0.5$ modeling annotator noise.
This noise level makes roughly $30\%$ of pairs ambiguous (i.e., the noisy label can differ from the noiseless one), providing a realistic level of annotator disagreement.

The biased teacher uses reward $r^f(x) = (w^\star + \mu v)^\top x$ where $v \sim \mathcal{N}(0, I_m)$ is drawn once per trial and normalized to unit norm, and $\mu \ge 0$ controls bias intensity.
Teacher pseudo labels are generated via the same Bradley--Terry mechanism but with $r^f$ replacing $r^\star$ and no additional noise (the teacher is deterministic given its weights).
As $\mu$ increases, the teacher's reward direction deviates further from $w^\star$, increasing the fraction of pseudo labels that disagree with the ground truth.

\subsection{Training details}

We train a linear reward model $r_\phi(x) = \phi^\top x$ by minimizing the Bradley--Terry negative log-likelihood~\eqref{eq:rm_loss} on pairwise feature differences $x_1 - x_2$.
All methods use Adam with learning rate $10^{-3}$, $(\beta_1, \beta_2) = (0.9, 0.999)$, batch size $32$, and run for $T = 2{,}000$ iterations.

The headline runs reported in \Secref{sec:exp_synthetic} use the joint corrected-norm controller of \Secref{sec:plugin}. In this small-scale Bradley--Terry setting, per-example gradients are essentially free, so we additionally report a finite-sample plug-in estimator as a sanity-check companion: at each step we compute the per-example second-moment statistics directly and plug them into the closed-form oracle of Proposition~\ref{prop:oracle}.
The two controllers track each other closely throughout training (Figure~\ref{fig:synth_traj}), confirming that the corrected-norm surrogate $\ell_t$ tracks the per-step oracle without requiring per-example gradient covariance estimates.

For the plug-in sanity check, $(a_t, b_t)$ are computed at each step from per-sample residual statistics using the closed-form expressions derived in \Secref{sec:plugin}.
In this Bradley--Terry setting the per-sample gradient factors as $\nabla\loss(\phi; z_i, y_i) = z_i\bigl(\sigmoid(z_i^\top \phi) - y_i\bigr)$, so $\norm{\nabla\loss(\phi; z_i, y_i)}^2 = \norm{z_i}^2 r_i^2$ where $r_i := \sigmoid(z_i^\top \phi) - y_i$ is the per-sample residual available from the forward pass. Writing $r_i^f := \sigmoid(z_i^\top \phi) - f(x_i)$ for the pseudo-label residual, the finite-sample-corrected estimators are
\begin{equation}
    \begin{aligned}
        \hat\sigma^2 &= \frac{n}{n-1}\!\left(\tfrac{1}{n}\textstyle\sum_i \norm{z_i}^2 r_i^2 - \norm{\glab_t}^2\right), \quad
        \hat\sigma_f^2 = \frac{n}{n-1}\!\left(\tfrac{1}{n}\textstyle\sum_i \norm{z_i}^2 (r_i^f)^2 - \norm{\gnf_t}^2\right), \\
        \hat C &= \frac{n}{n-1}\!\left(\tfrac{1}{n}\textstyle\sum_i \norm{z_i}^2 r_i\, r_i^f - \inner{\glab_t}{\gnf_t}\right),
    \end{aligned}
    \label{eq:stat_estimates}
\end{equation}
which are unbiased under i.i.d.\ sampling and computable as $O(nd)$ scalar operations over the forward-pass residuals.
The estimates $\hat\sigma^2$, $\hat\sigma_f^2$, $\hat{C}$ are smoothed with an exponential moving average (rate $\alpha = 0.02$) and fed into the oracle formula (Proposition~\ref{prop:oracle}).
At the first step the estimator is initialized to $(a_1, b_1) = (1, 0)$, corresponding to labeled-only training, and subsequent updates are predictable with respect to the filtration $\mathcal{F}_{t-1}$ as required by Theorem~\ref{thm:main}.

For PP-SSL, $\lambda_t$ is tuned online by projected gradient descent on $\norm{\hat{g}_t}^2$ with step size $0.01$ and clipping to $[0,1]$, following the procedure described in \citet{Shoham2025PredictionPoweredSL}.

All methods share the same practical minibatch training loop and optimizer; they differ only in the gradient estimator.

\subsection{Metrics}

We report three evaluation metrics, all computed on a held-out test set of $10{,}000$ pairs with ground-truth labels:
\begin{itemize}
    \item \textbf{Preference accuracy}: the fraction of pairs whose predicted preference direction matches the ground truth. This is the primary metric reported in Figure~\ref{fig:synthetic}.
    \item \textbf{Reward MSE}: $\norm{\phi - w^\star}^2$, measuring how closely the learned weights recover the true reward direction.
    \item \textbf{Bradley--Terry log-loss}: the negative log-likelihood~\eqref{eq:rm_loss} evaluated on the test set, measuring calibration of the predicted preference probabilities.
\end{itemize}
The main text reports preference accuracy; reward MSE and log-loss exhibit consistent trends and are available upon request.

\subsection{Subgroup analysis}

To examine performance on easy versus hard preference comparisons, we split the test set by the magnitude of the true reward difference $\Delta_i = |r^\star(x_{1,i}) - r^\star(x_{2,i})|$.
Pairs with $\Delta_i$ above the test-set median form Group~A (easy: clear preference signal), and pairs below form Group~B (hard: ambiguous comparisons where annotator noise is most impactful).
This split is performed on the test set only and does not affect training.

The subgroup analysis reveals that ABC-Align's advantage is concentrated on hard pairs (Figure~\ref{fig:synthetic}c).
On easy pairs, all methods that use pseudo labels perform similarly, since the teacher's ranking is usually correct regardless of bias.
On hard pairs, however, the labeled-only and unbiased estimators (DR, PP-SSL) are variance-limited because $n = 50$ provides little signal for near-tied comparisons.
ABC-Align's ability to lean more heavily on pseudo labels while applying a controlled correction yields the largest gains in this regime.

\subsection{Experimental conditions}

We use $n = 50$ labeled pairs with $N = 5{,}000$ unlabeled pairs ($100\times$ ratio).
For the main synthetic curve in Figure~\ref{fig:synthetic}, we sweep $\mu \in \{0.0, 0.2, \ldots, 1.0\}$ and average over $25$ trials.
For the teaser heatmap (Figure~\ref{fig:teaser}) and the trajectory and MSE ablations (Figures~\ref{fig:synth_traj} and~\ref{fig:synth_mse}), we use a coarser $\mu$ grid and average over $10$ trials.
Across all settings, we vary the random seed for $w^\star$, $v$, and the sampled comparison pairs.
Error bars in Figure~\ref{fig:synthetic} show $\pm 1$ standard error across trials.

\section{Tabular versions of the main-text figures}
\label{app:tables}

For readers who prefer exact numbers, Table~\ref{tab:rm_ood} gives the tabular version of Figure~\ref{fig:rm_ood}.

\begin{table}[t]
  \centering
  \caption{Out-of-domain reward-model benchmarks (label fraction $0.05$, $\rho = 0.10$). Best per column in \textbf{bold}; $\Delta$ is ABC-Align minus the best baseline.}
  \label{tab:rm_ood}
  \small
  \begin{tabular}{lcccc}
    \toprule
    & \multicolumn{2}{c}{RewardBench} & \multicolumn{2}{c}{PPE} \\
    \cmidrule(lr){2-3}\cmidrule(lr){4-5}
    Method & overall & human pref. & human pref. & math best-of-$k$ \\
    \midrule
    Labeled-only & 0.611 & 0.600 & 0.569 & 0.544 \\
    Naive combined & 0.648 & 0.621 & 0.598 & 0.573 \\
    PP-SSL & 0.670 & 0.637 & 0.623 & 0.589 \\
    ABC-Align (ours) & \textbf{0.692} & \textbf{0.676} & \textbf{0.650} & \textbf{0.613} \\
    \midrule
    $\Delta$ vs.\ best baseline & $+0.022$ & $+0.039$ & $+0.027$ & $+0.024$ \\
    \bottomrule
  \end{tabular}
\end{table}

\section{LLM experiment details: RM, DPO, and GRPO}
\label{app:llm_details}

This appendix collects the training configurations for the three LLM-alignment experiments of \Secref{sec:exp_rm}--\Secref{sec:exp_grpo}. Across all three, the methods share architecture, tokenizer, optimizer, and update budget within each experiment; only the gradient estimator differs.

\subsection{Reward model training}
\label{app:rm_details}

Sequences are truncated to $4096$ tokens, training runs for two epochs in bf16, and the labeled and unlabeled batch sizes are $16$ and $128$, respectively. All methods use AdaGrad-Norm~\citep{Ward2018AdaGradSS} with the same base step size.
Although the codebase also supports a full-model ABC-Align reward-training path, the main comparison keeps the Llama-3.2-3B-Instruct backbone frozen and trains only the scalar score head, so that differences are attributable to the gradient estimator rather than to full-model optimization effects.
Evaluation combines in-domain held-out pairwise accuracy on the Skywork validation split with the out-of-domain RewardBench~\citep{Lambert2024RewardBenchER,Malik2025RewardBench2A} and PPE~\citep{Frick2024HowTE} benchmarks, emphasizing human-preference accuracy and the math best-of-$k$ proxy.

\subsection{DPO-based policy alignment}
\label{app:dpo_details}

All methods share the same architecture, tokenizer, and optimizer: AdamW~\citep{Kingma2014AdamAM} with learning rate $10^{-6}$, $(\beta_1, \beta_2) = (0.9, 0.999)$, and no weight decay.
Sequences are truncated to $1{,}024$ tokens.
Training runs for one epoch in bf16 with gradient checkpointing enabled.
The labeled batch size is $8$ and the unlabeled batch size is $32$.
For ABC-Align, $(a_t, b_t)$ is updated by the joint corrected-norm controller of \Secref{sec:plugin}, with the symmetric split-batch construction of~\eqref{eq:H_hat} replaced by the one-sided streaming variant $\widehat H_t^{\text{stream}} := -\inner{g_A}{d_B}$ for full-model memory feasibility; this reduces peak extra memory to a single full-model gradient buffer while preserving conditional unbiasedness ($\E_t[\widehat H_t^{\text{stream}}] = H_t$ by independence of $A, B$), so Proposition~\ref{prop:joint_regret} applies verbatim to the streaming construction.
In the runs, $\widehat H_t^{\text{stream}}$ is additionally EMA-smoothed (rate $\alpha_H = 0.05$) before being fed into $\ell_t$---a variance-reduction heuristic that forfeits exact conditional unbiasedness and adds a drift term to the regret bound (Remark~\ref{rem:ema_heuristic})---and $(a_{t+1}, b_{t+1})$ is produced by projected AdaGrad on $\ell_t$ over $\mathcal K = [0, 1] \times [0, b_{\max}]$.
For PP-SSL, $\lambda_t$ is adapted online via projected SGD on $\|\hat g_t\|^2$.

\paragraph{Adaptation to full-model DPO.}
Unlike the score-head reward model, DPO fine-tunes millions of policy parameters; the controller is designed to add at most one full-model gradient buffer on top of the standard DPO step.
The three batch gradients $\glab_t, \gnf_t, \gtNf_t$ are produced by the same labeled and unlabeled DPO forwards as a standard run, reusing the shared activations via the scalar-margin factorization of Lemma~\ref{lem:scalar_margin}.
All scalar primitives feeding $\ell_t$---$\|\dt\|^2$, $\|c_t\|^2$, $\inner{\dt}{c_t}$, $\inner{\gnf_t}{\dt}$, $\inner{\gnf_t}{c_t}$, and the streaming $\widehat H_t^{\text{stream}}$---are aggregate dot products on these component gradients, computed via a single scalar all-reduce per primitive under FSDP/ZeRO; no per-example gradient vectors are materialized.
The dominant cost ratio versus labeled-only SGD is $(N + 2n)/(N + n) \approx 1$ when $N \gg n$.

\subsection{On-policy GRPO}
\label{app:grpo_details}

We fine-tune Llama-3.2-3B with GRPO on grade-school math prompts with verifiable answers, sampling a rollout group of $G = 8$ completions per prompt.
Labeled prompts receive a verifiable correctness reward (the trusted, human-anchored source); unlabeled prompts are scored only by a teacher reward model.
Source-specific group-relative advantages are formed via~\eqref{eq:grpo_adv_label}, and the KL anchor to $\pi_{\mathrm{ref}}$ is treated as a deterministic offset as in~\eqref{eq:abc_grpo_update}.
At each iteration the per-group statistics $(\widehat{\norm{B_{t,\mathrm{G}}}^2}, \widehat\sigma_{t,\mathrm{G}}^2, \widehat\sigma_{t,\mathrm{G},f}^2, \widehat C_{t,\mathrm{G}})$ are estimated from the labeled prompt-group batch using the split-batch plug-in of~\eqref{eq:grpo_splitbatch} and substituted into the closed form~\eqref{eq:oracle_ab} to obtain $(a_t, b_t)$.
We sweep the labeled fraction over $\{0.05, 0.10, 0.30\}$ at a fixed teacher-noise level and average over seeds, reporting held-out GSM8K pass@$1$ on a disjoint evaluation split.
The four compared paradigms (Labeled-only, Naive combined, PP-SSL, ABC-Align) share the same rollout budget, optimizer, and KL coefficient; only the advantage-component gradient estimator differs.

\section{Two-Timescale Debiasing of the ABC-Align Estimator}
\label{app:two_timescale}

In this appendix we augment Algorithm~\ref{alg:abcssl} with an auxiliary fast-timescale update that asymptotically eliminates the residual squared bias appearing in Theorem~\ref{thm:main}, giving a convergence guarantee without a persistent bias floor. The construction is orthogonal to the choice of controller for $(a_t, b_t)$: it acts on the slow update through a debiased gradient $\widetilde g_t = \glam - (1-a_t)\xi_t$, leaving the controller unchanged.

\subsection{Motivation}
\label{app:tt_motivation}

Lemma~\ref{lem:bias} gives the conditional mean of the ABC-Align estimator as
\[
    \E_t[\glam] \;=\; \nabla\Ls(\theta_t) + (1-a_t)\,B(\theta_t), \qquad B(\theta) := \nabla\Lf(\theta) - \nabla\Ls(\theta),
\]
and Theorem~\ref{thm:main} correspondingly bounds the time-averaged squared gradient by
\[
    \frac{1}{T}\sum_{t=1}^T \E\norm{\nabla\Ls(\theta_t)}^2 \;\le\; \frac{2\Delta_0}{\eta T} \;+\; \frac{1}{T}\sum_{t=1}^T \E\bigl[(1{-}a_t)^2 \Bt^2\bigr] + \eta\beta\,\overline{\mathrm{Var}_t(\glam)}.
\]
The corrected-norm joint controller of \Secref{sec:plugin} attains the joint regret guarantee $\bar V_T \le \bar V_T^\dagger + O(T^{-1/2})$ against the best fixed pair in $\mathcal K$ (Proposition~\ref{prop:joint_regret}). But the per-round oracle of Proposition~\ref{prop:oracle} typically selects $a_t^\star < 1$ in order to trade controlled bias for variance reduction, and well-tuned pairs correspondingly retain an $\overline{(1-a_t)^2\Bt^2}$ contribution. Consequently, no matter how well the controller tracks the oracle, Theorem~\ref{thm:main} only guarantees convergence to a neighborhood of stationarity whose radius is controlled by this residual bias floor; the controller cannot remove it without abandoning the oracle (and hence its variance optimality).
The two-timescale construction below decouples these concerns: it eliminates the bias floor asymptotically while preserving whatever variance-optimality the controller delivers, since the fast-timescale update modifies the slow gradient without altering $(a_t, b_t)$.

A principled way to remove this floor is to augment the slow parameter update with a fast auxiliary iterate that tracks the unknown bias $B(\theta_t)$ online.
This is the classical two-timescale stochastic approximation paradigm of \citet{Borkar1997StochasticAW} and \citet{Borkar2008StochasticAA}: two coupled updates run at different rates so that, on the slow timescale, the fast iterate appears to have already converged.
The asymptotic theory of such schemes traces back to \citet{PolyakJuditsky1992AccelerationOS}, \citet{KondaTsitsiklis2003OnActorCritic}, and \citet{MokkademPelletier2006ConvergenceRO}.
A substantial recent effort has sharpened the finite-time theory: \citet{Doan2021NonlinearTT} obtains an $O(k^{-2/3})$ mean-square rate for nonlinear two-timescale stochastic approximation, \citet{HongWaiWangYang2023ATS} gives matching bilevel rates with explicit dependence on inner-problem conditioning, \citet{HanLiZhang2024FinitetimeDC} show decoupled convergence under a nested local-linearity condition, \citet{Doan2024FastNonlinearTT} attains the improved $O(1/k)$ rate via Ruppert--Polyak-style averaging of the operators, and \citet{Kwon2025TwoTimescaleLSA} analyze constant-stepsize variants via Wasserstein-geometric ergodicity.
Our debiasing scheme takes the same shape as a Robbins--Monro fixed-point iteration for $B(\theta)$ coupled with a slow parameter update on $\Ls$, and so inherits the TTSA structure.
Finally, the idea of debiasing a biased gradient by a running estimate that is itself updated online connects to the STORM-style momentum variance reduction of \citet{CutkoskyOrabona2019MomentumBV} and the broader debiased-SSL literature (\citealp{Chen2022DebiasedST}), though the present analysis does not rely on any of those techniques directly.

\subsection{Debiased estimator and algorithm}
\label{app:tt_algo}

Recall that the ABC-Align correction signal $\dt := \glab_t - \gnf_t$ satisfies $\E_t[-\dt] = B(\theta_t)$, so $-\dt$ is an unbiased (noisy) sample of the unknown bias.
We introduce an auxiliary variable $\xi_t \in \R^d$ that tracks $B(\theta_t)$ via the Robbins--Monro fixed-point update
\begin{equation}
    \xi_{t+1} \;=\; (1 - \gamma_t)\,\xi_t \;-\; \gamma_t\,\dt,
    \label{eq:tt_xi_update}
\end{equation}
with fast step size $\gamma_t > 0$.
Given $\xi_t$, we form the \emph{debiased} ABC-Align gradient estimator
\begin{equation}
    \widetilde{g}_t \;:=\; \glam \;-\; (1 - a_t)\,\xi_t,
    \label{eq:tt_debiased}
\end{equation}
and run the slow update
\begin{equation}
    \theta_{t+1} \;=\; \theta_t \;-\; \eta_t\,\widetilde{g}_t, \qquad \eta_t = o(\gamma_t).
    \label{eq:tt_slow_update}
\end{equation}
Because $\xi_t \in \mathcal{F}_{t-1}$ by construction, $\E_t[\xi_t] = \xi_t$, so the conditional mean of $\widetilde{g}_t$ is
\begin{equation}
    \E_t[\widetilde{g}_t] \;=\; \nabla\Ls(\theta_t) + (1-a_t)\bigl(B(\theta_t) - \xi_t\bigr) \;=\; \nabla\Ls(\theta_t) - (1-a_t)\,e_t,
    \label{eq:tt_condmean}
\end{equation}
where
\begin{equation}
    e_t \;:=\; \xi_t - B(\theta_t)
    \label{eq:tt_error_def}
\end{equation}
is the fast-timescale tracking error.
Driving $e_t \to 0$ in mean square therefore eliminates the bias of $\widetilde{g}_t$.

We package these updates in Algorithm~\ref{alg:two_timescale_abcssl}, which extends Algorithm~\ref{alg:abcssl} by a single $\R^d$-valued state $\xi_t$ and two scalar operations.

\begin{algorithm}[t]
\caption{Two-timescale ABC-Align with fast-timescale bias tracking}
\label{alg:two_timescale_abcssl}
\begin{algorithmic}[1]
\REQUIRE Teacher $f$, slow step sizes $\{\eta_t\}$, fast step sizes $\{\gamma_t\}$, iterations $T$
\STATE Initialize $\theta_1 \in \R^d$, $\xi_1 \leftarrow 0$, $(a_1, b_1) \leftarrow (1, 0)$ \hfill\COMMENT{\textcolor{magenta}{labeled-only at first step}}
\FOR{$t = 1, \ldots, T$}
    \STATE Draw labeled batch $\mathcal{B}_t^{\mathrm{lab}}$ (split into halves $A, B$) and unlabeled batch $\mathcal{B}_t^{\mathrm{unl}}$
    \STATE Compute $\glab_t$, $\gnf_t$, $\gtNf_t$, half-batch gradients $g_A, g_A^f, g_B, g_B^f$; set $\dt \leftarrow \glab_t - \gnf_t$, $c_t \leftarrow \gtNf_t - \gnf_t$, $d_A \leftarrow g_A - g_A^f$, $d_B \leftarrow g_B - g_B^f$
    \STATE $\glam \leftarrow \gnf_t + a_t\,\dt + b_t\,c_t$ \hfill\COMMENT{\textcolor{magenta}{ABC-Align gradient~\eqref{eq:abc_grad}}}
    \STATE $\widetilde{g}_t \leftarrow \glam - (1 - a_t)\,\xi_t$ \hfill\COMMENT{\textcolor{magenta}{debiased estimator~\eqref{eq:tt_debiased}}}
    \STATE $\theta_{t+1} \leftarrow \theta_t - \eta_t\,\widetilde{g}_t$ \hfill\COMMENT{\textcolor{magenta}{slow update~\eqref{eq:tt_slow_update}}}
    \STATE $\xi_{t+1} \leftarrow (1 - \gamma_t)\,\xi_t - \gamma_t\,\dt$ \hfill\COMMENT{\textcolor{magenta}{fast Robbins--Monro update~\eqref{eq:tt_xi_update}}}
    \STATE Form scalar primitives $\norm{\dt}^2, \norm{c_t}^2, \inner{\dt}{c_t}, \inner{\gnf_t}{\dt}, \inner{\gnf_t}{c_t}$, split-batch $\widehat H_t \leftarrow -\tfrac{1}{2}(\inner{g_A}{d_B} + \inner{g_B}{d_A})$
    \STATE $(a_{t+1}, b_{t+1}) \leftarrow$ one corrected-norm controller step (projected AdaGrad, Algorithm~\ref{alg:abcssl}) projected onto $\mathcal K$
\ENDFOR
\RETURN $\theta_\tau$ for uniform random $\tau \in \{1, \ldots, T\}$
\end{algorithmic}
\end{algorithm}

\begin{remark}[$\xi_t$ as an EMA-like bias tracker]
\label{rem:tt_ema}
The update~\eqref{eq:tt_xi_update} is precisely an exponential moving average of $-\dt$ with rate $\gamma_t$, so $\xi_t$ is an EMA-like quantity tracking the vector bias $B(\theta_t)$. Unlike the scalar EMA smoothing of $\widehat H_t$ used in the experiments (Remark~\ref{rem:ema_heuristic}), however, $\xi_t$ is not used by the controller: it is an auxiliary $\R^d$-valued debiasing variable that enters only the slow update~\eqref{eq:tt_slow_update} via $\widetilde g_t = \glam - (1-a_t)\xi_t$. The fast update runs in parallel with the controller and introduces no additional batch-gradient computation.
\end{remark}

\subsection{Assumptions}
\label{app:tt_assumptions}

In addition to Assumptions~\ref{asm:smooth} (smoothness of $\Ls$) and~\ref{asm:bounded} (bounded suboptimality $\Delta_0$), we require regularity on $B$, an iterate-boundedness condition on the debiased estimator, and a standard two-timescale step-size schedule.

\begin{assumption}[Smoothness of $\Lf$ and Lipschitz bias]
\label{asm:tt_Bsmooth}
The map $\Lf$ is $\beta_f$-smooth (i.e., $\nabla \Lf$ is $\beta_f$ Lipschitz continuous). Consequently $B = \nabla\Lf - \nabla\Ls$ is Lipschitz continuous with constant $L_B \leq \beta + \beta_f$---that is,
\[
    \norm{B(\theta) - B(\theta')} \;\le\; L_B\,\norm{\theta - \theta'} \quad \text{for all } \theta, \theta' \in \R^d.
\]
\end{assumption}

\begin{remark}[Noise structure of the bias sample]
\label{rem:tt_noise}
Define $\zeta_t := -\dt - B(\theta_t)$. By Lemma~\ref{lem:bias}, $\E_t[\zeta_t] = 0$, so $\zeta_t$ is a centered, conditionally unbiased sample of the bias. Moreover, $\zeta_t$ is precisely the centered labeled-batch estimate of $B(\theta_t)$, and the per-sample variance computation behind Proposition~\ref{prop:mse} gives
\[
    \E_t\norm{\zeta_t}^2 \;=\; \frac{\sigma_D^2(\theta_t)}{n},
    \qquad \sigma_D^2(\theta) := \sigma^2(\theta) + \sigma_f^2(\theta) - 2\,C(\theta) \;\ge\; 0,
\]
with $\sigma^2, \sigma_f^2, C$ from~\eqref{eq:sigma_sq}--\eqref{eq:C_def} evaluated at $\theta$. These statistics are defined pointwise in $\theta$ and need not be uniformly bounded, so the along-the-trajectory bound $\E_t\norm{\zeta_t}^2 \le \sigma_\zeta^2$ that the analysis uses is stated as part of Assumption~\ref{asm:tt_G} below.
\end{remark}

\begin{assumption}[Second-moment bounds]
\label{asm:tt_G}
There exist $G, \sigma_\zeta < \infty$ such that, for all $t \ge 1$, $\E\norm{\widetilde{g}_t}^2 \le G^2$ and $\E_t\norm{\zeta_t}^2 = \sigma_D^2(\theta_t)/n \le \sigma_\zeta^2$.
\end{assumption}

Assumption~\ref{asm:tt_G} is a standard iterate-boundedness-type hypothesis in the stochastic approximation literature~\citep{Borkar2008StochasticAA}. A sufficient primitive condition for both bounds is $(a)$~bounded per-sample gradients $\norm{\nabla\loss_\theta(x,y)} \le G_0$ uniformly in $\theta,x,y$, which bounds $\norm{\glam}$ and gives $\sigma_D^2(\theta) \le 16 G_0^2$ uniformly, combined with $(b)$~a uniform bound on $\norm{\xi_t}$, which follows from $(a)$~plus $\norm{B(\theta_t)} \le 2G_0$ plus induction on~\eqref{eq:tt_xi_update}. We state Assumption~\ref{asm:tt_G} directly in order to keep the analysis independent of these additional structural conditions.

\begin{assumption}[Step-size schedule]
\label{asm:tt_steps}
The step sizes satisfy $\gamma_t \in (0, 1]$, $\eta_t > 0$, $\gamma_t \to 0$, $\eta_t \to 0$, together with
\begin{equation}
    \sum_{t \ge 1} \gamma_t = \infty, \qquad \sum_{t \ge 1} \gamma_t^2 < \infty, \qquad \sum_{t \ge 1} \eta_t = \infty, \qquad \frac{\eta_t}{\gamma_t} \to 0, \qquad \sum_{t \ge 1} \frac{\eta_t^2}{\gamma_t} < \infty.
    \label{eq:tt_step_cond}
\end{equation}
\end{assumption}

A concrete schedule that satisfies~\eqref{eq:tt_step_cond} is $\gamma_t = c_1\, t^{-2/3}$ and $\eta_t = c_2\, t^{-1}$ with $0 < c_1 \le 1$ and $0 < c_2 \le 1/\beta$ (so that $\gamma_t \le 1$ and $\eta_t \le 1/\beta$ for all $t$); under this choice $\sum_t \eta_t = c_2 \sum_t t^{-1} = \infty$, $\eta_t/\gamma_t = (c_2/c_1)\,t^{-1/3} \to 0$, and $\eta_t^2/\gamma_t = (c_2^2/c_1)\,t^{-4/3}$ is summable.
This is the canonical decoupling used by~\citet{Borkar2008StochasticAA}, under which the standard two-timescale analysis yields an $O(k^{-2/3})$ mean-square tracking rate by direct substitution.
When a constant slow step size is preferred (to match Theorem~\ref{thm:main}), we obtain a polynomial finite-horizon rate by tuning both step sizes to the horizon $T$; see Theorem~\ref{thm:two_timescale_finite} below.

\subsection{Convergence guarantee}
\label{app:tt_theorem}

Define
\begin{equation}
    V_t \;:=\; \E\norm{e_t}^2
    \label{eq:tt_Vt_def}
\end{equation}
as the mean-square tracking error at time $t$.

\begin{theorem}[Convergence of two-timescale ABC-Align]
\label{thm:two_timescale}
The proof depends only on predictability of $(a_t, b_t)$ and is therefore unchanged from the original analysis under any predictable rule, including the corrected-norm joint controller of \Secref{sec:plugin}.
Suppose Assumptions~\ref{asm:smooth},~\ref{asm:bounded},~\ref{asm:tt_Bsmooth},~\ref{asm:tt_G} and step-size condition~\eqref{eq:tt_step_cond} hold, and let $\{(a_t, b_t)\}_{t=1}^T$ be any predictable sequence with $a_t \in [0,1]$, $b_t \ge 0$.
The iterates of Algorithm~\ref{alg:two_timescale_abcssl} with $\eta_t \le 1/\beta$ satisfy:
\begin{enumerate}[leftmargin=2em,label=(\roman*)]
\item \emph{(Fast-timescale tracking.)}
    \begin{equation}
        V_{t+1} \;\le\; (1 - \gamma_t/2)\,V_t \;+\; 4\,L_B^2 G^2\,\frac{\eta_t^2}{\gamma_t} \;+\; 2\,\sigma_\zeta^2\,\gamma_t^2,
        \label{eq:tt_tracking}
    \end{equation}
    and hence $V_t \to 0$; for the concrete schedule $\gamma_t = c_1 t^{-2/3}$, $\eta_t = c_2 t^{-1}$, moreover $V_t = O(t^{-2/3})$.
\item \emph{(Slow-timescale convergence.)}
    \begin{equation}
        \frac{1}{\sum_{s \le T}\eta_s}\sum_{t=1}^T \eta_t\,\E\norm{\nabla\Ls(\theta_t)}^2
        \;\le\; \frac{2\Delta_0}{\sum_{t \le T}\eta_t}
        + \frac{\sum_{t=1}^T \eta_t\,(1-a_t)^2 V_t}{\sum_{t \le T}\eta_t}
        + \frac{\beta\,\sum_{t=1}^T \eta_t^2\,\E\norm{\widetilde{g}_t}^2}{\sum_{t \le T}\eta_t}.
        \label{eq:tt_main_bound}
    \end{equation}
\item \emph{(Rate.)}
    For $\gamma_t = c_1 t^{-2/3}$, $\eta_t = c_2 t^{-1}$,
    \[
        \min_{1 \le t \le T}\E\norm{\nabla\Ls(\theta_t)}^2 \;\le\; \frac{1}{\sum_{t\le T}\eta_t}\sum_{t=1}^T\eta_t\,\E\norm{\nabla\Ls(\theta_t)}^2 \;=\; O\!\left(\frac{1}{\log T}\right).
    \]
\end{enumerate}
\end{theorem}

\noindent In contrast to Theorem~\ref{thm:main}, whose bound tolerates a persistent bias floor $\overline{(1-a_t)^2\Bt^2}$, Theorem~\ref{thm:two_timescale} replaces this term by $(1-a_t)^2 V_t$ with $V_t \to 0$, so the bias contribution vanishes asymptotically without constraining $a_t$.

\begin{proof}
The proof mirrors the structure of Appendix~\ref{app:thm_main}: we first control the fast-timescale tracking error $e_t$ (Step~1), then run a biased-SGD descent argument on the slow timescale (Step~2), and finally combine the two into the claimed rate (Step~3).

\medskip
\noindent\textbf{Step 1: Fast-timescale tracking error.}
By Remark~\ref{rem:tt_noise} and Assumption~\ref{asm:tt_G}, $-\dt = B(\theta_t) + \zeta_t$ with $\E_t[\zeta_t] = 0$ and $\E_t\norm{\zeta_t}^2 \le \sigma_\zeta^2$.
Substituting into~\eqref{eq:tt_xi_update},
\[
    \xi_{t+1} = (1-\gamma_t)\,\xi_t + \gamma_t\bigl(B(\theta_t) + \zeta_t\bigr).
\]
Subtracting $B(\theta_{t+1})$ from both sides,
\begin{align*}
    e_{t+1} \;&=\; \xi_{t+1} - B(\theta_{t+1}) \\
    &=\; (1-\gamma_t)\xi_t + \gamma_t B(\theta_t) + \gamma_t \zeta_t - B(\theta_{t+1}) \\
    &=\; (1-\gamma_t)\bigl(\xi_t - B(\theta_t)\bigr) + (1-\gamma_t)B(\theta_t) + \gamma_t B(\theta_t) - B(\theta_{t+1}) + \gamma_t\zeta_t \\
    &=\; (1-\gamma_t)\,e_t \;-\; \Delta B_t \;+\; \gamma_t\,\zeta_t,
\end{align*}
where $\Delta B_t := B(\theta_{t+1}) - B(\theta_t)$. This yields the key recursion
\begin{equation}
    e_{t+1} \;=\; (1-\gamma_t)\,e_t \;-\; \Delta B_t \;+\; \gamma_t\,\zeta_t.
    \label{eq:tt_recursion}
\end{equation}
By Assumption~\ref{asm:tt_Bsmooth} and~\eqref{eq:tt_slow_update}, $\norm{\Delta B_t} \le L_B\,\eta_t\,\norm{\widetilde g_t}$, and by Assumption~\ref{asm:tt_G},
\begin{equation}
    \E\norm{\Delta B_t}^2 \;\le\; L_B^2\,G^2\,\eta_t^2.
    \label{eq:tt_DeltaB}
\end{equation}
Next we take expectations of $\norm{e_{t+1}}^2$.
Expanding the square in~\eqref{eq:tt_recursion} (recall $\gamma_t \in (0,1]$ from Assumption~\ref{asm:tt_steps}),
\begin{equation}
    \norm{e_{t+1}}^2
    = \norm{(1-\gamma_t)e_t - \Delta B_t}^2
    + 2\gamma_t(1-\gamma_t)\inner{e_t}{\zeta_t}
    - 2\gamma_t\inner{\Delta B_t}{\zeta_t}
    + \gamma_t^2\norm{\zeta_t}^2.
    \label{eq:tt_expand}
\end{equation}
We handle the two cross terms separately. For the first, $e_t \in \mathcal{F}_{t-1}$ and $\E_t[\zeta_t]=0$ by Lemma~\ref{lem:bias}, so by the tower rule
\begin{equation}
    \E\!\bigl[\inner{e_t}{\zeta_t}\bigr]
    \;=\; \E\!\bigl[\inner{e_t}{\E_t[\zeta_t]}\bigr]
    \;=\; 0.
    \label{eq:tt_cross1}
\end{equation}
For the second, $\Delta B_t = B(\theta_{t+1})-B(\theta_t)$ depends on $\theta_{t+1} = \theta_t - \eta_t\widetilde g_t$ and hence on both batches at step~$t$, so $\langle \Delta B_t, \zeta_t\rangle$ is not zero-mean. Young's inequality with parameter $\gamma_t^{-1}$ bounds it by
\begin{equation}
    2\gamma_t\,\bigl|\inner{\Delta B_t}{\zeta_t}\bigr|
    \;\le\; \norm{\Delta B_t}^2 + \gamma_t^2\norm{\zeta_t}^2.
    \label{eq:tt_cross2}
\end{equation}
The leading quadratic term in~\eqref{eq:tt_expand} is controlled by Young's inequality with parameter $\rho = \gamma_t/2$:
\begin{align}
    \norm{(1-\gamma_t)e_t - \Delta B_t}^2
    &\le (1+\gamma_t/2)(1-\gamma_t)^2\,\norm{e_t}^2 + (1 + 2/\gamma_t)\,\norm{\Delta B_t}^2 \notag \\
    &\le (1 - \gamma_t/2)\,\norm{e_t}^2 + \frac{3}{\gamma_t}\,\norm{\Delta B_t}^2,
    \label{eq:tt_A_bound}
\end{align}
where the last inequality uses the direct expansion $(1+\gamma_t/2)(1-\gamma_t)^2 = 1 - \tfrac{3}{2}\gamma_t + \tfrac{1}{2}\gamma_t^3 \le 1 - \gamma_t/2$ for $\gamma_t \in (0,1]$ (since then $\gamma_t^3 \le \gamma_t$), together with $1 + 2/\gamma_t \le 3/\gamma_t$.
Substituting~\eqref{eq:tt_cross1}--\eqref{eq:tt_A_bound} into~\eqref{eq:tt_expand} and taking total expectations,
\begin{equation}
    \E\norm{e_{t+1}}^2 \;\le\; (1 - \gamma_t/2)\,\E\norm{e_t}^2 \;+\; \frac{4}{\gamma_t}\,\E\norm{\Delta B_t}^2 \;+\; 2\gamma_t^2\,\sigma_\zeta^2,
    \label{eq:tt_Esq}
\end{equation}
where the coefficient $4/\gamma_t$ on $\E\norm{\Delta B_t}^2$ combines the $3/\gamma_t$ from~\eqref{eq:tt_A_bound} with the $1 \le 1/\gamma_t$ contribution from~\eqref{eq:tt_cross2} (using $\gamma_t \le 1$), and $2\gamma_t^2\sigma_\zeta^2$ combines the $\gamma_t^2\norm{\zeta_t}^2$ term in~\eqref{eq:tt_expand} with the one in~\eqref{eq:tt_cross2}, followed by the second-moment bound from Assumption~\ref{asm:tt_G}.
Substituting~\eqref{eq:tt_DeltaB} into~\eqref{eq:tt_Esq} gives precisely~\eqref{eq:tt_tracking}:
\[
    V_{t+1} \;\le\; (1 - \gamma_t/2)\,V_t \;+\; 4\,L_B^2 G^2\,\frac{\eta_t^2}{\gamma_t} \;+\; 2\sigma_\zeta^2\,\gamma_t^2.
\]

Write $A_t := 4L_B^2 G^2\,\eta_t^2/\gamma_t + 2\sigma_\zeta^2\gamma_t^2$, so that~\eqref{eq:tt_tracking} reads $V_{t+1} \le (1 - \gamma_t/2)V_t + A_t$, and note
\[
    \frac{A_t}{\gamma_t} \;=\; 4L_B^2 G^2\,\frac{\eta_t^2}{\gamma_t^2} \;+\; 2\sigma_\zeta^2\,\gamma_t \;\longrightarrow\; 0
\]
by $\eta_t/\gamma_t \to 0$ and $\gamma_t \to 0$.

\emph{Convergence $V_t \to 0$ under~\eqref{eq:tt_step_cond}.}
Fix $\varepsilon > 0$ and pick $t_\varepsilon$ with $A_t \le \varepsilon\gamma_t/2$ for all $t \ge t_\varepsilon$. Then for $t \ge t_\varepsilon$,
\[
    V_{t+1} - \varepsilon \;\le\; (1 - \gamma_t/2)\,V_t + \tfrac{\varepsilon\gamma_t}{2} - \varepsilon \;=\; (1 - \gamma_t/2)\,(V_t - \varepsilon),
\]
and since $\sum_t \gamma_t = \infty$ implies $\prod_{s \ge t_\varepsilon}(1 - \gamma_s/2) = 0$, we get $\limsup_t V_t \le \varepsilon$; as $\varepsilon > 0$ was arbitrary, $V_t \to 0$. (This is the standard comparison argument of Chung; see, e.g.,~\citealp{Borkar2008StochasticAA}.) The initial condition $V_1 = \norm{B(\theta_1)}^2$ (since $\xi_1 = 0$) is forgotten through the same contraction.

\emph{Rate for the concrete schedule.}
Let $\gamma_t = c_1 t^{-2/3}$, $\eta_t = c_2 t^{-1}$. First, $A_s = (4L_B^2G^2 c_2^2/c_1)\,s^{-4/3} + 2\sigma_\zeta^2 c_1^2\,s^{-4/3}$ is summable, so unrolling~\eqref{eq:tt_tracking} crudely gives $\sup_t V_t \le V_1 + \sum_{s \ge 1} A_s =: \bar V_{\sup} < \infty$. Now split the horizon at $\lceil t/2 \rceil$: unrolling~\eqref{eq:tt_tracking} from $\lceil t/2 \rceil$ to $t$,
\[
    V_t \;\le\; \Bigl[\prod_{s=\lceil t/2\rceil}^{t-1}\bigl(1 - \tfrac{\gamma_s}{2}\bigr)\Bigr]\,V_{\lceil t/2\rceil}
    \;+\; \max_{\lceil t/2\rceil \le s < t}\frac{A_s}{\gamma_s}\;\sum_{s=\lceil t/2\rceil}^{t-1}\gamma_s\!\!\prod_{r=s+1}^{t-1}\bigl(1 - \tfrac{\gamma_r}{2}\bigr),
\]
where we bounded $A_s \le \gamma_s \max_{s'}(A_{s'}/\gamma_{s'})$ over the window. The weighted sum telescopes:
\[
    \sum_{s=\lceil t/2\rceil}^{t-1}\gamma_s\!\!\prod_{r=s+1}^{t-1}\bigl(1 - \tfrac{\gamma_r}{2}\bigr)
    \;=\; 2\sum_{s=\lceil t/2\rceil}^{t-1}\Bigl[\prod_{r=s+1}^{t-1}\bigl(1 - \tfrac{\gamma_r}{2}\bigr) - \prod_{r=s}^{t-1}\bigl(1 - \tfrac{\gamma_r}{2}\bigr)\Bigr] \;\le\; 2.
\]
Since $\sum_{s=\lceil t/2\rceil}^{t-1}\gamma_s/2 = \Theta(c_1 t^{1/3})$, the first term is at most $e^{-\Theta(t^{1/3})}\,\bar V_{\sup} = o(t^{-2/3})$, while $A_s/\gamma_s = (4L_B^2G^2 c_2^2/c_1^2 + 2\sigma_\zeta^2 c_1)\,s^{-2/3}$ gives $\max_{\lceil t/2\rceil \le s < t} A_s/\gamma_s = O(t^{-2/3})$. Hence $V_t = O(t^{-2/3})$.

\medskip
\noindent\textbf{Step 2: Slow-timescale descent.}
By Assumption~\ref{asm:smooth} and~\eqref{eq:tt_slow_update},
\[
    \Ls(\theta_{t+1}) \;\le\; \Ls(\theta_t) \;-\; \eta_t\inner{\nabla\Ls(\theta_t)}{\widetilde g_t} \;+\; \frac{\beta\eta_t^2}{2}\norm{\widetilde g_t}^2.
\]
Taking $\E_t[\cdot]$ and writing $\mu_t := \nabla\Ls(\theta_t)$ and $\bar g_t := \E_t[\widetilde g_t]$, we have $\bar g_t = \mu_t - (1-a_t)\,e_t$ by~\eqref{eq:tt_condmean}, and $\E_t[\norm{\widetilde g_t}^2] = \norm{\bar g_t}^2 + \mathrm{Var}_t(\widetilde g_t)$. Thus
\begin{equation}
    \E_t[\Ls(\theta_{t+1})] \;\le\; \Ls(\theta_t) \;-\; \eta_t \inner{\mu_t}{\bar g_t} \;+\; \frac{\beta\eta_t^2}{2}\bigl(\norm{\bar g_t}^2 + \mathrm{Var}_t(\widetilde g_t)\bigr).
    \label{eq:tt_cond_descent}
\end{equation}
The polarization identity gives
\begin{equation}
    \inner{\mu_t}{\bar g_t} \;=\; \tfrac{1}{2}\bigl(\norm{\mu_t}^2 + \norm{\bar g_t}^2 - \norm{\mu_t - \bar g_t}^2\bigr) \;=\; \tfrac{1}{2}\bigl(\norm{\mu_t}^2 + \norm{\bar g_t}^2 - (1-a_t)^2\norm{e_t}^2\bigr),
    \label{eq:tt_polar}
\end{equation}
where we used $\mu_t - \bar g_t = (1-a_t)\,e_t$. Substituting~\eqref{eq:tt_polar} into~\eqref{eq:tt_cond_descent},
\begin{equation}
    \E_t[\Ls(\theta_{t+1})] \;\le\; \Ls(\theta_t) - \frac{\eta_t}{2}\norm{\mu_t}^2 - \frac{\eta_t}{2}(1-\eta_t\beta)\norm{\bar g_t}^2 + \frac{\eta_t}{2}(1-a_t)^2\norm{e_t}^2 + \frac{\beta\eta_t^2}{2}\mathrm{Var}_t(\widetilde g_t).
    \label{eq:tt_after_polar}
\end{equation}
Since $\eta_t \le 1/\beta$, we have $1-\eta_t\beta \ge 0$, so the $-\tfrac{\eta_t}{2}(1-\eta_t\beta)\norm{\bar g_t}^2$ term is non-positive and can be dropped. Using $\beta\eta_t^2/2 \cdot \mathrm{Var}_t(\widetilde g_t) \le \beta\eta_t^2/2 \cdot \E_t\norm{\widetilde g_t}^2$ and rearranging,
\begin{equation}
    \norm{\nabla\Ls(\theta_t)}^2 \;\le\; \frac{2}{\eta_t}\bigl(\Ls(\theta_t) - \E_t[\Ls(\theta_{t+1})]\bigr) + (1-a_t)^2\norm{e_t}^2 + \beta\eta_t\,\E_t\norm{\widetilde g_t}^2.
    \label{eq:tt_per_step}
\end{equation}
Taking total expectation, multiplying by $\eta_t$, and summing from $t=1$ to $T$,
\begin{align*}
    \sum_{t=1}^T \eta_t\,\E\norm{\nabla\Ls(\theta_t)}^2
    &\le 2\sum_{t=1}^T \bigl(\E[\Ls(\theta_t)] - \E[\Ls(\theta_{t+1})]\bigr) + \sum_{t=1}^T \eta_t(1-a_t)^2 V_t + \beta\sum_{t=1}^T \eta_t^2\,\E\norm{\widetilde g_t}^2 \\
    &\le 2\Delta_0 + \sum_{t=1}^T \eta_t(1-a_t)^2\,V_t + \beta\sum_{t=1}^T \eta_t^2\,\E\norm{\widetilde g_t}^2.
\end{align*}
Dividing by $\sum_{t\le T}\eta_t$ yields~\eqref{eq:tt_main_bound}.

\medskip
\noindent\textbf{Step 3: Combining the rates.}
Consider the schedule $\gamma_t = c_1 t^{-2/3}$, $\eta_t = c_2 t^{-1}$. Then $V_t = O(t^{-2/3})$ by Step~1, and
\begin{align*}
    \sum_{t=1}^T \eta_t &\asymp \log T, \\
    \sum_{t=1}^T \eta_t(1-a_t)^2 V_t &\le \sum_{t=1}^T \eta_t V_t \;=\; O\!\left(\sum_{t=1}^T t^{-1}\,t^{-2/3}\right) \;=\; O\!\left(\sum_{t=1}^T t^{-5/3}\right) \;=\; O(1), \\
    \sum_{t=1}^T \eta_t^2\,\E\norm{\widetilde g_t}^2 &\le G^2 \sum_{t=1}^T \eta_t^2 \;=\; O\!\left(\sum_{t=1}^T t^{-2}\right) \;=\; O(1).
\end{align*}
Combining with~\eqref{eq:tt_main_bound},
\[
    \min_{1\le t \le T}\E\norm{\nabla\Ls(\theta_t)}^2 \;\le\; \frac{1}{\sum_{t \le T}\eta_t}\sum_{t=1}^T \eta_t\,\E\norm{\nabla\Ls(\theta_t)}^2 \;=\; O\!\left(\frac{1}{\log T}\right).
\]
The $1/\log T$ rate reflects that our functional---the time-averaged squared gradient of a nonconvex $\Ls$---is harder to control than the mean-square tracking error, whose $O(t^{-2/3})$ rate we do recover in part~(i). Obtaining a polynomial-in-$T$ rate on $\min_t \E\norm{\nabla\Ls(\theta_t)}^2$ within the decaying-step regime would require either (a)~a schedule with $\eta_t$ decaying slower than $t^{-1}$ (e.g.\ $\eta_t = t^{-a}$, $\gamma_t = t^{-b}$ with $1/2 < b < a < 1$ and $2a - b > 1$, so that the slow step is asymptotically dominated by the fast one and $\sum \eta_t^2/\gamma_t < \infty$ is preserved), or (b)~additional structure such as a Polyak--{\L}ojasiewicz condition. The qualitative conclusion---that the bias term of Theorem~\ref{thm:main} is replaced by a vanishing tracking-error contribution---is already established by~\eqref{eq:tt_main_bound} combined with $V_t \to 0$; for a polynomial finite-time stationarity bound, we instead tune both step sizes to the horizon, as in Theorem~\ref{thm:two_timescale_finite} below.

\end{proof}

\paragraph{From asymptotic consistency to finite-horizon optimization.}
Theorem~\ref{thm:two_timescale} establishes asymptotic consistency: under classical decaying step sizes satisfying~\eqref{eq:tt_step_cond}, the $\eta_t$-weighted average of $\E\norm{\nabla\Ls(\theta_t)}^2$ vanishes as $T \to \infty$; for the concrete schedule $\gamma_t = c_1 t^{-2/3}$, $\eta_t = c_2 t^{-1}$, the rate is $O(1/\log T)$. The decaying-step regime is not, however, the right tool for a polynomial finite-time stationarity bound: with $\eta_t \equiv \eta$ fixed and $\gamma_t = c_1 t^{-2/3}$, the ratio $\eta_t/\gamma_t$ grows in $t$, the timescale-separation condition in~\eqref{eq:tt_step_cond} fails, and the $\eta^2 T^{2/3}$ piece of the tracking-error contribution does not vanish for any fixed $\eta > 0$. A cleaner approach is to fix a horizon $T$ and tune both \emph{constant} step sizes to $T$, optimizing the bias--variance--tracking trade-off explicitly.

\begin{theorem}[Finite-horizon tuned rate for two-timescale ABC-Align]
\label{thm:two_timescale_finite}
As in Theorem~\ref{thm:two_timescale}, the bound depends only on predictability of $(a_t, b_t)$ and applies under any predictable rule, including the corrected-norm joint controller of \Secref{sec:plugin}. Suppose Assumptions~\ref{asm:smooth},~\ref{asm:bounded},~\ref{asm:tt_Bsmooth}, and~\ref{asm:tt_G} hold. Fix a horizon $T \ge 1$, and let $\{(a_t, b_t)\}_{t=1}^T$ be any predictable sequence with $a_t \in [0, 1]$, $b_t \ge 0$. Run Algorithm~\ref{alg:two_timescale_abcssl} with constant step sizes $\eta_t \equiv \eta$ and $\gamma_t \equiv \gamma$, where $\eta \in (0, 1/\beta]$, $\gamma \in (0, 1]$, and $\eta \le \gamma$. Then
\begin{equation}
    \frac{1}{T}\sum_{t=1}^T \E\norm{\nabla\Ls(\theta_t)}^2
    \;\le\;
    \frac{2\Delta_0}{\eta T}
    \,+\, \frac{2 V_1}{\gamma T}
    \,+\, \frac{8\, L_B^2 G^2\,\eta^2}{\gamma^2}
    \,+\, 4\sigma_\zeta^2\,\gamma
    \,+\, \beta\eta G^2,
    \label{eq:tt_finite_bound}
\end{equation}
where $V_1 = \E\norm{e_1}^2 = \norm{B(\theta_1)}^2$. In particular, choosing $\eta = c_\eta T^{-1/2}$ and $\gamma = c_\gamma T^{-1/3}$ with $c_\eta \le 1/\beta$ and $c_\eta T^{-1/2} \le c_\gamma T^{-1/3} \le 1$ yields
\begin{equation}
    \frac{1}{T}\sum_{t=1}^T \E\norm{\nabla\Ls(\theta_t)}^2 \;=\; O\!\left(T^{-1/3}\right).
    \label{eq:tt_finite_rate}
\end{equation}
Equivalently, if $\tau$ is drawn uniformly from $\{1, \ldots, T\}$, then $\E\norm{\nabla\Ls(\theta_\tau)}^2 = O(T^{-1/3})$.
\end{theorem}

\begin{proof}
With constant $\gamma_t \equiv \gamma$, the fast-timescale recursion~\eqref{eq:tt_tracking} from Step~1 of the proof of Theorem~\ref{thm:two_timescale} becomes
\[
    V_{t+1} \;\le\; (1 - \gamma/2)\,V_t \;+\; \frac{4\, L_B^2 G^2\,\eta^2}{\gamma} \;+\; 2\sigma_\zeta^2\,\gamma^2.
\]
Unrolling and bounding the geometric tail by $\sum_{s \ge 0}(1 - \gamma/2)^s = 2/\gamma$,
\[
    V_t \;\le\; (1 - \gamma/2)^{t-1}\,V_1 \;+\; \frac{8\, L_B^2 G^2\,\eta^2}{\gamma^2} \;+\; 4\sigma_\zeta^2\,\gamma.
\]
Averaging over $t = 1, \ldots, T$ and using $\sum_{t=1}^T (1-\gamma/2)^{t-1} \le 2/\gamma$,
\begin{equation}
    \frac{1}{T}\sum_{t=1}^T V_t \;\le\; \frac{2 V_1}{\gamma T} \;+\; \frac{8\, L_B^2 G^2\,\eta^2}{\gamma^2} \;+\; 4\sigma_\zeta^2\,\gamma.
    \label{eq:tt_const_Vavg}
\end{equation}
Substituting into the slow-timescale bound~\eqref{eq:tt_main_bound} with constant $\eta_t \equiv \eta$ (so $\sum_t \eta_t = \eta T$), and using $(1-a_t)^2 \le 1$ together with Assumption~\ref{asm:tt_G},
\[
    \frac{1}{T}\sum_{t=1}^T \E\norm{\nabla\Ls(\theta_t)}^2
    \;\le\; \frac{2\Delta_0}{\eta T} + \frac{1}{T}\sum_{t=1}^T V_t + \beta\eta G^2.
\]
Combining with~\eqref{eq:tt_const_Vavg} yields~\eqref{eq:tt_finite_bound}. For the rate~\eqref{eq:tt_finite_rate}, the five terms with $\eta = c_\eta T^{-1/2}$, $\gamma = c_\gamma T^{-1/3}$ scale as $T^{-1/2}, T^{-2/3}, T^{-1/3}, T^{-1/3}, T^{-1/2}$ respectively, and the slowest decay $T^{-1/3}$ dominates.
\end{proof}

\paragraph{Three regimes.}
Theorems~\ref{thm:two_timescale} and~\ref{thm:two_timescale_finite} answer different questions and are not directly comparable.
\begin{itemize}[leftmargin=2em]
\item \emph{Asymptotic consistency} (Theorem~\ref{thm:two_timescale}): a classical decaying-step schedule satisfying~\eqref{eq:tt_step_cond} drives the $\eta_t$-weighted average of $\E\norm{\nabla\Ls(\theta_t)}^2$ to zero; for the concrete schedule with $\eta_t = c_2 t^{-1}$ the rate is $O(1/\log T)$, with averaging weights decaying like $1/t$.
\item \emph{Finite-horizon optimization} (Theorem~\ref{thm:two_timescale_finite}): fixing a budget $T$ and tuning the two constant step sizes to $T$ explicitly trades off bias, variance, and tracking error, yielding the polynomial rate $O(T^{-1/3})$ on the uniformly averaged ($1/T$-weighted) gradient norm.
\item \emph{Truly non-asymptotic / anytime bounds} would require step-size choices independent of any future horizon and, in the nonconvex setting, typically demand additional structure (e.g., a Polyak--{\L}ojasiewicz condition). We do not establish such a guarantee here.
\end{itemize}
The $O(1/\log T)$ vs.\ $O(T^{-1/3})$ gap reflects these different question regimes---different step-size families, different averaging weights, different notions of progress---and is not a deficiency of either result.

\subsection{Practical considerations}
\label{app:tt_practical}

\paragraph{Memory and compute overhead.}
Algorithm~\ref{alg:two_timescale_abcssl} requires storing the single vector $\xi_t \in \R^d$, identical in size to the model gradient, and performs two scalar-vector operations per step ($\xi_t$-update and subtraction in~\eqref{eq:tt_debiased}). No additional batch gradients are computed beyond those already required by Algorithm~\ref{alg:abcssl}, so the per-iteration cost is dominated by the same batch-gradient evaluations as vanilla ABC-Align under the corrected-norm controller.

\paragraph{Initialization.}
We initialize $\xi_1 = 0$, which corresponds to the prior $B(\theta_1) \approx 0$ (equivalently, an unbiased teacher at initialization). The first-step debiased estimator then reduces to $\widetilde g_1 = \glam$, matching Algorithm~\ref{alg:abcssl}; this is consistent with the first-step choice $(a_1, b_1) = (1, 0)$ already used in ABC-Align.

\paragraph{Interaction with the corrected-norm controller.}
The fast-timescale update~\eqref{eq:tt_xi_update} and the joint corrected-norm controller of \Secref{sec:plugin} operate independently: $\xi_t$ debiases the slow update, while the controller minimizes the per-step MSE in $(a_t, b_t)$. They share only the $\dt$ stream---$\xi_t$ consumes it as the bias proxy $\E_t[-\dt] = B(\theta_t)$, while the controller consumes it for the scalar primitives $\norm{\dt}^2, \inner{\dt}{c_t}, \inner{\gnf_t}{\dt}$ and for the half-batch construction of $\widehat H_t$.

\paragraph{Constant-stepsize regime and connections to recent work.}
A constant slow step size $\eta_t \equiv \eta$ and constant fast step $\gamma_t \equiv \gamma$ with $\eta \ll \gamma$ matches the setting analyzed in~\citet{Kwon2025TwoTimescaleLSA}, whose Wasserstein-ergodicity machinery implies that $(\theta_t, \xi_t)$ admits a unique joint stationary distribution and gives explicit bounds on the bias and variance of each iterate in terms of the two step sizes. Porting those techniques to our nonlinear nonconvex setting---where $\Ls$ is $\beta$-smooth but not strongly monotone in $\theta$---and pairing them with the Ruppert--Polyak averaging of~\citet{Doan2024FastNonlinearTT} to reach an $O(1/T)$ rate is a natural direction for future work.

\paragraph{Caveats.}
The main hypotheses that deserve scrutiny are (i) Assumption~\ref{asm:tt_Bsmooth}, which requires both $\Ls$ and $\Lf$ to be $\beta$-smooth; (ii) Assumption~\ref{asm:tt_G}, iterate-boundedness and noise-second-moment hypotheses ($\E\norm{\widetilde g_t}^2 \le G^2$ and $\E_t\norm{\zeta_t}^2 \le \sigma_\zeta^2$ along the trajectory) that are standard in the stochastic approximation literature but not proved from first principles here; and (iii) the step-size schedule~\eqref{eq:tt_step_cond}, whose $\eta_t/\gamma_t \to 0$ requirement dictates that the slow timescale is asymptotically dominated by the fast one. None of these conditions are required for Theorem~\ref{thm:main} itself; they are the price of removing the residual bias floor.

\end{document}